%% file: arxiv.tex
\documentclass{article}
\pdfoutput=1

\input{latexdefinitions.tex}

\usepackage[accepted]{icml2020_arxiv}

\icmltitlerunning{How Good is the Bayes Posterior in Deep Neural Networks Really?}

\begin{document}

\twocolumn[
\icmltitle{How Good is the Bayes Posterior in Deep Neural Networks Really?}

\input{authors.tex}

\icmlkeywords{Bayesian Deep Learning}

\vskip 0.3in
]

\newcommand{\workAtGoogle}{\textsuperscript{+}Work done while at Google }
\printAffiliationsAndNotice{\icmlEqualContribution\workAtGoogle} %

\input{maincontent.tex}

\setcounter{section}{0}
\renewcommand{\thesection}{\Alph{section}}

\newpage
\input{supplementarycontent.tex}

\bibliography{arxiv.bib}
\bibliographystyle{icml2020}

\end{document}

%% file: latexdefinitions.tex
\usepackage{microtype}
\usepackage{graphicx}
\usepackage{caption}
\usepackage{subcaption}
\usepackage{booktabs} %

\usepackage{tcolorbox}  %

\newcommand{\hypothesis}[2]{%
\begin{tcolorbox}[colback=green!10!white,leftrule=2.5mm,size=title]
\textbf{#1}: #2
\end{tcolorbox}
\vspace{-0.1cm}%
}

\usepackage{float}

\usepackage{tikz}
\usetikzlibrary{arrows,positioning}
\tikzset{
    >=stealth',
    punkt/.style={
           rectangle,
           rounded corners,
           draw=black, very thick,
           text width=6.5em,
           minimum height=2em,
           text centered},
    pil/.style={
           ->,
           thick,
           shorten <=2pt,
           shorten >=2pt},
    el/.style={
        inner sep=2pt,
        align=left,
        right=0.15cm}
    elv/.style={
        inner sep=2pt,
        align=left,
        above=0.15cm}
}

\usepackage{nameref,zref-xr}    %

\def\pars{{\boldsymbol{\theta}}}
\def\moms{{\mathbf{m}}}
\def\rawm{{\mathbf{v}}}
\def\rands{{\mathbf{R}}}

\usepackage[T1]{fontenc}

\usepackage{amsmath,amsfonts,bm}

\DeclareMathOperator*{\argmin}{argmin}
\usepackage{pifont}%
\newcommand{\cmark}{\ding{51}}%
\newcommand{\xmark}{\ding{55}}%

\usepackage{amsthm}
\usepackage{url}
\usepackage{xcolor}
\usepackage{wrapfig}
\usepackage{paralist}
\usepackage{todonotes}
\usepackage{tikz}

\usetikzlibrary{bayesnet}

\newtheorem{proposition}{Proposition}%

\newcommand{\experiment}[1]{\textbf{\underline{#1}}}

\usepackage[linesnumbered,ruled,vlined,algo2e]{algorithm2e}
\SetAlFnt{\small}
\SetAlCapFnt{\small}
\SetAlCapNameFnt{\small}
\makeatletter
\newcommand{\removelatexerror}{\let\@latex@error\@gobble}
\makeatother

\SetCommentSty{mycommfont}
\SetKwInput{KwInput}{Input}                %
\SetKwInput{KwOutput}{Output}              %

%% file: authors.tex
\icmlsetsymbol{equal}{*}
\icmlsetsymbol{atgoogle}{+}

\begin{icmlauthorlist}
\icmlauthor{Florian Wenzel}{equal,google}
\icmlauthor{Kevin Roth}{equal,atgoogle,ethz}
\icmlauthor{Bastiaan S. Veeling}{equal,atgoogle,uva,google}
\icmlauthor{Jakub \'Swi\k{a}tkowski}{warsaw,atgoogle}
\icmlauthor{Linh Tran}{imperial,atgoogle}
\icmlauthor{Stephan Mandt}{uci,atgoogle}
\icmlauthor{Jasper Snoek}{google}
\icmlauthor{Tim Salimans}{google}
\icmlauthor{Rodolphe Jenatton}{google}
\icmlauthor{Sebastian Nowozin}{msr,atgoogle}
\end{icmlauthorlist}

\icmlaffiliation{ethz}{ETH Zurich}
\icmlaffiliation{uva}{University of Amsterdam}
\icmlaffiliation{warsaw}{University of Warsaw}
\icmlaffiliation{imperial}{Imperial College London}
\icmlaffiliation{uci}{University of California, Irvine}
\icmlaffiliation{msr}{Microsoft Research}
\icmlaffiliation{google}{Google Research}

\icmlcorrespondingauthor{Florian Wenzel}{\mbox{florianwenzel@google.com}}

%% file: maincontent.tex
\begin{abstract}
During the past five years the Bayesian deep learning community has developed increasingly accurate and efficient approximate inference procedures that allow for Bayesian inference in deep neural networks.
However, despite this algorithmic progress and the promise of improved uncertainty quantification and sample efficiency there are---as of early 2020---no publicized deployments of Bayesian neural networks in industrial practice.
In this work we cast doubt on the current understanding of Bayes posteriors in popular deep neural networks:
we demonstrate through careful MCMC sampling that the posterior predictive induced by the Bayes posterior yields systematically worse predictions compared to simpler methods including point estimates obtained from SGD.
Furthermore, we demonstrate that predictive performance is improved significantly through the use of a ``cold posterior'' that overcounts evidence.
Such cold posteriors sharply deviate from the Bayesian paradigm but are commonly used as heuristic in Bayesian deep learning papers.
We put forward several hypotheses that could explain cold posteriors and evaluate the hypotheses through experiments.
Our work questions the goal of accurate posterior approximations in Bayesian deep learning:
If the true Bayes posterior is poor, what is the use of more accurate approximations?
Instead, we argue that it is timely to focus on understanding the origin of the improved performance of cold posteriors.

\textsc{Code:}~\url{https://github.com/google-research/google-research/tree/master/cold_posterior_bnn}
\end{abstract}

\section{Introduction}
In supervised deep learning we use a training dataset
$\mathcal{D} = \{(x_i,y_i)\}_{i=1,\dots,n}$ and a probabilistic model
$p(y|x,\pars)$ to minimize the regularized cross-entropy objective,
\begin{equation}
L(\pars) := -\frac{1}{n} \sum_{i=1}^n \log p(y_i|x_i,\pars) + \Omega(\pars),
\label{eqn:dnnloss}
\end{equation}
where $\Omega(\pars)$ is a regularizer over model parameters.
We approximately optimize~(\ref{eqn:dnnloss}) using variants of stochastic gradient descent (SGD),~\citep{sutskever2013momentum}.
Beside being efficient, the SGD minibatch noise also has generalization benefits~\citep{masters2018batchsize,mandt2017stochastic}.

\begin{figure}[!t]
\center{\includegraphics[width=\columnwidth]{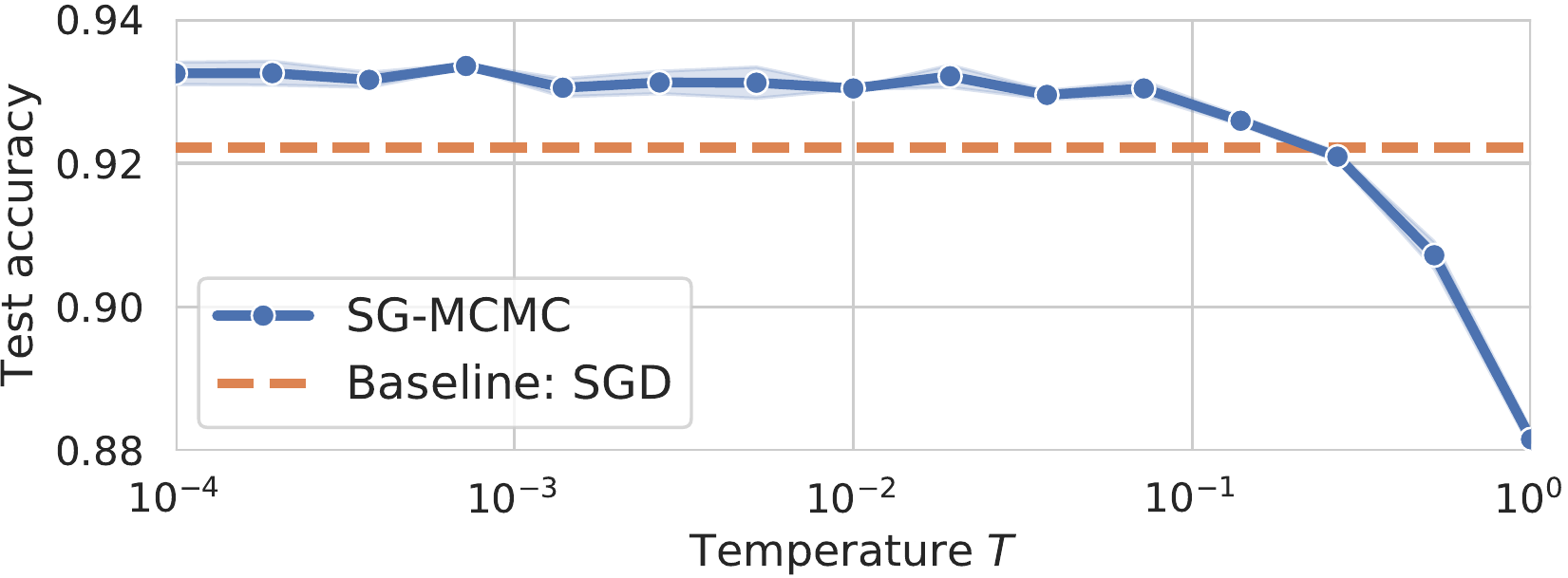}}%
\vspace{-0.35cm}%
\caption{The ``\textbf{cold posterior}'' effect: for a ResNet-20 on CIFAR-10 we can improve the generalization performance significantly by cooling the posterior with a temperature $T \ll 1$, deviating from the Bayes posterior
$p(\pars|\mathcal{D}) \propto \exp(-U(\pars)/T)$ at $T=1$.}%
\label{fig:resnet-tempdep-acc}%
\vspace{-0.35cm}%
\end{figure}

\subsection{Bayesian Deep Learning}
In Bayesian deep learning we do not optimize for a \emph{single} likely model but instead want to discover \emph{all} likely models.
To this end we approximate the \emph{posterior distribution} over model parameters,
$p(\pars|\mathcal{D}) \propto \exp(-U(\pars)/T)$, where $U(\pars)$
is the \emph{posterior energy function},
\begin{equation}
U(\pars) := - \sum_{i=1}^n \log p(y_i|x_i,\pars) - \log p(\pars),
\label{eqn:U}
\end{equation}
and $T$ is a \emph{temperature}.
Here $p(\pars)$ is a \emph{proper} prior density function, for example a Gaussian density.
If we scale $U(\pars)$ by $1/n$ and set $\Omega(\pars)=-\frac{1}{n} \log p(\pars)$ we recover $L(\pars)$ in~(\ref{eqn:dnnloss}).
Therefore $\exp(-U(\pars))$ simply gives high probability to models which have low loss $L(\pars)$.
Given $p(\pars|\mathcal{D})$ we \emph{predict} on a new instance $x$ by averaging over all likely models,
\begin{equation}
    p(y|x,\mathcal{D}) = \int p(y|x,\pars) \, p(\pars|\mathcal{D}) \,\textrm{d}\pars,
    \label{eqn:postpred}
\end{equation}
where~(\ref{eqn:postpred}) is also known as \emph{posterior predictive} or \emph{Bayes ensemble}.
Solving the integral~(\ref{eqn:postpred}) exactly is not possible.
Instead, we approximate the integral using a sample approximation,
$p(y|x,\mathcal{D}) \approx \frac{1}{S} \sum_{s=1}^S p(y|x,\pars^{(s)})$,
where $\pars^{(s)}$, $s=1,\dots,S$, is approximately sampled from $p(\pars|\mathcal{D})$.

\newcommand{\phenomenon}{Cold Posteriors}
The remainder of this paper studies a surprising effect shown in~Figure~\ref{fig:resnet-tempdep-acc}, the ``\emph{\phenomenon}'' effect:
for deep neural networks the Bayes posterior (at temperature $T=1$) works poorly but by cooling the posterior using a temperature $T < 1$ we can significantly improve the prediction performance.
\begin{tcolorbox}[colback=black!10!white,size=title]
\textbf{\phenomenon}: among all temperized posteriors the best posterior predictive performance on holdout data is achieved at temperature $T < 1$.
\end{tcolorbox}%
\vspace{-0.3cm}%

\subsection{Why Should Bayes ($T=1$) be Better?}\label{sec:bayes-better}
Why would we expect that predictions made by the \emph{ensemble model}~(\ref{eqn:postpred}) could improve over predictions made at a single well-chosen parameter?
There are three reasons:
1. \emph{Theory}: for several models where the predictive performance can be analyzed it is known that the posterior predictive~(\ref{eqn:postpred}) can dominate common point-wise estimators based on the likelihood,~\citep{komaki1996predictive}, even in the case of misspecification,~\citep{fushiki2005bootstrapprediction,ramamoorthi2015misspecification};
2. \emph{Classical empirical evidence}: for classical statistical models, averaged predictions~(\ref{eqn:postpred}) have been observed to be more robust in practice,~\citep{geisser1993predictiveinference}; and
3. \emph{Model averaging}: recent deep learning models based on deterministic model averages,~\citep{lakshminarayanan2017deepensembles,ovadia2019modeluncertainty}, have shown good predictive performance.

Note that a large body of work in the area of Bayesian deep learning in the last five years is motivated by the assertion that predicting using~(\ref{eqn:postpred}) is desirable.
We will confront this assertion through a simple experiment to show that our understanding of the Bayes posterior in deep models is limited.  Our work makes the following \textbf{contributions}:
\begin{compactitem}
\item We demonstrate for two models and tasks (ResNet-20 on CIFAR-10 and CNN-LSTM on IMDB) that the Bayes posterior predictive has poor performance compared to SGD-trained models.
\item We put forth and systematically examine hypotheses that could explain the observed behaviour.
\item We introduce two new diagnostic tools for assessing the approximation quality of stochastic gradient Markov chain Monte Carlo methods (SG-MCMC) and demonstrate that the posterior is accurately simulated by existing SG-MCMC methods.
\end{compactitem}

\section{Cold Posteriors Perform Better}\label{sec:demo}
We now examine the quality of the posterior predictive for two simple deep neural networks.
We will describe details of the models, priors, and approximate inference methods in Section~\ref{sec:bdl-practice} and Appendix~\ref{sec:resnet-details} to~\ref{sec:cnnlstm-details}.
In particular, we will study the accuracy of our approximate inference and the influence of the prior in great detail in Section~\ref{sec:inference-accuracy} and Section~\ref{sec:prior}, respectively.
Here we show that temperized Bayes ensembles obtained via low temperatures $T < 1$ outperform the true Bayes posterior at temperature $T=1$.

\subsection{Deep Learning Models: ResNet-20 and LSTM}

\begin{figure}[!t]
\center{\includegraphics[width=\columnwidth]{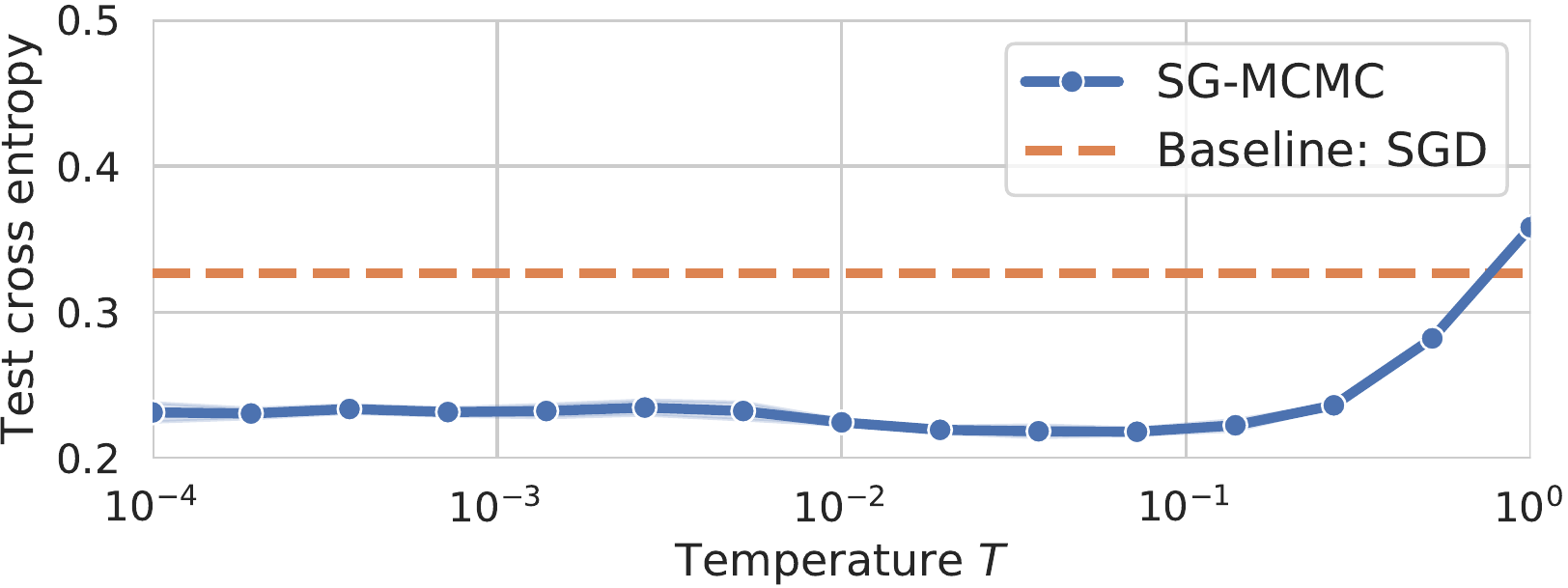}}%
\vspace{-0.4cm}%
\caption{Predictive performance on the CIFAR-10 test set for a cooled ResNet-20 Bayes posterior.
The SGD baseline is separately tuned for the same model (Appendix~\ref{sec:resnet-sgd-details}).}%
\label{fig:resnet-tempdep-ce}%
\vspace{-0.3cm}%
\end{figure}

\textbf{ResNet-20 on CIFAR-10.}
Figure~\ref{fig:resnet-tempdep-acc} and~\ref{fig:resnet-tempdep-ce} show the test accuracy and test cross-entropy of a Bayes prediction~(\ref{eqn:postpred}) for a ResNet-20 on the CIFAR-10 classification task.\footnote{A similar plot is Figure~3 in~\citep{baldock2019bayesiannn} and another is in the appendix of~\citep{zhang2019cyclicalsgmcmc}.}
We can clearly see that both accuracy and cross-entropy are significantly improved for a temperature $T < 1/10$ and that this trend is consistent.
Also, surprisingly this trend holds all the way to small $T = 10^{-4}$: the test performance obtained from an ensemble of models at temperature $T=10^{-4}$ is superior to the one obtained from $T=1$ and better than the performance of a single model trained with SGD. In Appendix~\ref{sec:tempplots-appendix} we show that the uncertainty metrics Brier score~\cite{brier1950verification} and expected calibration error (ECE)~\cite{naeini2015obtaining} are also improved by cold posteriors.

\begin{figure}[!t]
\center{\includegraphics[width=\columnwidth]{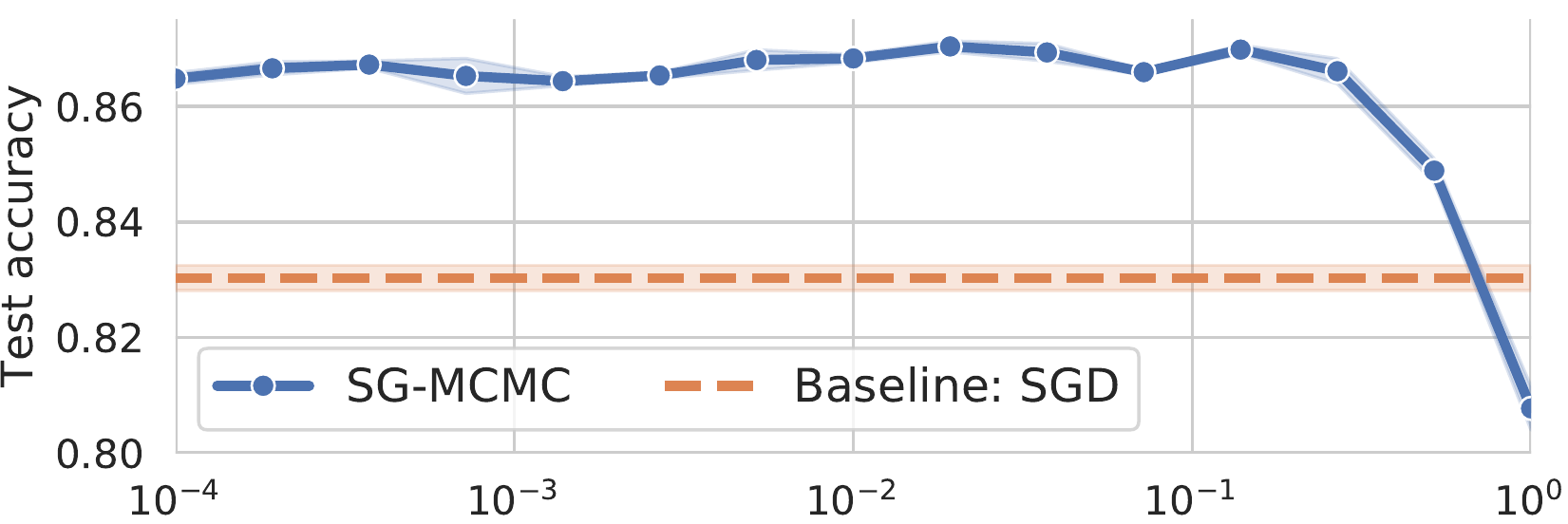}}%
\vspace{-0.2cm}%
\center{\includegraphics[width=\columnwidth]{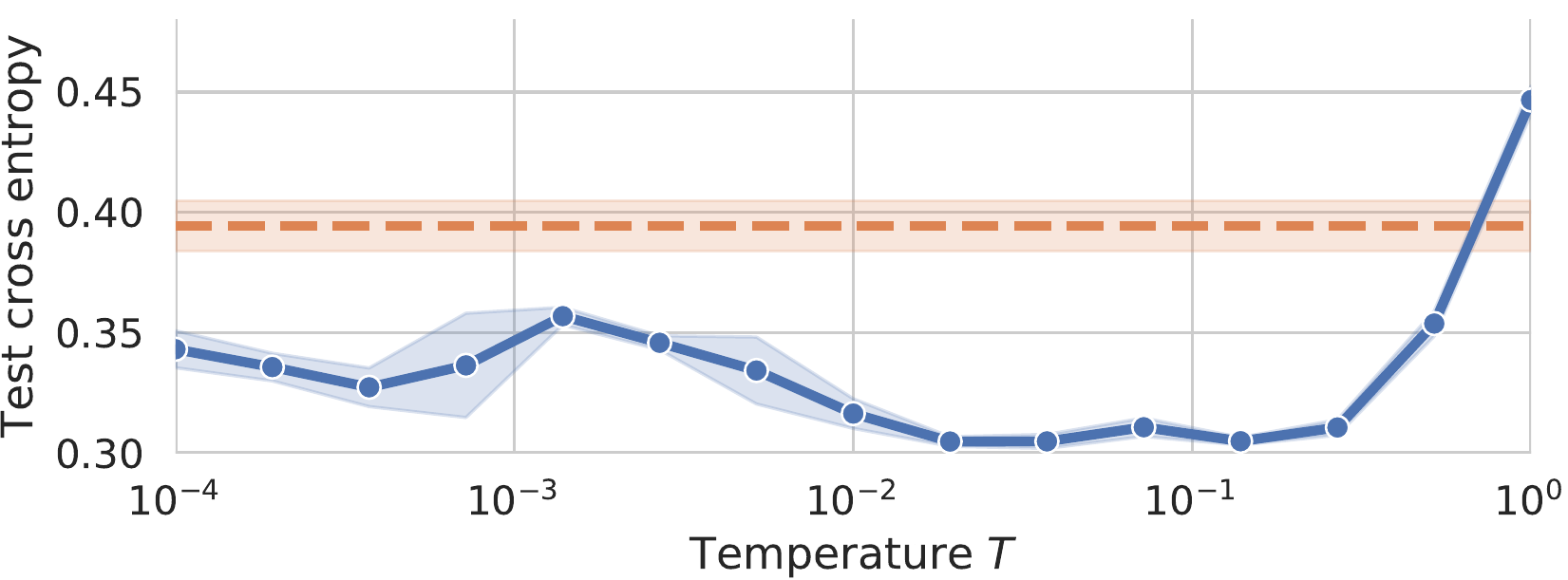}}%
\vspace{-0.4cm}%
\caption{Predictive performance on the IMDB sentiment task test set for a tempered CNN-LSTM Bayes posterior.  Error bars are $\pm$ one standard error over three runs.  See Appendix~\ref{sec:cnnlstm-sgd-details}.}%
\label{fig:imdb-tempdep}%
\vspace{-0.35cm}%
\end{figure}

\textbf{CNN-LSTM on IMDB text classification.}
Figure~\ref{fig:imdb-tempdep} shows the test accuracy and test cross-entropy of the tempered prediction~(\ref{eqn:postpred}) for a CNN-LSTM model on the IMDB sentiment classification task.
The optimal predictive performance is again achieved for a tempered posterior with a temperature range of approximately $0.01 < T < 0.2$.

\subsection{Why is a Temperature of $T < 1$ a Problem?}
There are two reasons why cold posteriors are problematic.
\emph{First}, $T < 1$ corresponds to artificially sharpening the posterior, which can be interpreted as overcounting the data by a factor of $1/T$ and a rescaling\footnote{E.g., using a Normal prior with temperature $T$ results in a Normal distribution with scaled variance by a factor of $T$.} of the prior as $p(\theta)^\frac{1}{T}$. This is equivalent to a Bayes posterior obtained from a dataset consisting of $1/T$ replications of the original data, giving too strong evidence to individual models. For $T=0$, all posterior probability mass is concentrated on the set of maximum a posteriori (MAP) point estimates.
\emph{Second}, $T=1$ corresponds to the true Bayes posterior and performance gains for $T < 1$ point to a deeper and potentially resolvable problem with the prior, likelihood, or inference procedure.

\subsection{Confirmation from the Literature}
Should the strong performance of tempering the posterior with $T \ll 1$ surprise us?
It certainly is an observation that needs to be explained, but it is not new:
if we comb the literature of Bayesian inference in deep neural networks we find broader evidence for this phenomenon.

\paragraph{Related work that uses $T < 1$ posteriors in SG-MCMC.}
The following table lists work that uses SG-MCMC on deep neural networks and tempers the posterior.\footnote{%
For~\citep{li2016preconditionedsgld} the tempering with $T=1/\sqrt{n}$ arises due to an implementation mistake.
For~\citep{heek2019sgmcmc} we communicated with the authors, and tempering arises due to overcounting data by a factor of $5$, approximately justified by data augmentation, corresponding to $T=1/5$.
For~\citep{zhang2019cyclicalsgmcmc} the original implementation contains inadvertent tempering, however, the authors added a study of tempering in a revision.
}
\begin{center}%
\vspace{-0.25cm}%
\scalebox{0.9}{%
\begin{tabular}{ll}\toprule
Reference & Temperature $T$\\ \midrule
\citep{li2016preconditionedsgld} & $1/\sqrt{n}$\\
\citep{leimkuhler2019partitionedlangevin} & $T < 10^{-3}$\\
\citep{heek2019sgmcmc} & $T=1/5$\\
\citep{zhang2019cyclicalsgmcmc} & $T=1/\sqrt{50000}$\\
\bottomrule
\end{tabular}%
}%
\vspace{-0.25cm}%
\end{center}%

\paragraph{Related work that uses $T < 1$ posteriors in Variational Bayes.}
In the variational Bayes approach to Bayesian neural networks,~\citep{blundell2015weightuncertainty,hinton1993mdl,mackay1995ensemblelearning,barber1998ensemblelearning} we optimize the parameters $\mathbf{\tau}$ of a variational distribution $q(\pars|\mathbf{\tau})$ by maximizing the evidence lower bound (ELBO),
\begin{equation}
    \mathbb{E}_{\pars \sim q(\pars|\mathbf{\tau})}\!\!\left[\sum_{i=1}^n \log p(y_i|x_i,\pars)\right]\!
    - \lambda D_{\textrm{KL}}(q(\pars|\mathbf{\tau}) \| p(\pars)).\!\!\!
    \label{eqn:vb-bnn}
\end{equation}
For $\lambda=1$ this directly minimizes $D_{\textrm{KL}}(q(\pars|\mathbf{\tau})\,\|\,p(\pars|\mathcal{D}))$ and thus for sufficiently rich variational families will closely approximate the true Bayes posterior $p(\pars|\mathcal{D})$.
However, in practice researchers discovered that using values $\lambda < 1$ provides better predictive performance, with common values shown in the following table.\footnote{For~\citep{osawa2019practicalbdl} scaling with $\lambda$ arises due to their use of a ``data augmentation factor'' $\rho \in \{5, 10\}$.}
\begin{center}%
\vspace{-0.1cm}%
\scalebox{0.9}{%
\begin{tabular}{ll}\toprule
Reference & KL term weight $\lambda$ in~(\ref{eqn:vb-bnn})\\ \midrule
\citep{zhang2018noisynaturalgradient} & $\lambda \in \{1/2, 1/10\}$ \\
\citep{bae2018kfacvb} & tuning of $\lambda$, unspecified \\
\citep{osawa2019practicalbdl} & $\lambda \in \{1/5, 1/10\}$\\
\citep{ashukha2019uncertaintyestimation} & $\lambda$ from $10^{-5}$ to $10^{-3}$\\
\bottomrule
\vspace{-0.25cm}%
\end{tabular}%
}%
\end{center}
In Appendix~\ref{sec:tempered-elbo} we show that the KL-weighted ELBO~(\ref{eqn:vb-bnn}) arises from tempering the likelihood part of the posterior.

From the above list we can see that the cold posterior problem has left a trail in the literature, and in fact we are not aware of \emph{any} published work demonstrating well-performing Bayesian deep learning at temperature $T=1$.
We now give details on how we perform accurate Bayesian posterior inference in deep learning models.

\section{Bayesian Deep Learning in Practice}\label{sec:bdl-practice}
In this section we describe how we achieve efficient and accurate simulation of Bayesian neural network posteriors.
This section does not contain any major novel contribution but instead combines existing work.

\subsection{Posterior Simulation using Langevin Dynamics}\label{sec:langevin-dynamics}
To generate approximate parameter samples $\pars \sim p(\pars\,|\,\mathcal{D})$ we
consider \emph{Langevin dynamics} over parameters $\pars \in \mathbb{R}^d$ and momenta
$\moms \in \mathbb{R}^d$, defined by the Langevin stochastic differential equation (SDE),
\begin{align}
\textrm{d}\, \pars &= \mathbf{M}^{-1}\, \moms \,\textrm{d}t,\label{eqn:langevin1}\\
\textrm{d}\, \moms &= -\nabla_{\pars} U(\pars)\,\textrm{d}t
    -\gamma \moms\,\textrm{d}t
    + \sqrt{2 \gamma T} \, \mathbf{M}^{1/2} \,\textrm{d}\mathbf{W}.\label{eqn:langevin2}
\end{align}
Here $U(\pars)$ is the \emph{posterior energy} defined in~(\ref{eqn:U}), and $T > 0$ is the \emph{temperature}.
We use $\mathbf{W}$ to denote a standard multivariate Wiener process, which we can loosely understand as a generalized Gaussian distribution~\citep{sarkka2019appliedsde,leimkuhler2016molecular}.
The \emph{mass matrix} $\mathbf{M}$ is a preconditioner, and if we use no preconditioner then
$\mathbf{M}=I$, such that all $\mathbf{M}$-related terms vanish from the equations.
The \emph{friction} parameter $\gamma > 0$ controls both the strength of coupling between the moments $\moms$ and parameters $\pars$ as well as the amount of injected noise~\citep{langevin1908,leimkuhler2016molecular}.
For any friction $\gamma > 0$ the SDE~(\ref{eqn:langevin1}--\ref{eqn:langevin2}) has the same limiting distribution, but the choice of friction \emph{does} affect the speed of convergence to this distribution.
Simulating the continuous Langevin SDE~(\ref{eqn:langevin1}--\ref{eqn:langevin2}) produces a trajectory distributed according to $\exp(-U(\pars)/T)$ and the Bayes posterior is recovered for $T=1$.

\subsection{Stochastic Gradient MCMC (SG-MCMC)}\label{sec:sgmcmc}
Bayesian inference now corresponds to simulating the above SDE~(\ref{eqn:langevin1}--\ref{eqn:langevin2}) and this requires numerical discretization.
For efficiency \emph{stochastic gradient Markov chain Monte Carlo} (SG-MCMC) methods further approximate $\nabla_{\pars} U(\pars)$ with a minibatch gradient~\citep{welling2011bayesian,chen2014stochastichmc}.
For a minibatch $B \subset \{1,2,\dots,n\}$ we
first compute the minibatch average gradient $\tilde{G}(\pars)$,
\begin{equation}
\nabla_{\pars} \tilde{G}(\pars) := -\frac{1}{|B|} \sum_{i \in B} \nabla_{\pars} \log p(y_i|x_i,\pars)
- \frac{1}{n} \nabla_{\pars} \log p(\pars),
\label{eqn:Ggrad}
\end{equation}
and approximate $\nabla_{\pars} U(\pars)$ with the unbiased estimate $\nabla_{\pars} \tilde{U}(\pars) = n \nabla_{\pars} \tilde{G}(\pars)$. %
Here $|B|$ is the minibatch size and $n$ is the training set size;
in particular, note that the log prior scales with $1/n$ regardless of the batch size.

The SDE~(\ref{eqn:langevin1}--\ref{eqn:langevin2}) is defined in continuous time ($\textrm{d}t$), and in order to
solve the dynamics numerically we have to discretize the time domain~\citep{sarkka2019appliedsde}.
In this work we use a simple first-order symplectic Euler discretization,~\citep{leimkuhler2016molecular}, as first proposed for~(\ref{eqn:langevin1}--\ref{eqn:langevin2}) by~\cite{chen2014stochastichmc}.
Recent work has used more sophisticated discretizations,~\citep{chen2015convergencesgmcmc,shang2015covariancethermostat,heber2019tati,heek2019sgmcmc}.
Applying the symplectic Euler scheme to~(\ref{eqn:langevin1}--\ref{eqn:langevin2}) gives the discrete time update equations,
\begin{align}
\moms^{(t)} &= (1-h\gamma) \, \moms^{(t-1)} - h n \nabla_{\pars} \tilde{G}(\pars^{(t-1)})\label{eqn:symeuler-moms1}\\
& \quad + \sqrt{2 \gamma h T} \, \mathbf{M}^{1/2} \, \rands^{(t)},\label{eqn:symeuler-moms2}\\
\pars^{(t)} &= \pars^{(t-1)} + h \, \mathbf{M}^{-1} \moms^{(t)},\label{eqn:symeuler-pars}
\end{align}
where $\rands^{(t)} \sim \mathcal{N}_d(0, I_d)$ is a standard Normal vector.

In (\ref{eqn:symeuler-moms1}--\ref{eqn:symeuler-pars}), the parameterization is in terms of step size $h$ and friction $\gamma$.
These quantities are different from typical SGD parameters.
In Appendix~\ref{sec:sgmcmc-deeplearning-parameterization} we establish an exact correspondence between the SGD learning rate $\ell$ and momentum decay parameters $\beta$ and SG-MCMC parameters.
For the symplectic Euler discretization of Langevin dynamics, we derive this relationship as $h := \sqrt{\ell/n}$, and $\gamma := (1-\beta) \sqrt{n/\ell}$, where $n$ is the total training set size.

\subsection{Accurate SG-MCMC Simulation}\label{sec:accurate-sgmcmc}
In practice there remain two sources of error when following the dynamics~(\ref{eqn:symeuler-moms1}--\ref{eqn:symeuler-pars}):
\begin{compactitem}
\item \emph{Minibatch noise}: $\nabla_{\pars} \tilde{U}(\pars)$ is an unbiased estimate of $\nabla_{\pars} U(\pars)$ but contains additional estimation variance.
\item \emph{Discretization error}: we incur error by following a continuous-time path~(\ref{eqn:langevin1}--\ref{eqn:langevin2}) using discrete steps~(\ref{eqn:symeuler-moms1}--\ref{eqn:symeuler-pars}).
\end{compactitem}
We use two methods to reduce these errors: \emph{preconditioning} and \emph{cyclical time stepping}.

\textbf{Layerwise Preconditioning.}
Preconditioning through a choice of matrix $\mathbf{M}$ is a common way to improve the behavior of optimization methods.
\citet{li2016preconditionedsgld} and~\citet{ma2015complete} proposed preconditioning for SG-MCMC methods, and in the context of molecular dynamics the use of a matrix $\mathbf{M}$ has a long tradition as well,~\citep{leimkuhler2016molecular}.
Li's proposal is an adaptive preconditioner inspired by RMSprop,~\citep{tieleman2012rmsprop}.
Unfortunately, using the discretized Langevin dynamics with a preconditioner $\mathbf{M}(\pars)$ that depends on $\pars$ compromises the correctness of the dynamics.\footnote{\citet{li2016preconditionedsgld} derives the required correction term, which however is expensive to compute and omitted in practice.}
We propose a simpler preconditioner that limits the frequency of adaptating $\mathbf{M}$:
after a number of iterations we estimate a new preconditioner $\mathbf{M}$ using a small number of batches, say 32, but without updating any model parameters.
This preconditioner then remains fixed for a number of iterations, for example, the number of iterations it takes to visit the training set once, i.e. one epoch.
We found this strategy to be highly effective at improving simulation accuracy.
For details, please see Appendix~\ref{sec:precond}.

\textbf{Cyclical time stepping.}
The second method to improve simulation accuracy is to decrease the discretization step size $h$.
\citet{chen2015convergencesgmcmc} studied the consequence of both minibatch noise and discretization error on simulation accuracy and showed that the overall simulation error goes to zero for $h \searrow 0$.
While lowering the step size $h$ to a small value would also make the method slow, recently \citet{zhang2019cyclicalsgmcmc} propose to perform \emph{cycles} of iterations $t=1,2,\dots$ with a high-to-low step size schedule $h_0 \, C(t)$ described by an initial step size $h_0$ and a function $C(t)$ that starts at $C(1)=1$ and has $C(L) = 0$ for a cycle length of $L$ iterations.
Such cycles retain fast simulation speed in the beginning while accepting simulation error.
Towards the end of each cycle however, a small step size ensures an accurate simulation.
We use the cosine schedule from~\citep{zhang2019cyclicalsgmcmc} for $C(t)$, see Appendix~\ref{sec:model-details}.

We integrate these two techniques together into a practical SG-MCMC procedure, Algorithm~\ref{alg:symeuler}.
When no preconditioning and no cosine schedule is used ($\mathbf{M}=I$ and $C(t)=1$ in all iterations) and $T(t)=0$ this algorithm is equivalent to \emph{Tensorflow}'s SGD with momentum (Appendix~\ref{sec:sgd}).

\begin{figure}[!tb]
\vspace{-0.2cm}%
\centering
\begin{minipage}{\columnwidth}%
\begingroup
\removelatexerror
\begin{algorithm2e}[H]
\DontPrintSemicolon
\SetKwComment{tcn}{}{}%
\SetKwFunction{FnSymEulerCanonical}{SymEulerSGMCMC\!}
\SetKw{KwYield}{yield}
\SetKwProg{Fn}{Function}{}{}
\Fn{\FnSymEulerCanonical{$\tilde{G}$, $\pars^{(0)}$, $\ell$, $\beta$,\! $n$,\! $T$}}{
  \KwInput{
    $\tilde{G}: \Theta \to \mathbb{R}$ mean energy function estimate;
    $\pars^{(0)} \in \mathbb{R}^d$ initial parameter;
    $\ell > 0$ learning rate;
    $\beta \in [0,1)$ momentum decay;
    $n$ total training set size;
    $T(t) \geq 0$ temperature schedule
  }
  \KwOutput{Sequence $\pars^{(t)}$, $t=1,2,\dots$}
  $h_0 \leftarrow \sqrt{\ell / n}$ \tcp*{SDE time step}
  $\gamma \leftarrow (1-\beta) \sqrt{n/\ell}$ \tcp*{friction}
  Sample $\moms^{(0)} \sim \mathcal{N}_d(0, I_d)$\;%
  $\mathbf{M} \leftarrow I$ \tcp*{Initial $\mathbf{M}$}
  \For{ $t=1,2,\dots$ }{
    \If{ new epoch }{
        $\moms_c \leftarrow \mathbf{M}^{-1/2} \, \moms^{(t-1)}$\;%
        $\mathbf{M} \leftarrow \texttt{EstimateM}(\tilde{G},\pars^{(t-1)})$\;
        $\moms^{(t-1)} \leftarrow \mathbf{M}^{1/2} \, \moms_c$\;%
    }
    $h \leftarrow C(t) \, h_0$ \tcp*{Cyclic modulation}
    Sample $\rands^{(t)} \sim \mathcal{N}_d(0, I_d)$ \tcp*{noise}
    $\moms^{(t)} \leftarrow (1-h\gamma) \, \moms^{(t-1)} - h n \nabla_{\pars} \tilde{G}(\pars^{(t-1)}) + \!\sqrt{2 \gamma h T(t)} \, \mathbf{M}^{1/2} \, \rands^{(t)}$\tcn*{}\label{line:symeuler-moms}
    $\pars^{(t)} \leftarrow \pars^{(t-1)} + h \, \mathbf{M}^{-1} \moms^{(t)}$\label{line:symeuler-pars}\;%
    \If{ end of cycle }{
        \KwYield{$\pars^{(t)}$} \tcp*{Parameter sample}
    }
  }
}
\caption{Symplectic Euler Langevin scheme.}%
\label{alg:symeuler}%
\end{algorithm2e}%
\endgroup
\end{minipage}%
\vspace{-0.5cm}%
\end{figure}%

Coming back to the \phenomenon~effect, what could explain the poor performance at temperature $T=1$?
With our Bayesian hearts, there are only three possible areas to examine: the inference, the prior, or the likelihood function.

\section{Inference: Is it Accurate?}\label{sec:inference-accuracy}
Both the Bayes posterior and the cooled posteriors are all intractable.
Moreover, it is plausible that the high-dimensional posterior landscape of a deep network may lead to difficult-to-simulate SDE dynamics~(\ref{eqn:langevin1}--\ref{eqn:langevin2}).
Our approximate SG-MCMC inference method further has to deal with minibatch noise and produces only a finite sample approximation to the predictive integral~(\ref{eqn:postpred}).
Taken together, could the \phenomenon~effect arise from a poor inference accuracy?
\subsection{Hypothesis: Inaccurate SDE Simulation}
\hypothesis{Inaccurate SDE Simulation Hypothesis}{the SDE~(\ref{eqn:langevin1}--\ref{eqn:langevin2}) is poorly simulated.}
To gain confidence that our SG-MCMC method simulates the posterior accurately, we introduce diagnostics that previously have not been used in the SG-MCMC context:
\begin{compactitem}
\item \textbf{Kinetic temperatures} (Appendix~\ref{sec:kinetic-temp}): we report per-variable statistics derived from the moments $\moms$.  For these so called \emph{kinetic temperatures} we know the exact sampling distribution under Langevin dynamics and compute their 99\% confidence intervals.
\item \textbf{Configurational temperatures} (Appendix~\ref{sec:conf-temp}): we report per-variable statistics derived from $\langle \pars, \nabla_\pars U(\pars)\rangle$.  For these \emph{configurational temperatures} we know the expected value under Langevin dynamics.
\end{compactitem}
We propose to use these diagnostics to assess simulation accuracy of SG-MCMC methods.
We introduce the diagnostics and our new results in detail in Appendix~\ref{sec:diagnostics}.

\experiment{Inference Diagnostics Experiment:}
In Appendix~\ref{sec:simablation} we report a detailed study of simulation accuracy for both models.
This study reports accurate simulation for both models when both preconditioning and cyclic time stepping are used.
We can therefore with reasonably high confidence rule out a poor simulation of the SDE.
All remaining experiments in this paper also pass the simulation accuracy diagnostics.

\begin{figure}[!t]
\vspace{-0.2cm}%
\center{\includegraphics[width=\columnwidth]{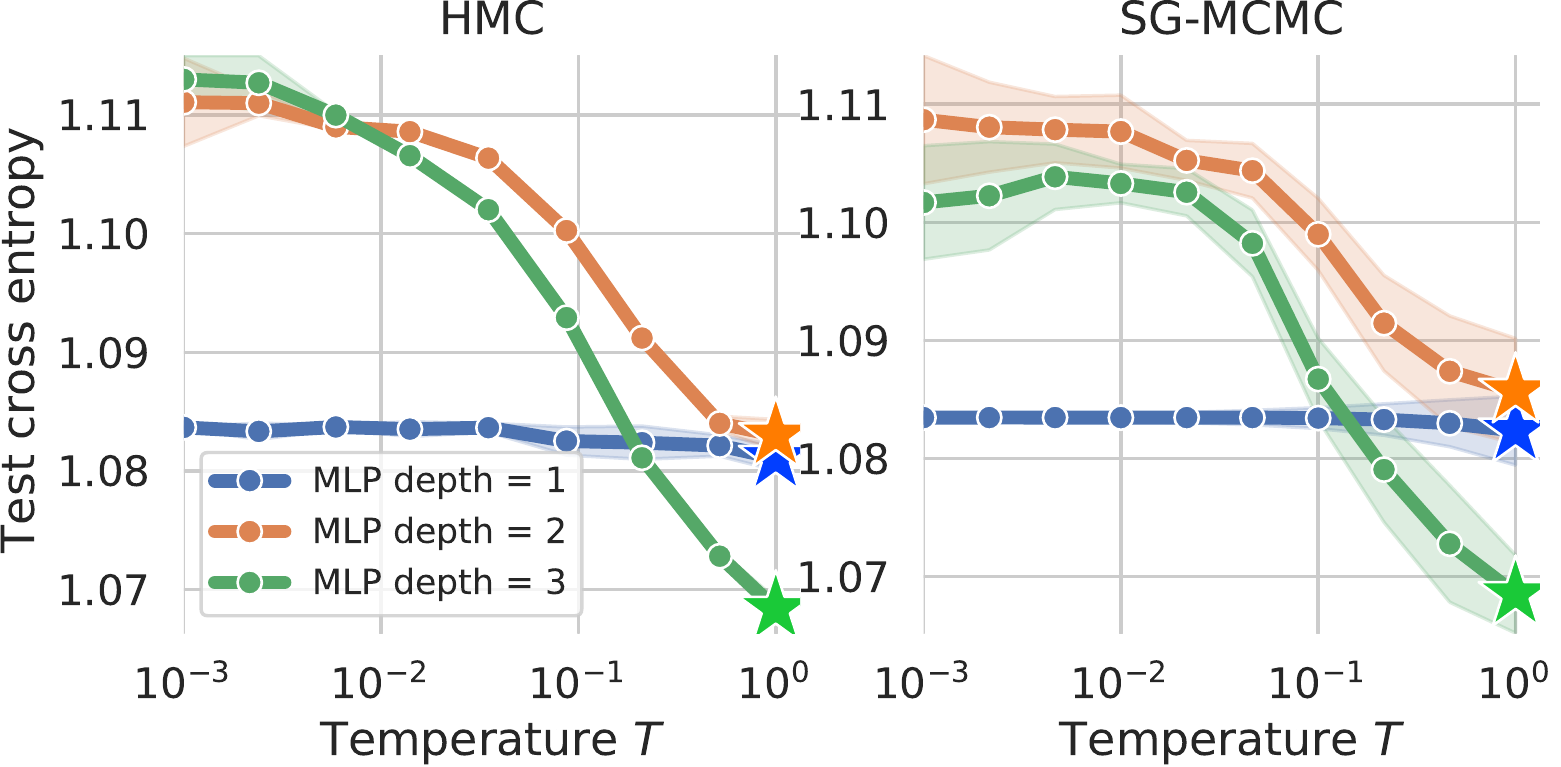}}%
\vspace{-0.3cm}%
\caption{HMC (left) agrees closely with SG-MCMC (right) for synthetic data on multilayer perceptrons.  A star indicates the optimal temperature for each model: for the synthetic data sampled from the prior there are no cold posteriors and both sampling methods perform best at $T=1$.}%
\label{fig:sgmcmc_vs_hmc}%
\vspace{-0.25cm}%
\end{figure}

\subsection{Hypothesis: Biased SG-MCMC}\label{sec:hyp_biased_sgmcmc}

\hypothesis{Biased SG-MCMC Hypothesis}{Lack of accept/reject Metropolis-Hastings corrections in SG-MCMC introduces bias.}
In Markov chain Monte Carlo it is common to use an additional accept-reject step that corrects for bias in the sampling procedure.  For MCMC applied to deep learning this correction step is too expensive and therefore omitted in SG-MCMC methods, which is valid for small time steps only,~\citep{chen2015convergencesgmcmc}.
If accept-reject is computationally feasible the resulting procedure is called \emph{Hamiltonian Monte Carlo} (HMC)~\citep{neal2011mcmc, betancourt2015hamiltonian,duane1987hybrid,hoffman2014no}.
Because it provides unbiased simulation, we can consider HMC the \emph{gold standard},~\citep{neal1995bayesianneuralnetworks}.
We now compare gold standard HMC against SG-MCMC on a small example where comparison is feasible.
We provide details of our HMC setup in Appendix~\ref{sec:practical_usage_hmc}.

\experiment{HMC Experiment:}
we construct a simple setup using a multilayer perceptron (MLP) where by construction $T=1$ is optimal;
such Bayes optimality must hold in expectation if the data is generated by the prior and model that we use for inference,~\citep{berger1985statisticaldecisiontheory}.
Thus, we can ensure that if the cold posterior effect is observed it must be due to a problem in our inference method.
We perform all inference without minibatching ($|B|=n$) and test MLPs of varying number of one to three layers, ten hidden units each, and using the ReLU activation.
As HMC implementation we use \texttt{tfp.mcmc.HamiltonianMonteCarlo} from \emph{Tensorflow Probability}~\citep{dillon2017tensorflow,lao2020tfpmcmc}:
Details for our data and HMC are in Appendix~\ref{sec:generate_data_from_mlp}--\ref{sec:practical_usage_hmc}. 

In Figure~\ref{fig:sgmcmc_vs_hmc} the SG-MCMC results agree very well with the HMC results with optimal predictions at $T=1$, i.e. no cold posteriors are present.
For the cases tested we conclude that SG-MCMC is almost as accurate as HMC and the lack of accept-reject correction cannot explain cold posteriors. 
Appendix~\ref{sec:practical_usage_hmc} further shows that SG-MCMC and HMC are in good agreement when inspecting the KL divergence of their resulting predictive distributions.

\subsection{Hypothesis: Stochastic Gradient Noise}
\hypothesis{Minibatch Noise Hypothesis}{gradient noise from minibatching causes inaccurate sampling at $T=1$.}

Gradient noise due to minibatching can be heavy-tailed and non-Gaussian even for large batch sizes,~\citep{simsekli2019minibatchnoise}.
Our SG-MCMC method is only justified if the effect of noise will diminish for small time steps.
We therefore study the influence of batch size on predictive performance through the following experiment.

\experiment{Batchsize Experiment:}
we repeat the original ResNet-20/CIFAR-10 experiment at different temperatures for batch sizes in $\{32,64,128,256\}$ and study the variation of the predictive performance as a function of batch size.
Figure~\ref{fig:resnet-batchsize} and Figure~\ref{fig:cnnlstm-batchsize} show that while there is a small variation between different batch sizes $T < 1$ remains optimal for all batch sizes.
Therefore minibatch noise alone cannot explain the observed poor performance at $T=1$.

For both ResNet and CNN-LSTM the best cross-entropy is achieved by the smallest batch size of 32 and 16, respectively.
The smallest batch size has the \emph{largest} gradient noise.
We can interpret this noise as an additional heat source that increases the effective simulation temperature.
However, the noise distribution arising from minibatching is anisotropic, \citep{zhu2019anisotropicsgdnoise}, and this could perhaps aid generalization.  We will not study this hypothesis further here.

\begin{figure}[!t]
\vspace{-0.2cm}%
\center{\includegraphics[width=\columnwidth]{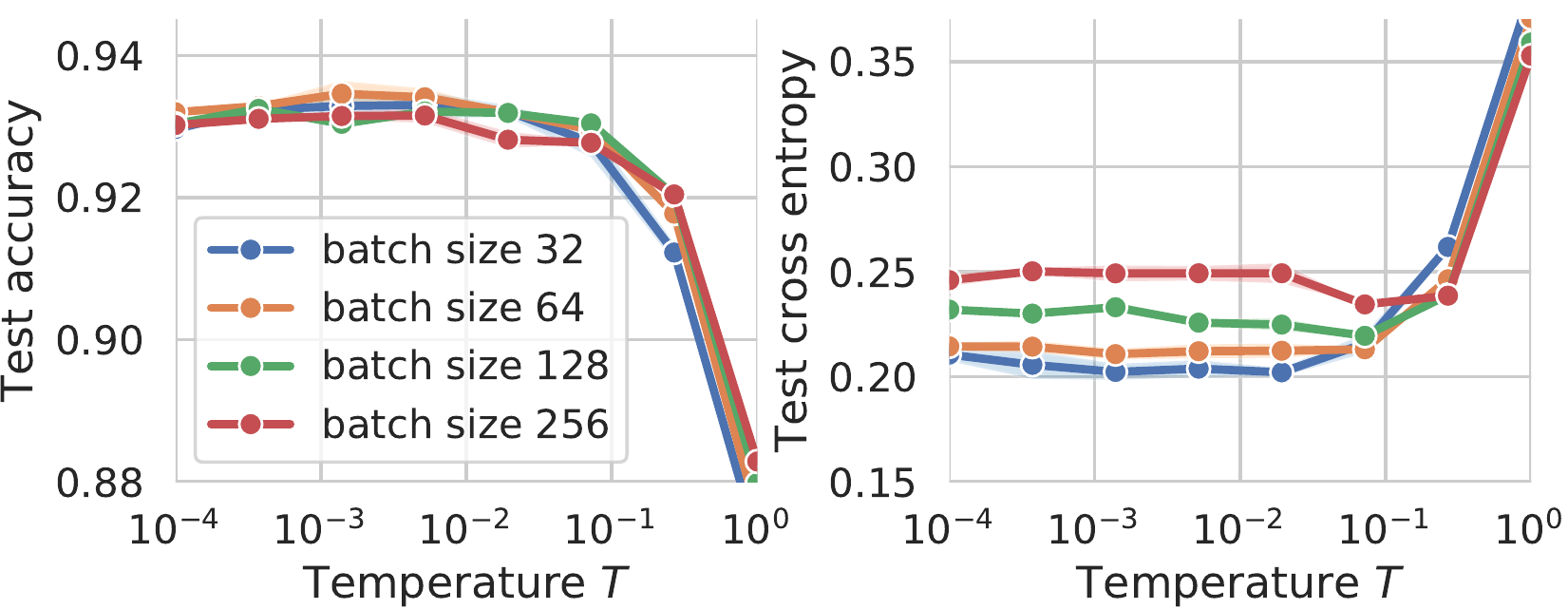}}%
\vspace{-0.3cm}%
\caption{Batch size dependence of the ResNet-20/CIFAR-10 ensemble performance, reporting mean and standard error (3 runs): for all batch sizes the optimal predictions are obtained for $T < 1$.}%
\label{fig:resnet-batchsize}%
\vspace{-0.35cm}%
\end{figure}

\begin{figure}[!t]
\vspace{-0.2cm}%
\center{\includegraphics[width=\columnwidth]{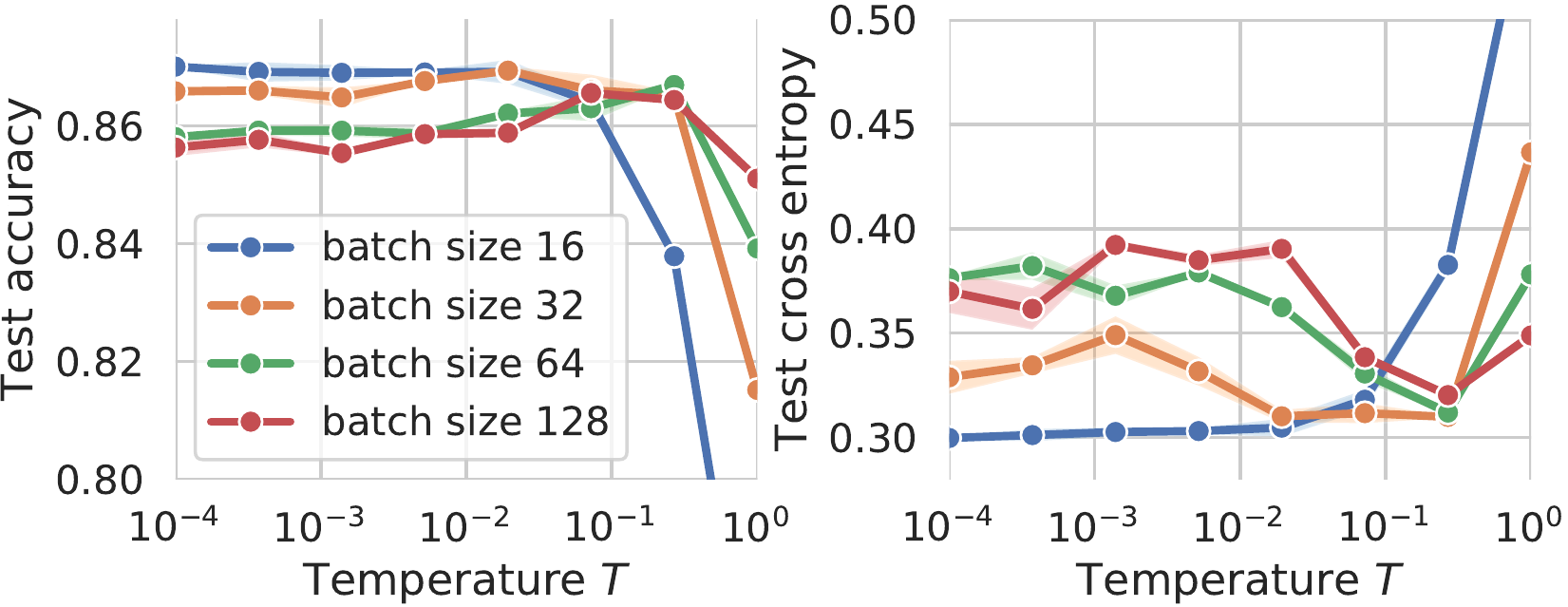}}%
\vspace{-0.3cm}%
\caption{Batch size dependence of the CNN-LSTM/IMDB ensemble performance, reporting mean and standard error (3 runs):
for all batch sizes, the optimal performance is achieved at $T < 1$.}%
\label{fig:cnnlstm-batchsize}%
\vspace{-0.35cm}%
\end{figure}

\subsection{Hypothesis: Bias-Variance Trade-off}
\hypothesis{Bias-variance Tradeoff Hypothesis}{For $T=1$ the posterior is diverse and there is high variance between model predictions.  For $T \ll 1$ we sample nearby modes and reduce prediction variance but increase bias; the variance dominates the error and reducing variance ($T \ll 1$) improves predictive performance.}

If this hypothesis were true then simply collecting more ensemble members, $S \to \infty$, would reduce the variance to arbitrary small values and thus fix the poor predictive performance we observe at $T=1$.  Doing so would require running our SG-MCMC schemes for longer---potentially for much longer.
We study this question in detail in Appendix~\ref{sec:biasvar-appendix} and conclude by an asymptotic analysis that the amount of variance cannot explain cold posteriors.

\section{Why Could the Bayes Posterior be Poor?}\label{sec:why_true_posterior_could_be_poor}
With some confidence in our approximate inference procedure what are the remaining possibilities that could explain the cold posterior effect?
The remaining two places to look at are the likelihood function and the prior.

\subsection{Problems in the Likelihood Function?}\label{sec:likelihood}
For Bayesian deep learning we use the same likelihood function $p(y|x,\pars)$ as we use for SGD.
Therefore, because the same likelihood function works well for SGD it appears an unlikely candidate to explain the cold posterior effect.
However, current deep learning models use a number of techniques---such as data augmentation, dropout, and batch normalization---that are not formal likelihood functions.
This observations brings us to the following hypothesis.

\hypothesis{Dirty Likelihood Hypothesis}{Deep learning practices that violate the likelihood principle (batch normalization, dropout, data augmentation) cause deviation from the Bayes posterior.}

In Appendix~\ref{sec:dirtylikelihood} we give a theory of ``\emph{Jensen posteriors}'' which describes the likelihood-like functions arising from modern deep learning techniques.
We report an experiment (Appendix~\ref{sec:frn-experiment}) that---while slightly inconclusive---demonstrates that cold posteriors remain when a clean likelihood is used in a suitably modified ResNet model; the CNN-LSTM model already had a clean likelihood function.

\begin{figure*}[!t]%
\vspace{-0.2cm}%
\centering%
\begin{minipage}{0.65\textwidth}%
\includegraphics[width=0.49\textwidth]{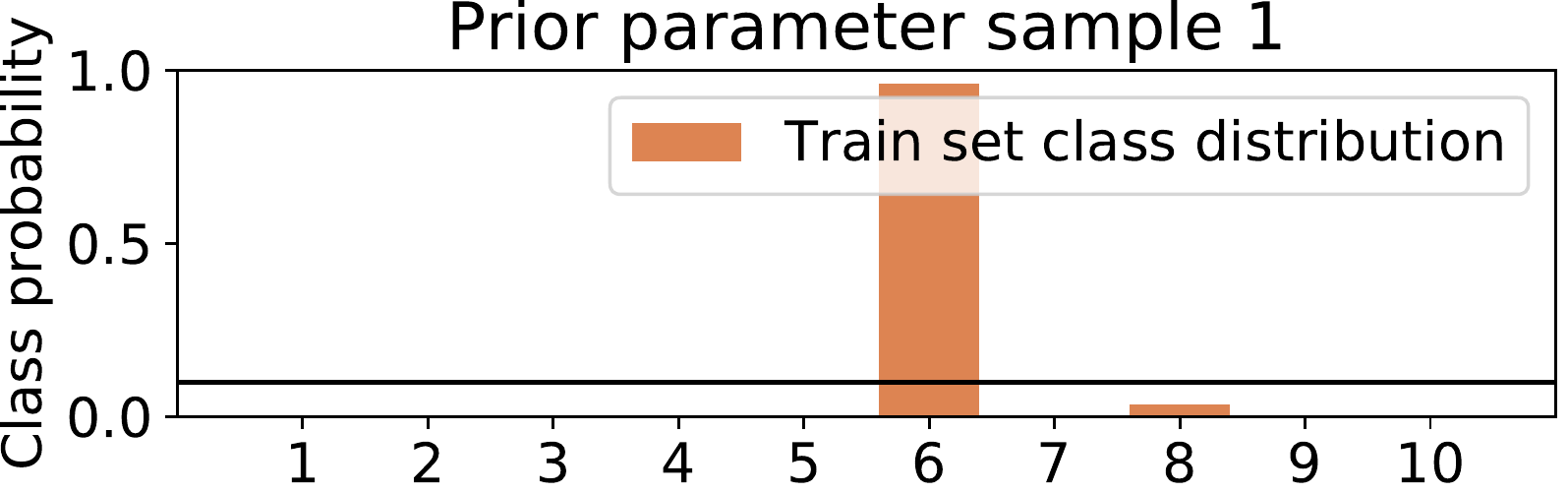}%
\hfill%
\includegraphics[width=0.49\textwidth]{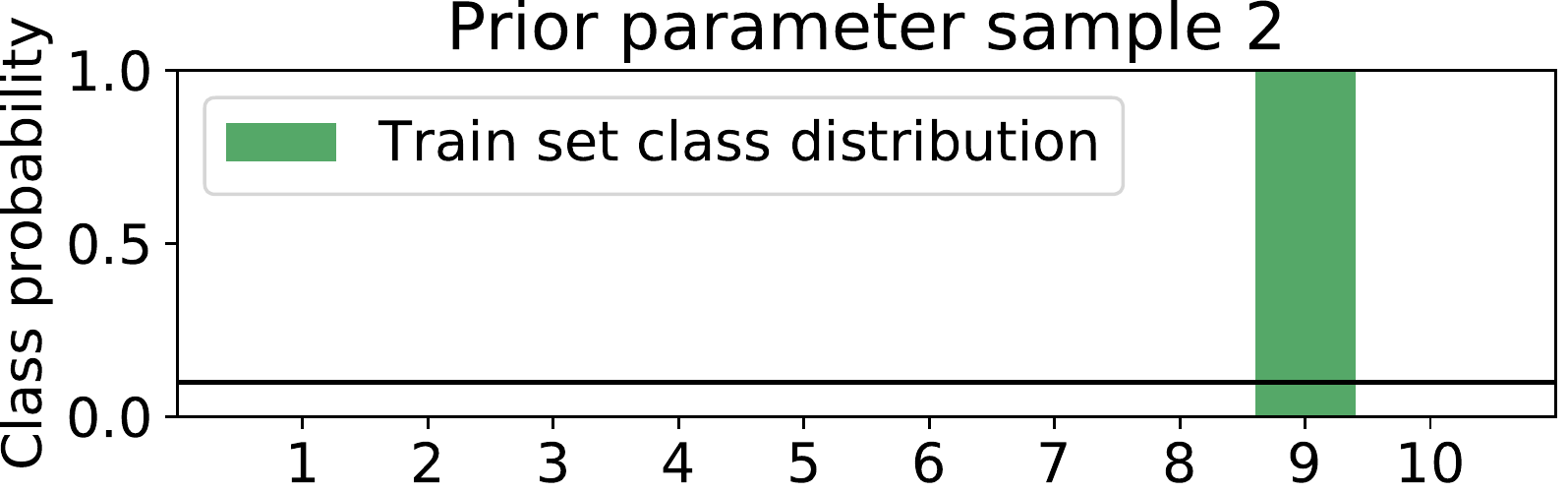}%
\vspace{-0.25cm}%
\captionof{figure}{ResNet-20/CIFAR-10 typical prior predictive distributions for 10 classes under a $\mathcal{N}(0,I)$ prior averaged over the entire training set, $\mathbb{E}_{x \sim p(x)}[p(y|x,\pars^{(i)})]$.  Each plot is for one sample $\pars^{(i)} \sim \mathcal{N}(0,I)$ from the prior.
Given a sample $\pars^{(i)}$ the average training data class distribution is highly concentrated around the same classes for all $x$.}%
\label{fig:resnet-priorsamples}%
\end{minipage}%
\hfill%
\begin{minipage}{0.32\textwidth}%
\includegraphics[width=\textwidth]{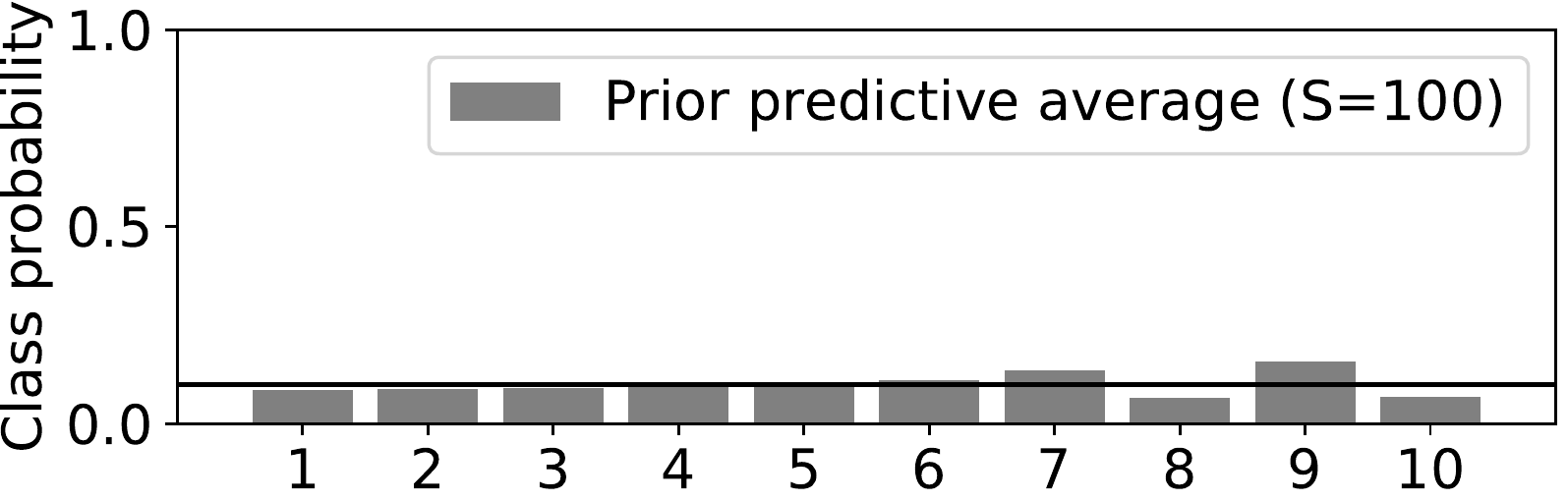}%
\vspace{-0.25cm}%
\captionof{figure}{ResNet-20/CIFAR-10 prior predictive $\mathbb{E}_{x \sim p(x)}[\mathbb{E}_{\pars \sim p(\pars)}[p(y|x,\pars)]]$ over 10 classes, estimated using $S=100$ prior samples $\pars^{(i)}$ and all training images.}%
\label{fig:resnet-priorpredictive}%
\end{minipage}%
\vspace{-0.4cm}%
\end{figure*}

\subsection{Problems with the Prior $p(\pars)$?}\label{sec:prior}
So far we have used a simple Normal prior, $p(\pars)=\mathcal{N}(0,I)$, as was done in prior work~\citep{zhang2019cyclicalsgmcmc,heek2019sgmcmc,ding2014bayesian,li2016preconditionedsgld,zhang2018noisynaturalgradient}.
But is this a good prior?

One could hope, that perhaps with an informed and structured model architecture, a simple prior could be sufficient in placing prior beliefs on suitable functions, as argued by~\citet{wilson2019bayesian}.
While plausible, we are mildly cautious because there are known examples where innocent looking priors have turned out to be unintentionally highly informative.\footnote{A shocking example in the Dirichlet-Multinomial model is given by~\citet{nemenman2002entropy}.  Importantly the unintended effect of the prior was not recognized when the model was originally proposed by~\citet{wolpert1995estimatingentropy}.}
Therefore, with the cold posterior effect having a track record in the literature, perhaps $p(\pars)=\mathcal{N}(0,I)$ could have similarly unintended effects of placing large prior mass on undesirable functions.
This leads us to the next hypothesis.

\hypothesis{Bad Prior Hypothesis}{The current priors used for BNN parameters are inadequate, unintentionally informative, and their effect becomes stronger with increasing model depths and capacity.}

To study the quality of our prior, we study typical functions obtained by sampling from the prior, as is good practice in model criticism,~\citep{gelman2013bayesiandataanalysis}.

\experiment{Prior Predictive Experiment:}
for our ResNet-20 model we generate samples $\pars^{(i)} \sim p(\pars) = \mathcal{N}(0,I)$
and look at the induced predictive distribution
$\mathbb{E}_{x \sim p(x)}[p(y|x,\pars^{(i)})]$ for each parameter sample, using the real CIFAR-10 training images.
From Figure~\ref{fig:resnet-priorsamples} we see that typical prior draws produce concentrated class distributions, indicating that the $\mathcal{N}(0,I)$ distribution is a poor prior for the ResNet-20 likelihood.
From Figure~\ref{fig:resnet-priorpredictive} we can see that the average predictions obtained from such concentrated functions remain close to the uniform class distribution.
Taken together, from a subjective Bayesian view $p(\pars)=\mathcal{N}(0,I)$ is a \emph{poor prior}: typical functions produced by this prior place a high probability the same few classes for all $x$.
In Appendix~\ref{sec:prior-predictive-gpgaussian} we carry out another prior predictive study using He-scaling priors,~\citep{he2015prelu}, which leads to similar results.

\experiment{Prior Variance $\sigma$ Scaling Experiment:}
in the previous experiment we found that the standard Normal prior is poor. Can the Normal prior $p(\pars)=\mathcal{N}(0,\sigma)$ be fixed by using a more appropriate variance $\sigma$? For our ResNet-20 model we employ Normal priors of varying variances. Figure~\ref{fig:priorvariance} shows that the cold posterior effect is present for all variances considered. Further investigations for known scaling laws in deep networks is given in Appendix~\ref{sec:prior-predictive-gpgaussian}. The cold posterior effect cannot be resolved by using the right scaling of the Normal prior.

\experiment{Training Set Size $n$ Scaling Experiment:}
the posterior energy $U(\pars)$ in~(\ref{eqn:U}) sums over all $n$ data log-likelihoods but adds $\log p(\pars)$ only once.
This means that the influence of $\log p(\pars)$ vanishes at a rate of $1/n$ and thus the prior will exert its strongest influence for small $n$.
We now study what happens for small $n$ by comparing the Bayes predictive under a $\mathcal{N}(0,I)$ prior against performing SGD maximum a posteriori (MAP) estimation on the \emph{same} log-posterior.\footnote{For SGD we minimize $U(\pars)/n$.}

Figure~\ref{fig:resnetcifar10-trainsize} and Figure~\ref{fig:cnnlstm-trainsize} show the predictive performance for ResNet-20 on CIFAR-10 and CNN-LSTM on IMDB, respectively.
These results differ markedly between the two models and datasets:
for ResNet-20 / CIFAR-10 the Bayes posterior at $T=1$ degrades gracefully for small $n$, whereas SGD suffers large losses in test cross-entropy for small $n$.
For CNN-LSTM / IMDB predictions from the Bayes posterior at $T=1$ deteriorate quickly in both test accuracy and cross entropy.
In all these runs SG-MCMC and SGD/MAP work with the same $U(\pars)$ and the difference is between integration and optimization.
The results are inconclusive but somewhat implicate the prior in the cold posterior effect:
as $n$ becomes small there is an increasing difference between the cross-entropy achieved by the Bayes prediction and the SGD estimate, for large $n$ the SGD estimate performs better.

\begin{figure}[!t]
\center{\includegraphics[width=\columnwidth]{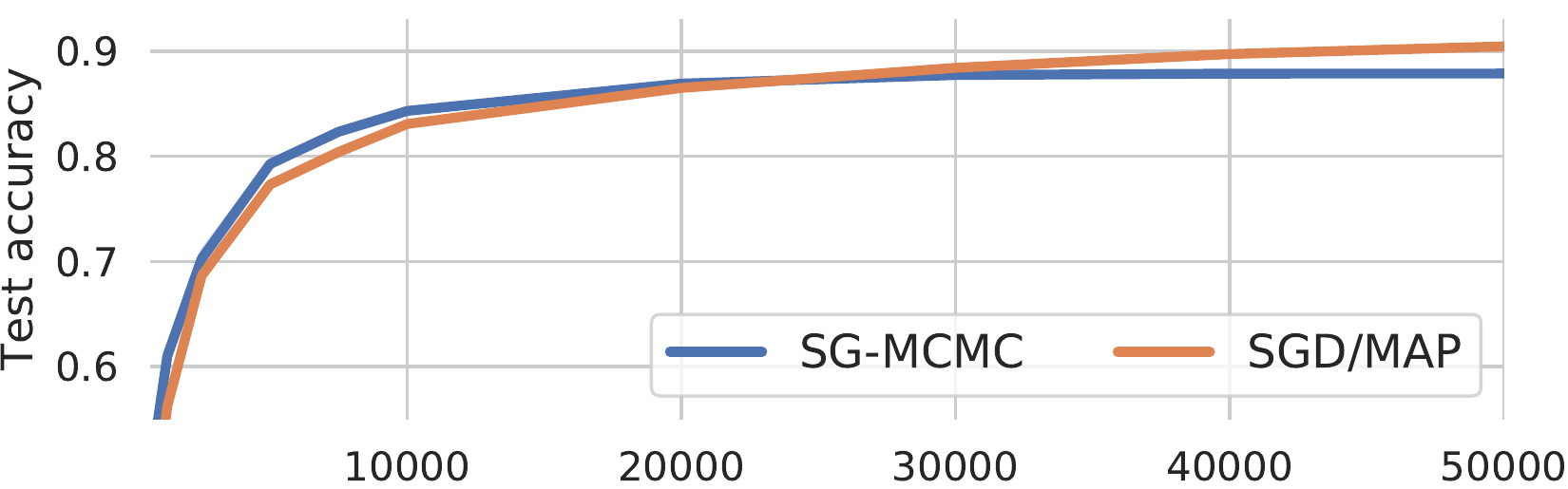}}%
\center{\includegraphics[width=\columnwidth]{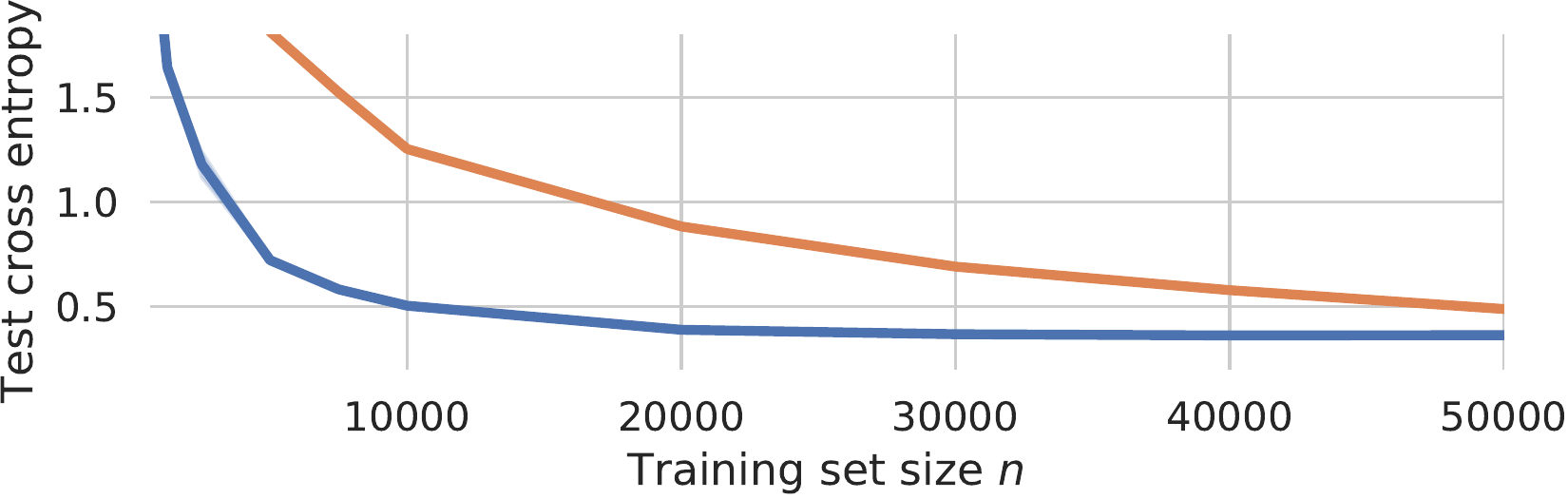}}%
\vspace{-0.3cm}%
\caption{ResNet-20/CIFAR-10 predictive performance as a function of training set size $n$.
The Bayes posterior ($T=1$) degrades gracefully as $n$ decreases, whereas SGD/MAP performs worse.}%
\label{fig:resnetcifar10-trainsize}%
\vspace{-0.3cm}%
\end{figure}

\begin{figure}[!t]
\center{\includegraphics[width=\columnwidth]{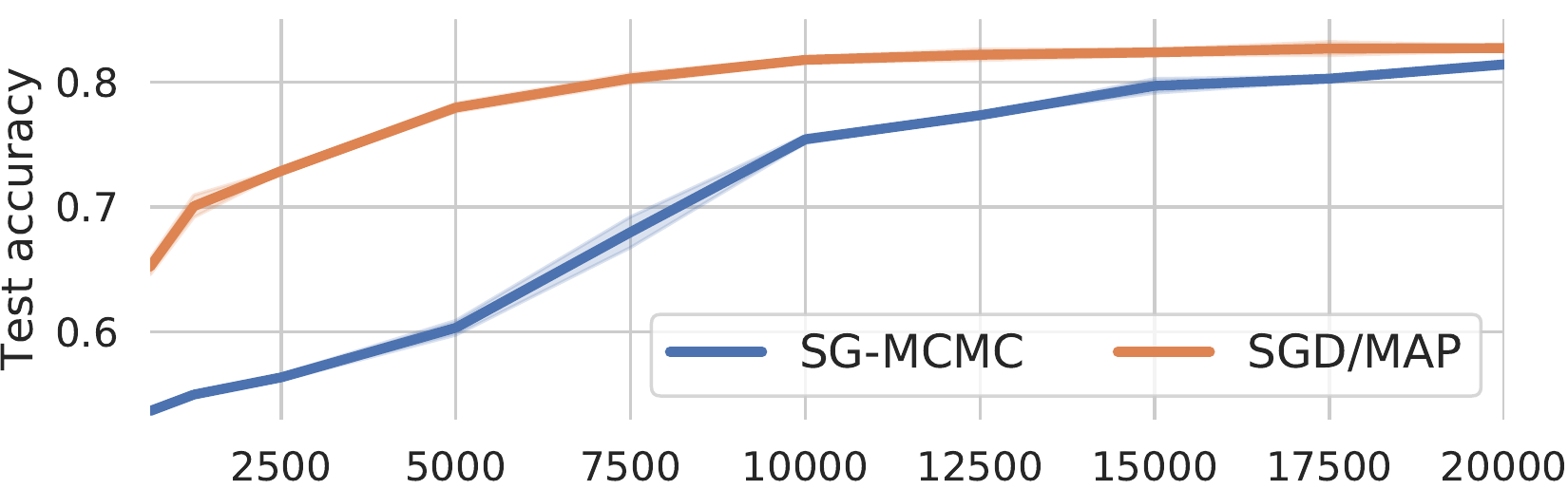}}%
\center{\includegraphics[width=\columnwidth]{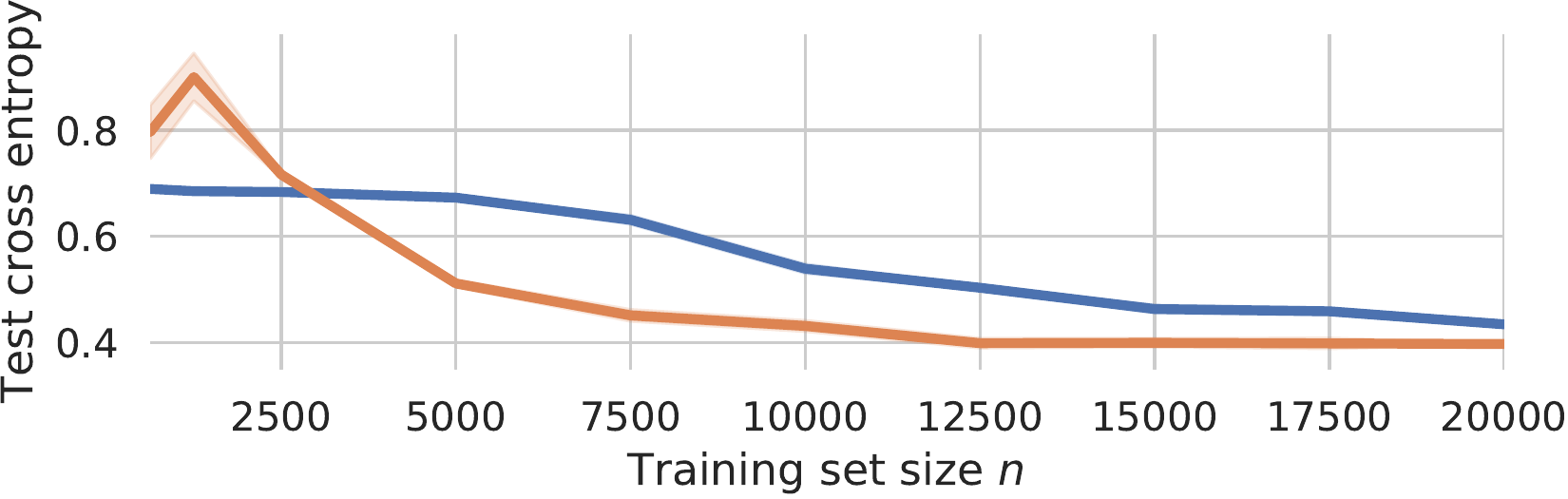}}%
\vspace{-0.3cm}%
\caption{CNN-LSTM/IMDB predictive performance as a function of training set size $n$.
The Bayes posterior ($T=1$) suffers more than the SGD performance, indicating a problematic prior.
}%
\label{fig:cnnlstm-trainsize}%
\vspace{-0.25cm}%
\end{figure}

\experiment{Capacity Experiment:}
we consider a MLP using a $\mathcal{N}(0,I)$ prior and study the relation of the network capacity to the cold posterior effect.
We train MLPs of varying depth (number of layers) and width (number of units per layer) at different temperatures on CIFAR-10.
Figure~\ref{fig:mlp-capacity} shows that for increasing capacity the cold posterior effect becomes more prominent.
This indicates a connection between model capacity and strength of the cold posterior effect.

\begin{figure}[!t]
\center{\includegraphics[width=0.5145\columnwidth]{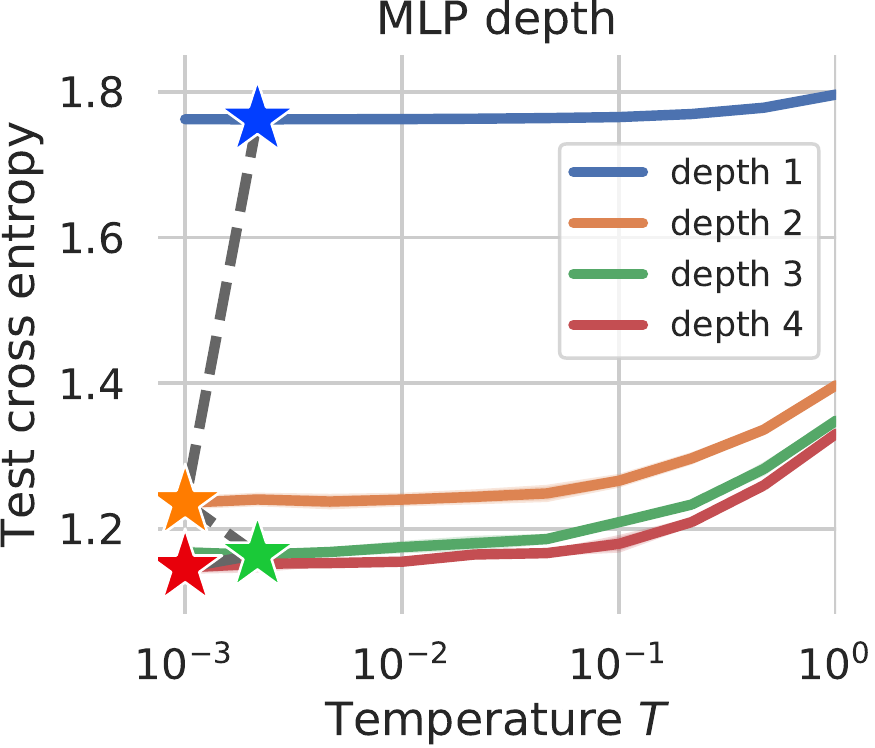}%
\includegraphics[width=0.4855\columnwidth]{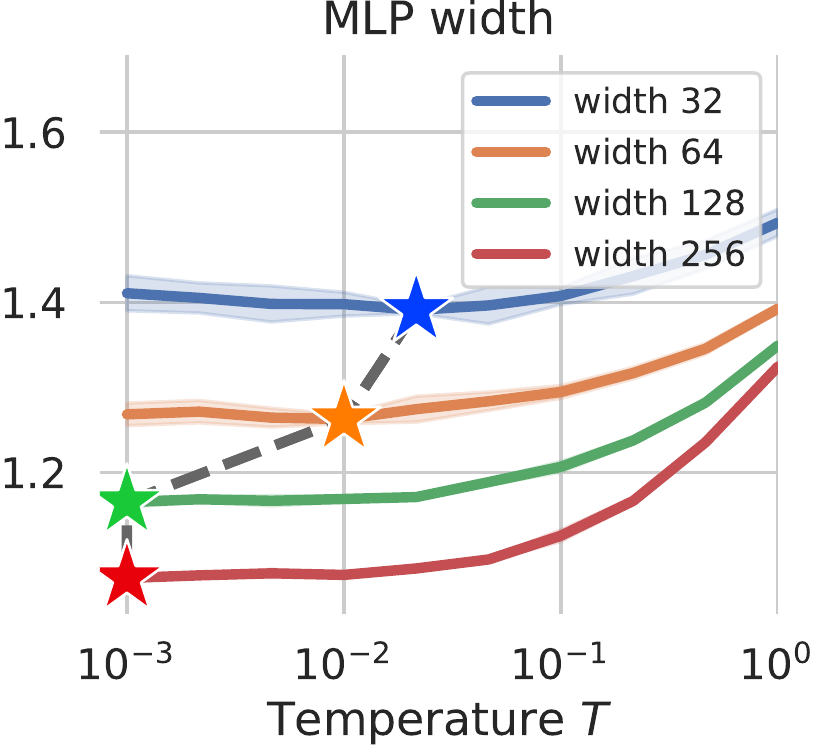}%
}%
\vspace{-0.3cm}%
\caption{MLP of different capacities (depth and width) on CIFAR-10. Left: we fix the width to 128 and vary the depth. Right: we fix the depth to 3 and vary the width.
Increasing capacity lowers the optimal temperature.}%
\label{fig:mlp-capacity}%
\end{figure}

\subsection{Inductive Bias due to SGD?}

\hypothesis{Implicit Initialization Prior in SGD}{The inductive bias from initialization is strong and beneficial for SGD but harmed by SG-MCMC sampling.}
Optimizing neural networks via SGD with a suitable initialization is known to have a beneficial inductive bias leading to good local optima,~\citep{masters2018batchsize,mandt2017stochastic}. Does SG-MCMC perform worse due to decreasing the influence of that bias?
We address this question by the following experiment.
We first run SGD until convergence, then switch over to SG-MCMC sampling for 500 epochs (10 cycles), and finally switch back to SGD again. Figure~\ref{fig:temperature_010} shows that SGD initialized by the last model of the SG-MCMC sampling dynamics recovers the same performance as vanilla SGD. This indicates that the beneficial initialization bias for SGD is not destroyed by SG-MCMC. Details can be found in Appendix~\ref{sec:implic_init-appendix}.

\begin{figure}[!t]
\vspace{-0.25cm}%
\center{\includegraphics[width=\columnwidth]{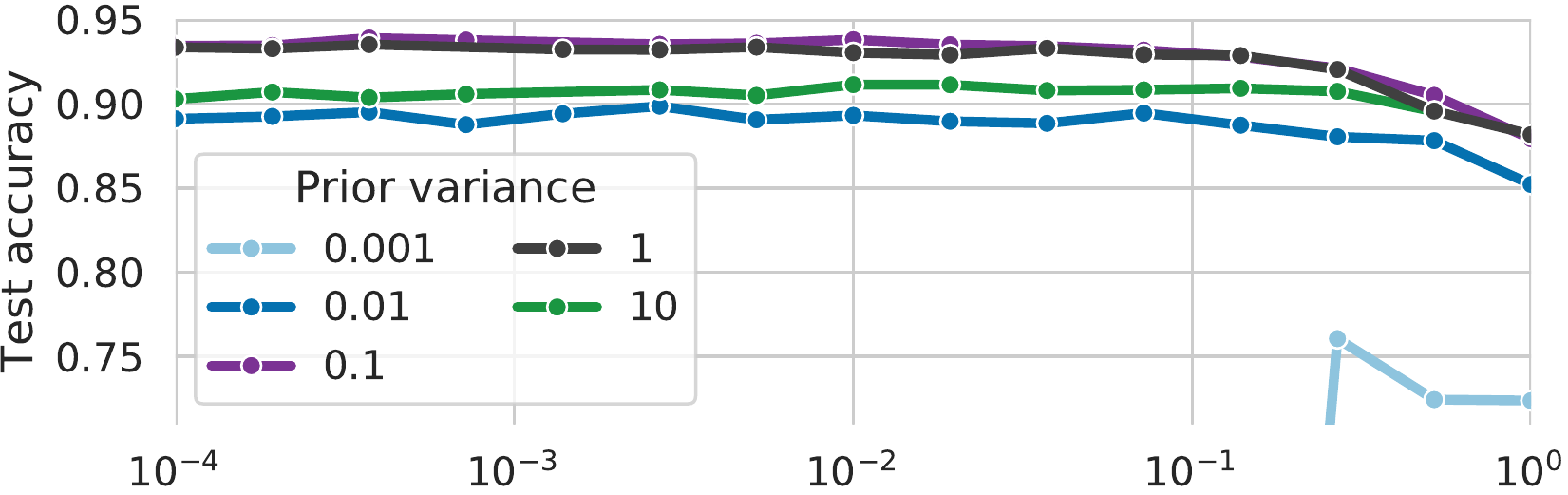}}%
\center{\includegraphics[width=\columnwidth]{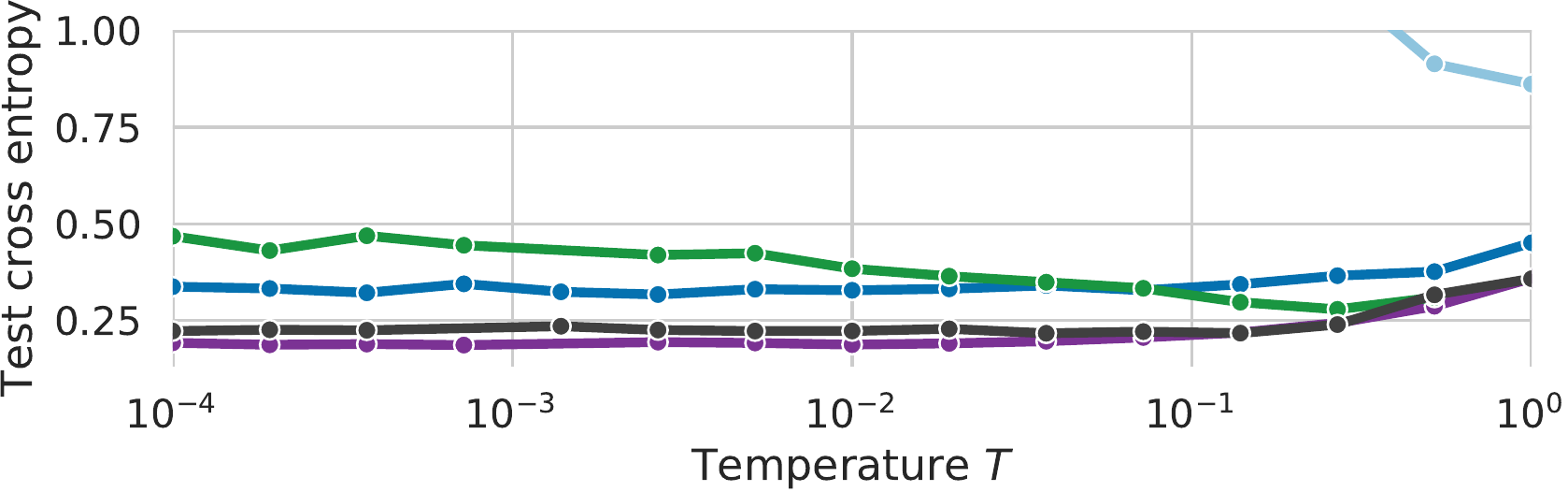}}%
\vspace{-0.45cm}%
\caption{ResNet-20/CIFAR-10 predictive performance as a function of temperature $T$ for different priors $p(\pars)=\mathcal{N}(0,\sigma)$. The cold posterior effect is present for all choices of the prior variance $\sigma$. For all models the optimal temperature is significantly smaller than one and for $\sigma=0.001$ the performance is poor for all temperatures. There is no ``simple'' fix of the prior.}%
\label{fig:priorvariance}%
\vspace{-0.25cm}%
\end{figure}

\section{Alternative Explanations?}
Are there other explanations we have not studied in this work?

\paragraph{Masegosa Posteriors.}
One exciting avenue of future exploration was provided to us after submitting this work: a compelling analysis of the failure to predict well under the Bayes posterior is given by~\citet{masegosa2019misspecification}.
In his analysis he first follows~\citet{germain2016pacbayesian} in identifying the Bayes posterior as a solution of a loose PAC-Bayes generalization bound on the predictive cross-entropy.
He then uses recent results demonstrating improved Jensen inequalities,~\citep{liao2019sharpeningjensen}, to derive alternative posteriors.
These alternative posteriors are \emph{not} Bayes posteriors and in fact explicitly encourage diversity among ensemble member predictions.
Moreover, the alternative posteriors can be shown to dominate the predictive performance achieved by the Bayes posterior when the model is misspecified.
We believe that these new ``Masegosa-posteriors'', while not explaining cold posteriors fully, may provide a more desirable approximation target than the Bayes posterior.  In addition, the Masegosa-posterior is compatible with both variational and SG-MCMC type algorithms.

\paragraph{Tempered observation model?}
In~\citep[Section 8.3]{wilson2020bayesianperspective} it is claimed that cold posteriors in one model correspond to untempered ($T=1$) Bayes posteriors in a modified model by a simple change of the likelihood function.  If this were the case, this would resolve the cold posterior problem and in fact point to a systematic way how to improve the Bayes posterior in many models.
However, the argument in~\citep{wilson2020bayesianperspective} is wrong, which we demonstrate and discuss in detail in Appendix~\ref{sec:templikelhood}.

\section{Related Work on Tempered Posteriors}

Statisticians have studied \emph{tempered} or \emph{fractional} posteriors for $T > 1$.
Motivated by the behavior of Bayesian inference in \emph{misspecified} models~\citep{grunwald2017inconsistency,jansen2013misspecification} develop the \emph{SafeBayes} approach and \citet{bhattacharya2019fractionalposteriors} develops \emph{fractional posteriors} with the goal of slowing posterior concentration.
The use of multiple temperatures $T > 1$ is also common in Monte Carlo simulation in the presence of rough energy landscapes, e.g.~\citep{earl2005paralleltempering,sugita1999replicaexchange,swendsen1986replicamontecarlo}.  However, the purpose of such tempering is to aid in accurate sampling at a desired target temperature, but not in changing the target distribution. \citep{DBLP:conf/aistats/MandtMARB16} studies temperature as a latent variable in the context of variational inference and shows that models often select temperatures different from one. 

\begin{figure}[!t]
\vspace{-0.25cm}%
\center{\includegraphics[width=\columnwidth]{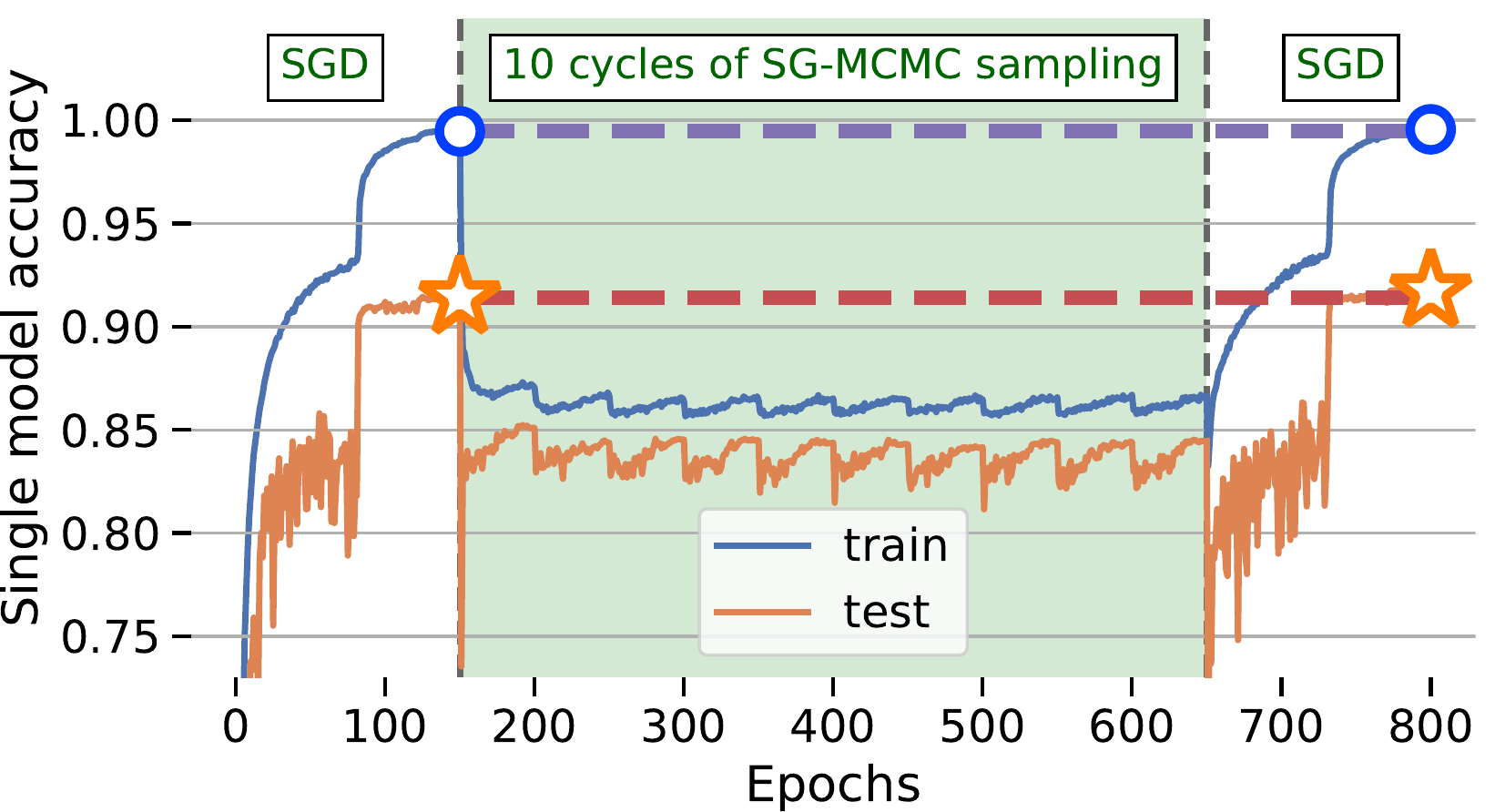}}%
\vspace{-0.45cm}%
\caption{Do the SG-MCMC dynamics harm a beneficial initialization bias used by SGD?
We first train a ResNet-20 on CIFAR-10 via SGD, then switch over to SG-MCMC sampling and finally switch back to SGD optimization.
We report the single-model test accuracy of SGD and the SG-MCMC chain as function of epochs. SGD recovers from being initialized by the SG-MCMC state.}%
\label{fig:temperature_010}%
\vspace{-0.25cm}%
\end{figure}

\section{Conclusion}
Our work has raised the question of cold posteriors but we did not fully resolve nor fix the cause for the cold posterior phenomenon.
Yet our experiments suggest the following.

\textbf{SG-MCMC is accurate enough:}
our experiments (Section~\ref{sec:inference-accuracy}--\ref{sec:why_true_posterior_could_be_poor}) and novel diagnostics (Appendix~\ref{sec:diagnostics}) indicate that current SG-MCMC methods are robust, scalable, and accurate enough to provide good approximations to parameter posteriors in deep nets.

\textbf{Cold posteriors work:}
while we do not fully understand cold posteriors, tempered SG-MCMC ensembles provide a way to train ensemble models with improved predictions compared to individual models.
However, taking into account the added computation from evaluating ensembles, there may be more practical methods,~\citep{lakshminarayanan2017deepensembles,wen2019batchensemble,ashukha2019uncertaintyestimation}.

\textbf{More work on priors for deep nets is needed:}
the experiments in Section~\ref{sec:prior} implicate the prior $p(\pars)$ in the cold posterior effect, although the prior may not be the only cause.
Our investigations
fail to produce a ``simple'' fix based on scaling the prior variance appropriately.
Future work on suitable priors for Bayesian neural networks is needed, building on recent advances,~\citep{sun2019functionalvb,pearce2019expressive,flam2017mapping,hafner2018noise}.

\paragraph{Acknowledgements.}
We would like to thank Dustin Tran for reading multiple drafts and providing detailed feedback on the work.
We also thank the four anonymous ICML 2020 reviewers for their detailed and helpful feedback.

%% file: supplementarycontent.tex
\section{Model Details}\label{sec:model-details}
We now give details regarding the models we use in all our experiments.
We use Tensorflow version 2.1 and carry out all experiments on Nvidia P100 accelerators.

\subsection{ResNet-20 CIFAR-10 Model}\label{sec:resnet-details}

We use the CIFAR-10 dataset from~\citep{krizhevsky2009cifar},
in ``version 3.0.0'' provided in Tensorflow Datasets.\footnote{See~\url{https://www.tensorflow.org/datasets/catalog/cifar10}}
We use the Tensorflow Datasets training/testing split of 50,000 and 10,000 images, respectively.

We use the ResNet-20 model from~\url{https://keras.io/examples/cifar10_resnet/} as a starting point.
For our SGD baseline we use the exact same setup as in the Keras example (200 epochs, learning rate schedule, SGD with Nesterov acceleration).
Notably the Keras example uses bias terms in all convolution layers, whereas some other implementations do not.

The Keras example page reports a reference test accuracy of 92.16 percent for the CIFAR-10 model, compared to our 92.22 percent accuracy.
This is consistent with the larger literature, collected for example at~\url{https://github.com/google/edward2/tree/master/baselines/cifar10}, with even higher accuracy achieved for variations of the ResNet model such as using wide layers, removing bias terms in the convolution layers, or additional regularization.

In this paper we study the phenomenon of poor $T=1$ posteriors obtained by SG-MCMC and therefore use an accurate simulation and sampling setup at the cost of runtime.
In order to obtain accurate simulations we use the following settings for SG-MCMC in every experiment, except where noted otherwise:

\begin{compactitem}
\item Number of epochs: 1500
\item Initial learning rate: $\ell=0.1$
\item Momentum decay: $\beta = 0.98$
\item Batch size: $|B| = 128$
\item Sampling start: begin at epoch $150$
\item Cycle length: $50$
\item Cycle schedule: cosine
\item Prior: $p(\pars) = \mathcal{N}(0,I)$
\end{compactitem}

For experiments on CIFAR-10 we use data augmentation as follows:
\begin{compactitem}
\item random left/right flipping of the input image;
\item border-padding by zero values, four pixels in horizontal and vertical direction, followed by a random cropping of the image to its original size.
\end{compactitem}

We visualize the cyclic schedule used in our ResNet-20 CIFAR-10 experiments in Figure~\ref{fig:cyclical-Ct-Tt-cifar}.

\begin{figure}[!t]
\center{\includegraphics[width=\columnwidth]{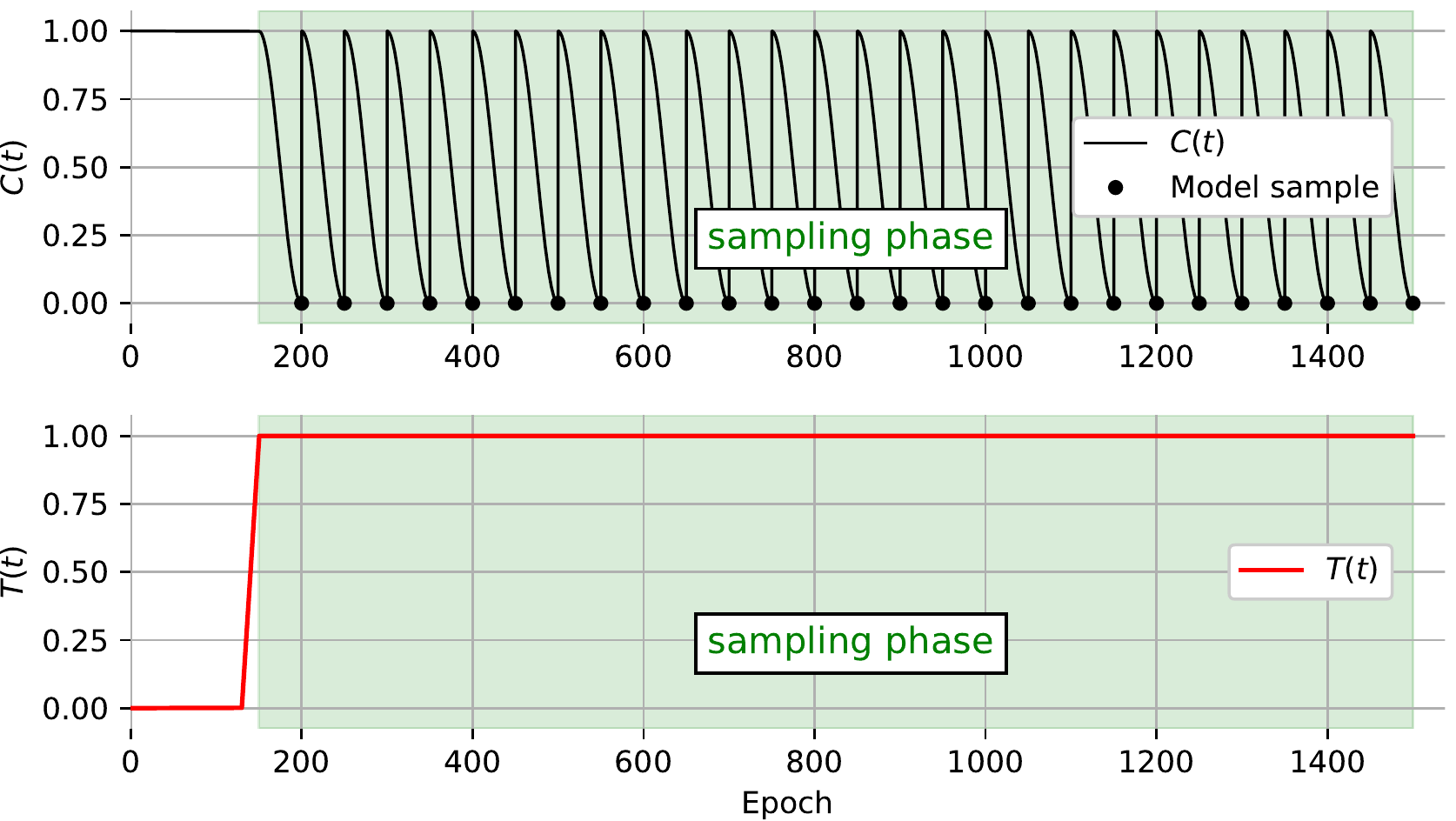}}%
\caption{Cyclical time stepping $C(t)$, and temperature ramp-up $T(t)$, as proposed by~\citet{zhang2019cyclicalsgmcmc} and used in Algorithm~\ref{alg:symeuler}, for our ResNet-20 CIFAR-10 model (Section~\ref{sec:resnet-details}).
We sample one model at the end of each cycle when the inference accuracy is best, obtaining an ensemble of 27 models.}%
\label{fig:cyclical-Ct-Tt-cifar}%
\end{figure}

\subsection{ResNet-20 CIFAR-10 SGD Baseline}\label{sec:resnet-sgd-details}
For the SGD baseline we follow the best practice from the existing Keras example which was tuned for generalization performance.
In particular we use:

\begin{compactitem}
\item Number of epochs: 200
\item Initial learning rate: $\ell=0.1$
\item Momentum term: $0.9$
\item L2 regularization coefficient: $0.002$
\item Batch size: $128$
\item Optimizer: SGD with Nesterov momentum
\item Learning rate schedule (epoch, $\ell$-multiplier):
$(80, 0.1)$, $(120,0.01)$, $(160,0.001)$, $(180,0.0005)$.
\end{compactitem}

Data augmentation is the same as described in Section~\ref{sec:resnet-details}.
We report the final validation performance and over the $200$ epochs do not observe any overfitting.

\subsection{CNN-LSTM IMDB Model}\label{sec:cnnlstm-details}

We use the IMDB sentiment classification text dataset provided by the \texttt{tensorflow.keras.datasets} API in Tensorflow version 2.1.
We use 20,000 words and a maximum sequence length of 100 tokens.
We use 20,000 training sequences and 25,000 testing sequences.

We use the CNN-LSTM example\footnote{Available at \url{https://github.com/keras-team/keras/blob/master/examples/imdb_cnn_lstm.py}} as a starting point.
For our SGD baseline we use the Keras model but add a prior $p(\theta) = \mathcal{N}(0,I)$ as used for the Bayesian posterior.
We then use the Tensorflow SGD implementation to optimize the resulting $U(\theta)$ function.
For SGD the model overfits and we therefore report the best end-of-epoch test accuracy and test cross-entropy achieved.

For all experiments, except where explicitly noted otherwise, we use the following parameters:

\begin{compactitem}
\item Number of epochs: 500
\item Initial learning rate: $\ell=0.1$
\item Momentum decay: $\beta = 0.98$
\item Batch size: $|B| = 32$
\item Sampling start: begin at epoch $50$
\item Cycle length: $25$
\item Cycle schedule: cosine
\item Prior: $p(\pars) = \mathcal{N}(0,I)$
\end{compactitem}

We visualize the cyclic schedule used in our CNN-LSTM IMDB experiments in Figure~\ref{fig:cyclical-Ct-Tt-imdb}.

\begin{figure}[!t]
\center{\includegraphics[width=\columnwidth]{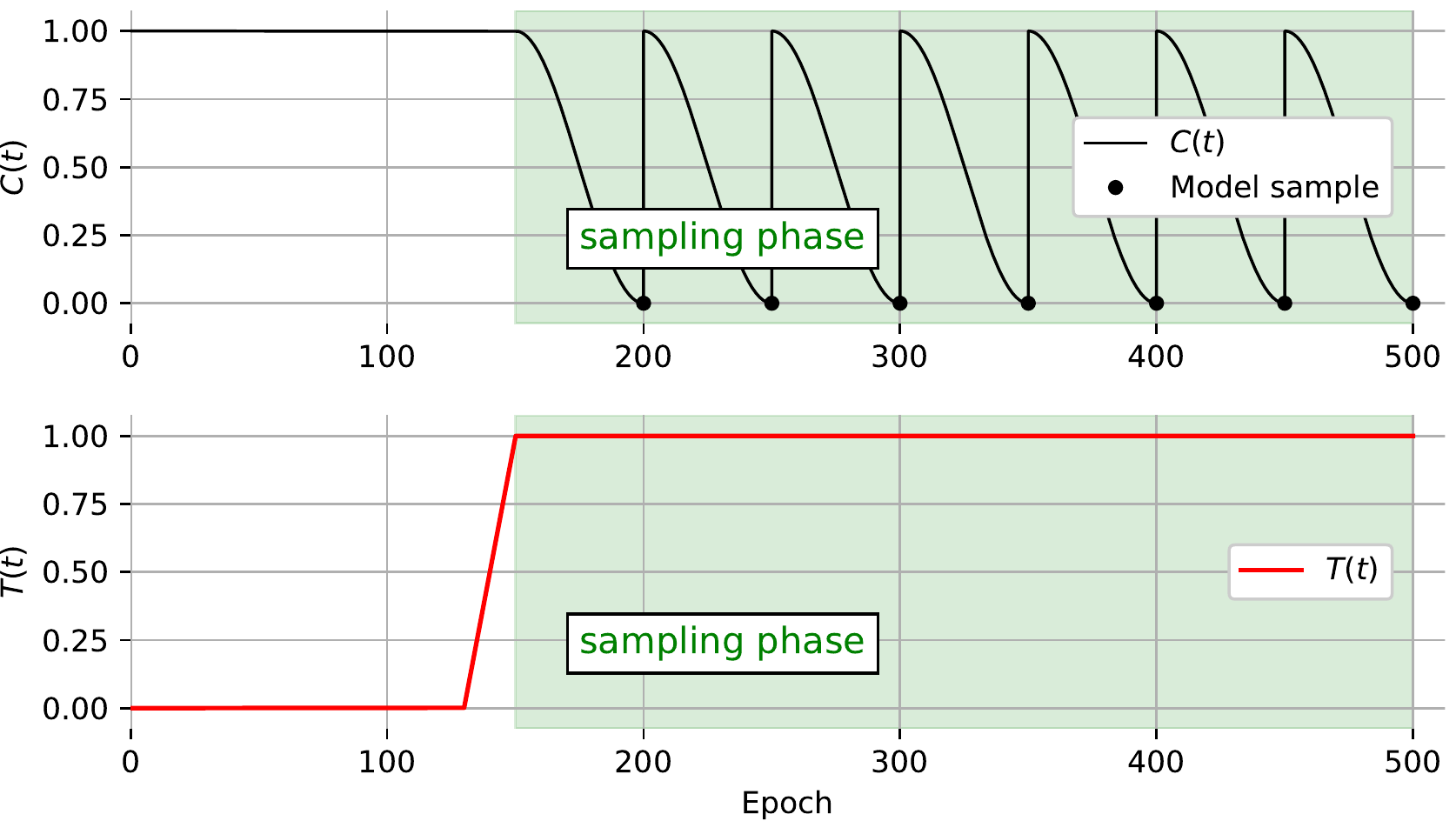}}%
\caption{Cyclical time stepping $C(t)$, and temperature ramp-up $T(t)$ for our CNN-LSTM IMDB model (Section~\ref{sec:cnnlstm-details}).
We sample one model at the end of each cycle when the inference accuracy is best, obtaining an ensemble of 7 models.}%
\label{fig:cyclical-Ct-Tt-imdb}%
\end{figure}

\subsection{CNN-LSTM IMDB SGD Baseline}\label{sec:cnnlstm-sgd-details}
The SGD baseline follows the Keras example settings:

\begin{compactitem}
\item Number of epochs: 50
\item Initial learning rate: $\ell=0.1$
\item Momentum term: $0.98$
\item Regularization: MAP with $\mathcal{N}(0,I)$ prior
\item Batch size: $32$
\item Optimizer: SGD with Nesterov momentum
\item Learning rate schedule: None
\end{compactitem}

We report the optimal test set performance from all end-of-epoch test evaluations.
This is necessary because there is significant overfitting after the first ten epochs.

\section{Deep Learning Parameterization of SG-MCMC Methods}\label{sec:sgmcmc-deeplearning-parameterization}
We derive the bijection between (learning rate $\ell$, momentum decay $\beta$) and (timestep $h$, friction $\gamma$) by considering the \emph{instantaneous gradient effect} $\alpha$ on the parameter, i.e. the amount by which the current gradient at time $t$ affects the current gradient update update at time $t$.
We set $\alpha=\ell/n$, where $\ell$ is the familiar learning rate parameter used in SGD and the factor
$1/n$ is to convert $\nabla_{\pars} U$ to $\nabla_{\pars} G$,
as $\nabla_{\pars} G = \nabla_{\pars} U / n$ is the familiar minibatch mean gradient.
Likewise, the \emph{momentum decay} is the factor $\beta < 1$ by which the momentum vector $\moms^{(t)}$ is shrunk in each discretized time step.
Having determined $\alpha$ and $\beta$ we can derive two non-linear equations that depend on the particular time discretization used; for the symplectic Euler Langevin scheme these are
\begin{equation}
h^2 = \alpha \quad\left(= \frac{\ell}{n}\right),%
\qquad\textrm{and}\qquad%
1-h\gamma = \beta.
\end{equation}
Solving these equations for $h$ and $\gamma$ simultaneously, given $\ell$, $n$, and $\beta$ yields the bijection
\begin{eqnarray}
h & = & \sqrt{\ell/n},\label{eqn:biject-h}\\
\gamma & = & (1-\beta) \sqrt{n/\ell}.\label{eqn:biject-gamma}
\end{eqnarray}

\section{Connection to Stochastic Gradient Descent (SGD)}\label{sec:sgd}
We now give a precise connection between stochastic gradient descent (SGD) and the symplectic Euler SG-MCMC method, Algorithm~\ref{alg:symeuler} from the main paper.

Algorithm~\ref{alg:sgd} gives the stochastic gradient descent (SGD) with momentum algorithm as implemented in
\emph{Tensorflow}'s version 2.1 optimization methods, \texttt{tensorflow.keras.optimizers.SGD}
and~\texttt{tensorflow.train.MomentumOptimizer}, \citep{abadi2016tensorflow}.

\begin{algorithm2e}[t]
\DontPrintSemicolon
\SetKwFunction{FnSGD}{SGD}
\SetKw{KwYield}{yield}
\SetKwProg{Fn}{Function}{}{}
\Fn{\FnSGD{$\tilde{G}$, $\pars^{(0)}$, $\ell$, $\beta$}}{
  \KwInput{
    $\tilde{G}: \Theta \to \mathbb{R}$ average batch loss function, cf equation~(\ref{eqn:Ggrad});
    $\pars^{(0)} \in \mathbb{R}^d$ initial parameter;
    $\ell > 0$ learning rate parameter;
    $\beta \in [0,1)$ momentum decay parameter.
  }
  \KwOutput{Parameter sequence $\pars^{(t)}$, at step $t=1,2,\dots$}
  $\moms^{(0)} \leftarrow \mathbf{0}$ \tcp*{Initialize momentum}
  \For{ $t=1,2,\dots$ }{
    $\moms^{(t)} \leftarrow \beta \, \moms^{(t-1)} - \ell \, \nabla_{\pars} \tilde{G}(\pars^{(t-1)})$
    \label{line:sgd-mom}\tcp*{Update momentum}
    $\pars^{(t)} \leftarrow \pars^{(t-1)} + \moms^{(t)}$ \label{line:sgd-pars}\tcp*{Update parameters}
    \KwYield{$\pars^{(t)}$} \tcp*{Parameter at step $t$}
  }%
}%
\caption{Stochastic Gradient Descent with Momentum (SGD) in Tensorflow.}%
\label{alg:sgd}%
\end{algorithm2e}

Starting with Algorithm~\ref{alg:sgd} we first perform an equivalent substitution of the moments,
\begin{eqnarray}
\tilde{\moms}^{(t)} & := & \sqrt{\frac{n}{\ell}} \, \moms^{(t)},\qquad\textrm{respectively,}\\
\moms^{(t)} & := & \sqrt{\frac{\ell}{n}} \, \tilde{\moms}^{(t)},
\end{eqnarray}
we obtain the update from line~\ref{line:sgd-mom} in Algorithm~\ref{alg:sgd},
\begin{equation}
\sqrt{\frac{\ell}{n}} \, \tilde\moms^{(t)} \leftarrow \beta \sqrt{\frac{\ell}{n}}\, \tilde\moms^{(t-1)} - \ell \, \nabla_{\pars} \tilde{G}(\pars^{(t-1)}.
\label{eqn:sgd-deriv1}
\end{equation}
Multiplying both sides of~(\ref{eqn:sgd-deriv1}) by $\sqrt{n} / \sqrt{\ell}$ we obtain an equivalent form of Algorithm~\ref{alg:sgd} with lines~\ref{line:sgd-mom}~and~\ref{line:sgd-pars} replaced by
\begin{eqnarray}
\tilde{\moms}^{(t)} & \leftarrow & \beta \, \tilde{\moms}^{(t-1)} - \sqrt{\ell n} \, \nabla_{\pars} \tilde{G}(\pars^{(t-1)}),\label{eqn:sgd-sqrt-moms}\\
\pars^{(t)} & \leftarrow & \pars^{(t-1)} + \sqrt{\frac{\ell}{n}}\, \tilde{\moms}^{(t)}.\label{eqn:sgd-sqrt-pars}
\end{eqnarray}
From the bijection~(\ref{eqn:biject-h}--\ref{eqn:biject-gamma}) we have
$h = \sqrt{\ell/n}$ and $\gamma = (1-\beta)\sqrt{n/\ell}$.
Solving for $\beta$ gives
\begin{equation}
    \beta = 1 - \gamma \sqrt{\frac{\ell}{n}} = 1 - \gamma h.\label{eqn:beta}
\end{equation}
We also have
\begin{equation}
    \sqrt{\ell n} = \sqrt{\frac{\ell}{n} n^2} = n \sqrt{\frac{\ell}{n}} = h \, n.\label{eqn:sqrt-ell-n}
\end{equation}
Substituting~(\ref{eqn:beta}) and~(\ref{eqn:sqrt-ell-n}) into~(\ref{eqn:sgd-sqrt-moms}) and~(\ref{eqn:sgd-sqrt-pars}) gives the equivalent updates
\begin{eqnarray}
\tilde{\moms}^{(t)} & \leftarrow & (1-\gamma h) \, \tilde{\moms}^{(t-1)} - h \, n \, \nabla_{\pars} \tilde{G}(\pars^{(t-1)}),\label{eqn:sgd-sqrt-moms-repar}\\
\pars^{(t)} & \leftarrow & \pars^{(t-1)} + h\, \tilde{\moms}^{(t)}.\label{eqn:sgd-sqrt-pars-repar}
\end{eqnarray}

These equivalent changes produce Algorithm~\ref{alg:sgde}.
Algorithm~\ref{alg:sgd} and Algorithm~\ref{alg:sgde} generate equivalent trajectories
$\pars^{(t)}$, $t=1,2,\dots$, but differ in the scaling of their momenta, $\moms^{(t)}$ and $\tilde{\moms}^{(t)}$.

\begin{algorithm2e}[t]
\DontPrintSemicolon
\SetKwFunction{FnSGDE}{SGDEquivalent}
\SetKw{KwYield}{yield}
\SetKwProg{Fn}{Function}{}{}
\Fn{\FnSGDE{$\tilde{G}$, $\pars^{(0)}$, $\ell$, $\beta$}}{
  \KwInput{
    $\tilde{G}: \Theta \to \mathbb{R}$ average batch loss function, cf equation~(\ref{eqn:Ggrad});
    $\pars^{(0)} \in \mathbb{R}^d$ initial parameter;
    $h > 0$ discretization step size parameter;
    $\gamma > 0$ friction parameter.
  }
  \KwOutput{Parameter sequence $\pars^{(t)}$, $t=1,2,\dots$, at step $t$}
  $\tilde{\moms}^{(0)} \leftarrow \mathbf{0}$ \tcp*{Initialize momentum}
  \For{ $t=1,2,\dots$ }{
    $\tilde{\moms}^{(t)} \leftarrow (1-\gamma h) \, \tilde{\moms}^{(t-1)} - h \, n \, \nabla_{\pars} \tilde{G}(\pars^{(t-1)})$\tcp*{Update momentum}\label{line:sgde-moms}
    $\pars^{(t)} \leftarrow \pars^{(t-1)} + h\, \tilde{\moms}^{(t)}$\tcp*{Update parameters}\label{line:sgde-pars}
    \KwYield{$\pars^{(t)}$} \tcp*{Parameter at step $t$}
  }%
}%
\caption{Stochastic Gradient Descent with Momentum (SGD), reparameterized.}%
\label{alg:sgde}%
\end{algorithm2e}

Comparing lines~\ref{line:sgde-moms}--\ref{line:sgde-pars} in Algorithm~\ref{alg:sgde} with
lines~\ref{line:symeuler-moms}--\ref{line:symeuler-pars} in Algorithm~\ref{alg:symeuler} from the main paper we see that when $\mathbf{M}=I$ and $C(t) = 1$ the only remaining difference between the updates is the additional noise $\sqrt{2 \gamma h T} \, \mathbf{M}^{1/2} \mathbf{R}^{(t)}$ in the SG-MCMC method.
In this \emph{precise} sense the SG-MCMC Algorithm~\ref{alg:symeuler} from the main paper is just ``SGD with noise''.

\section{Semi-Adaptive Estimation of Layerwise Preconditioner $\mathbf{M}$}\label{sec:precond}
During our experiments with deep learning models we noticed that both minibatch noise as well as gradient magnitudes tend to behave similar within a set of related parameters.
For example, for a given learning iteration, all gradients related to convolution kernel weights of the same convolution layer of a network tend to have similar magnitudes and minibatch noise variance.
At the same iteration they may be different from the magnitudes and minibatch noise variance of gradients of the parameters of another layer in the same network.

Therefore, we estimate a simple diagonal preconditioner that ties together the scale of all parameter elements that belong to the same model variable.
Moreover, we normalize the preconditioner so that the least sensitive variable always has scale one.
With such normalization, if all variables would be equally sensitive the preconditioner becomes $\mathbf{M}=I$, the identity preconditioner.

We estimate the layerwise preconditioner using Algorithm~\ref{alg:layerwiseprecond}.

\paragraph{Updating the preconditioner.}
In Langevin schemes the preconditioner couples the moment space to the parameter space.
If we use a new estimate $\mathbf{M}'$ to replace the old preconditioner $\mathbf{M}$ then we change this coupling and if left unchanged then the old moments $\moms$ would no longer have the correct distribution.\footnote{More precisely, $\mathbf{M}^{-1/2} \moms$ should always be distributed according to $\mathcal{N}(0,I)$.}
We therefore posit that upon changing the preconditioner the effect of the moments should remain the same.  To retain the full information in the current moments we set $\moms' = \mathbf{M}'^{1/2} \mathbf{M}^{-1/2} \moms$ which we can understand as
$\mathbf{M}'^{1/2} (\mathbf{M}^{-1/2} \moms)$, where the bracketed part canonicalizes the moments $\moms$ to the identity preconditioner, and $\mathbf{M}'^{1/2}$ transfers the canonical moments to the new preconditioner.

\begin{center}%
\begin{algorithm2e}[t]%
\DontPrintSemicolon
\SetKwComment{tcn}{}{}%
\SetKwFunction{FnLayerwisePreconditioning}{EstimateM}
\SetKw{KwReturn}{return}
\SetKwProg{Fn}{Function}{}{}
\Fn{\FnLayerwisePreconditioning{$\tilde{G}$, $\pars$, $K$, $\epsilon$}}{
  \KwInput{
    $\tilde{G}: \Theta \to \mathbb{R}$ mean energy function estimate;
    $(\pars_1, \dots, \pars_S) \in \mathbb{R}^{d_1 \times \dots \times d_S}$ current model parameter variables;
    $K$ number of minibatches (default $K=32$);
    $\epsilon$ regularization value (default $\epsilon = 10^{-7}$)
  }
  \KwOutput{Preconditioning matrix $\mathbf{M}$}
  \For{ $s=1,2,\dots,S$ }{
    $\rawm_s \leftarrow \textbf{0}$
  }
  \For{ $k=1,2,\dots,K$}{
    $\mathbf{g}^{(k)} \leftarrow \nabla_{\pars} \tilde{G}(\pars)$\tcp*{Noisy gradient}
      \For{ $s=1,2,\dots,S$ }{
        $\rawm_s \leftarrow \rawm_s + \mathbf{g}^{(k)}_s \cdot \mathbf{g}^{(k)}_s$
      }
  }
  \For{ $s=1,2,\dots,S$ }{
    $\sigma_s \leftarrow \sqrt{\epsilon + \frac{1}{d_s K} \sum_i \rawm_{s,i}}$\tcp*{RMSprop}
  }
  $\sigma_{\textrm{min}} \leftarrow \min_{s} \sigma_s$\tcp*{Least sensitive}
  \For{ $s=1,2,\dots,S$ }{
    $\mathbf{M}_s \leftarrow \frac{\sigma_s}{\sigma_{\textrm{min}}} \, I$
  }
  $\mathbf{M} \leftarrow \left[\begin{array}{ccc}
  \mathbf{M}_1 & \dots & 0\\
  \vdots & \ddots & \vdots \\
  0 & \dots & \mathbf{M}_S
  \end{array}\right]$\;%
  \KwReturn{$\mathbf{M}$}
}
\caption{Estimate Layerwise Preconditioner.}
\label{alg:layerwiseprecond}
\end{algorithm2e}
\end{center}

\section{Kullback-Leibler Scaling in Variational Bayesian Neural Networks}
\label{sec:tempered-elbo}

With the posterior energy $U(\pars)$ defined in the main paper we define two variants of tempered posterior energies:
\begin{compactitem}
\item Fully tempered energy: $U_F(\pars) = U(\pars)/T$, and
\item Partially tempered energy:
$U_P(\pars) = -\log p(\pars) -\frac{1}{T}\sum_{i=1}^n \log p(y_i|x_i,\pars)$.
\end{compactitem}
Note that $U_F(\pars)$ is used for all experiments in the paper and temper both the log-likelihood as well as the log-prior terms, whereas $U_P(\pars)$ only scales the log-likelihood terms while leaving the log-prior untouched.

We now show that Kullback-Leibler scaling as commonly done in variational Bayesian neural networks corresponds to approximating the partially tempered posterior,
\begin{equation}
p_P(\pars|\mathcal{D}) \propto \exp(-U_P(\pars)).
\end{equation}

For any distribution $q(\pars)$ we consider the Kullback-Leibler divergence,
\begin{align}
& D_{\textrm{KL}}(q(\pars) \,\|\, p_P(\pars|\mathcal{D}))\\
& = \mathbb{E}_{\pars \sim q(\pars)}\left[\log q(\pars) - \log p_P(\pars|\mathcal{D})\right]\\
& = \mathbb{E}_{\pars \sim q(\pars)}\left[
    \log q(\pars) - \log \frac{\exp(-U_P(\pars))}{\int \exp(-U_P(\pars')) \,\textrm{d}\pars'}
\right].\label{eqn:kl-norm}
\end{align}
The normalizing integral in~(\ref{eqn:kl-norm}) is not a function of $\pars$ and thus does not depend on $q(\pars)$, allowing us to simplify the equation further:
\begin{align}
& = \mathbb{E}_{\pars \sim q(\pars)}\left[
    \log q(\pars) - \log p(\pars) - \frac{1}{T} \sum_{i=1}^n \log p(y_i|x_i,\pars)
\right]\\
& \qquad\qquad + \underbrace{\log \int \exp(-U_P(\pars)) \,\textrm{d}\pars}_{\textrm{constant, $=: \log E_P$}}\\
& = D_{\textrm{KL}}(q(\pars) \,\|\, p(\pars)) - \frac{1}{T} \sum_{i=1}^n \log p(y_i|x_i,\pars) + \log E_P.\label{eqn:partial-kl}
\end{align}
Here we defined $E_P$ as the \emph{partial temperized evidence} which does not depend on $\pars$ and therefore becomes a constant.
The global minimizer of~(\ref{eqn:partial-kl}) over all distributions $q \in \mathcal{Q}$ is the unique distribution $p_P(\pars|\mathcal{D})$,~\citep{mackay1995ensemblelearning}.

We now consider this minimizer, substituting $\lambda := T$,
\begin{align}
& \argmin_{q \in \mathcal{Q}} \, D_{\textrm{KL}}(q(\pars) \,\|\, p_P(\pars|\mathcal{D}))\\
& = \argmin_{q \in \mathcal{Q}} \, D_{\textrm{KL}}(q(\pars) \,\|\, p(\pars)) - \frac{1}{T} \sum_{i=1}^n \log p(y_i|x_i,\pars)
\end{align}
The minimizing $q \in \mathcal{Q}$ does not depend on the overall scaling of the optimizing function.  We can therefore scale the function by a factor of $T$,
\begin{align}
& = \argmin_{q \in \mathcal{Q}} \, T D_{\textrm{KL}}(q(\pars) \,\|\, p(\pars)) - \sum_{i=1}^n \log p(y_i|x_i,\pars)
\end{align}
Substituting $\lambda := T$ yields
\begin{align}
& = \argmin_{q \in \mathcal{Q}} \, \lambda D_{\textrm{KL}}(q(\pars) \,\|\, p(\pars)) - \sum_{i=1}^n \log p(y_i|x_i,\pars).
\label{eqn:partial-nelbo}
\end{align}
The last equation,~(\ref{eqn:partial-nelbo}) is the KL-weighted negative evidence lower bound (ELBO) objective commonly used in variational Bayes for Bayesian neural networks, confer the ELBO equation~(\ref{eqn:vb-bnn}) from the main paper.

\section{Inference Bias-Variance Trade-off Hypothesis}\label{sec:biasvar-appendix}
\hypothesis{Bias-variance Tradeoff Hypothesis}{For $T=1$ the posterior is diverse and there is high variance between model predictions.  For $T \ll 1$ we sample nearby modes and reduce prediction variance but increase bias; the variance dominates the error and reducing variance ($T \ll 1$) improves predictive performance.}

We approach the hypothesis using a simple asymptotic argument.
We consider the SG-MCMC method we use, including preconditioning and cyclical time stepping.  Whereas within a cycle the Markov chain is non-homogeneous, if we consider only the end-of-cycle iterates that emit a parameter $\pars^{(t)}$, then this coarse-grained process is a homogeneous Markov chain.
For such Markov chains we can leverage generalized central limit theorems for functions of $\pars$, see e.g.~\citep{jones2004markovchainclt,haggstrom2007varianceclt}, and because of existence of limits we can consider the asymptotic behavior of the test cross-entropy performance measure $C(S)$ as we increase the ensemble size $S \to \infty$.

In particular, expectations of smooth functions of empirical means of $S$ samples have an expansion of the form,~\citep{nowozin2018jvi,schucany1971biasreduction},
\begin{equation}
\mathbb{E}[C(S)] = C(\infty) + a_1 \frac{1}{S} + a_2 \frac{1}{S^2} + \dots.
\label{eqn:jackknife-expansion}
\end{equation}

\experiment{Risk Asymptotics Experiment:}
if we can estimate $C(\infty)$ we know what performance we could achieve if we were to keep sampling.
To this end we apply a simple linear regression estimate, \citep{schucany1971biasreduction}, to the empirically observed performance estimates $\hat{C}(S)$ for different ensemble sizes $S$.
By truncation at second order, we obtain estimates for $C(\infty)$, $a_1$, and $a_2$.

In Figure~\ref{fig:resnet-asymptotics} we show the regressed test cross-entropy metric obtained by fitting~(\ref{eqn:jackknife-expansion}) to second order to all samples for $S \geq 20$ close to the asymptotic regime, and visualize the estimate $\hat{C}(\infty)$.
In Figure~\ref{fig:resnet-asymptotics-tempdep} we visualize our estimated $\hat{C}(\infty)$ as a function of the temperature $T$.
The results indicate two things:
\emph{first}, we could gain better predictive performance from running our SG-MCMC method for longer (Figure~\ref{fig:resnet-asymptotics});
but \emph{second}, the additional gain that could be obtained from longer sampling is too small to make $T=1$ superior to $T<1$ (Figure~\ref{fig:resnet-asymptotics-tempdep}).

\begin{figure}[!t]
\vspace{-0.2cm}%
\center{\includegraphics[width=\columnwidth]{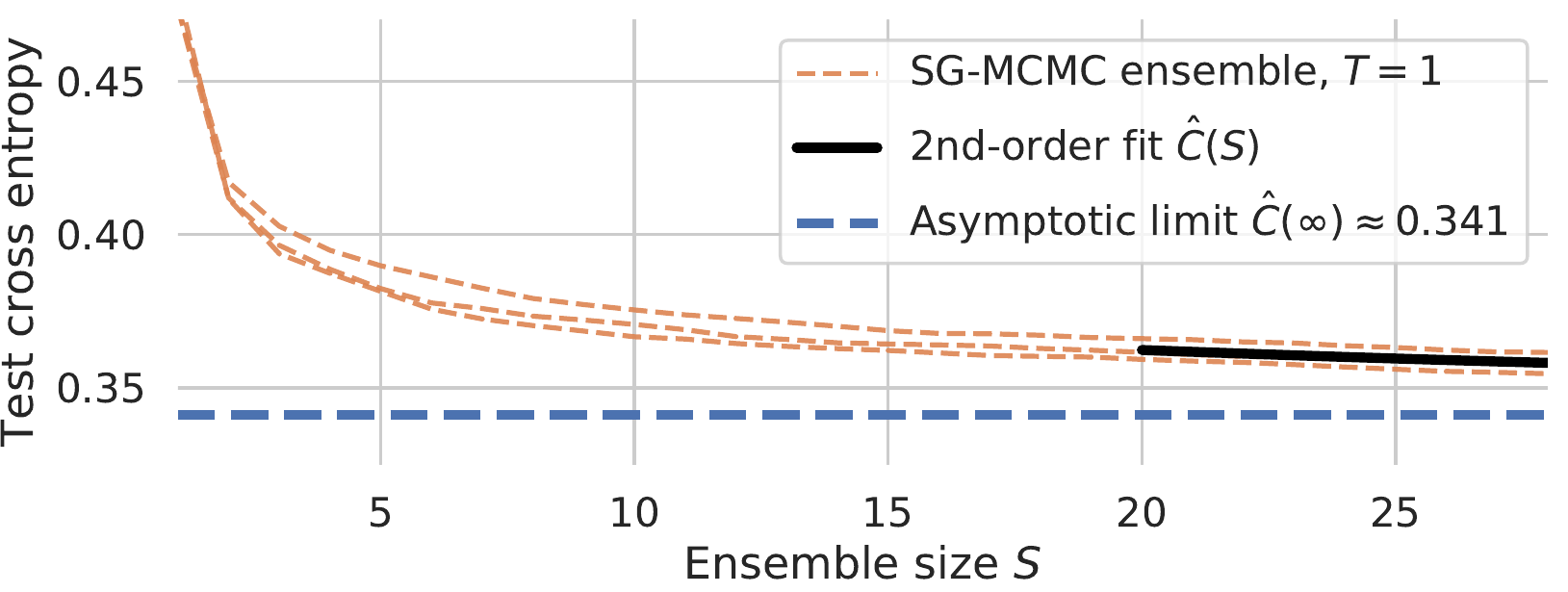}}%
\vspace{-0.3cm}%
\caption{Regressing the limiting ResNet-20/CIFAR-10 ensemble performance: at temperature $T=1$ an ensemble of size $S=\infty$ would achieve 0.341 test cross-entropy.
For SG-MCMC we show three different runs with varying seeds.}%
\label{fig:resnet-asymptotics}%
\vspace{-0.2cm}%
\end{figure}

\begin{figure}[!t]
\vspace{-0.2cm}%
\center{\includegraphics[width=\columnwidth]{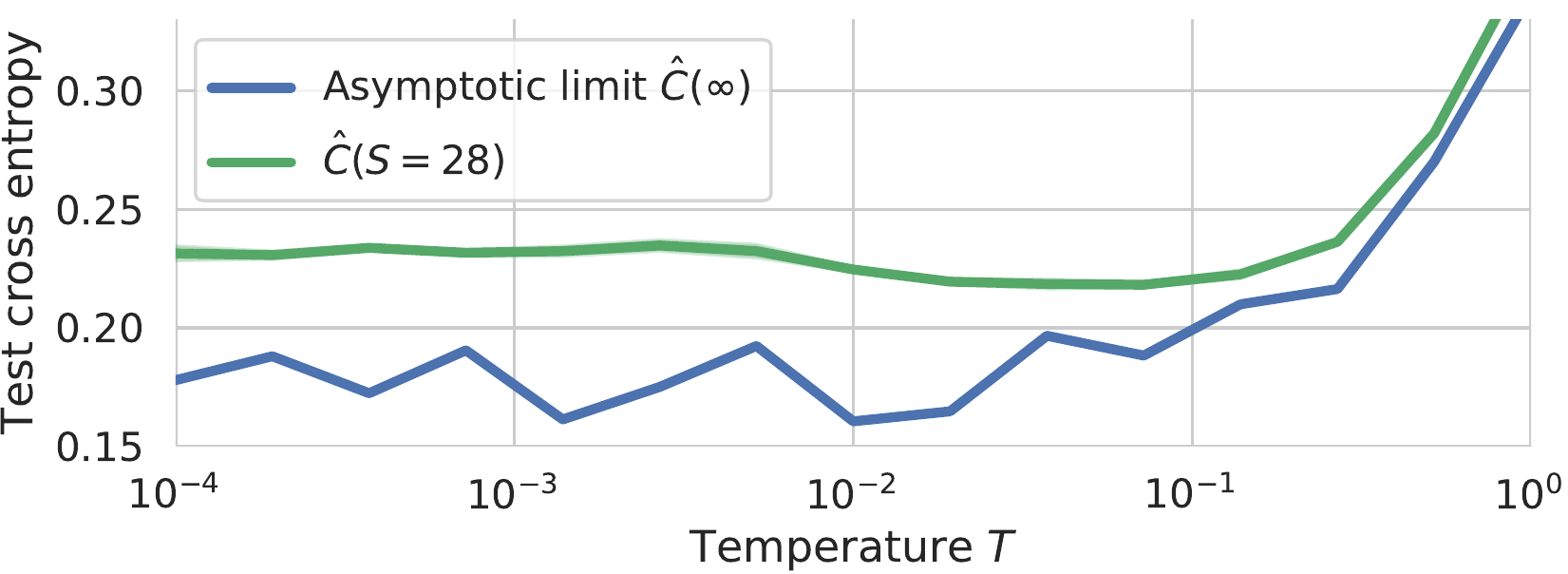}}%
\center{\includegraphics[width=\columnwidth]{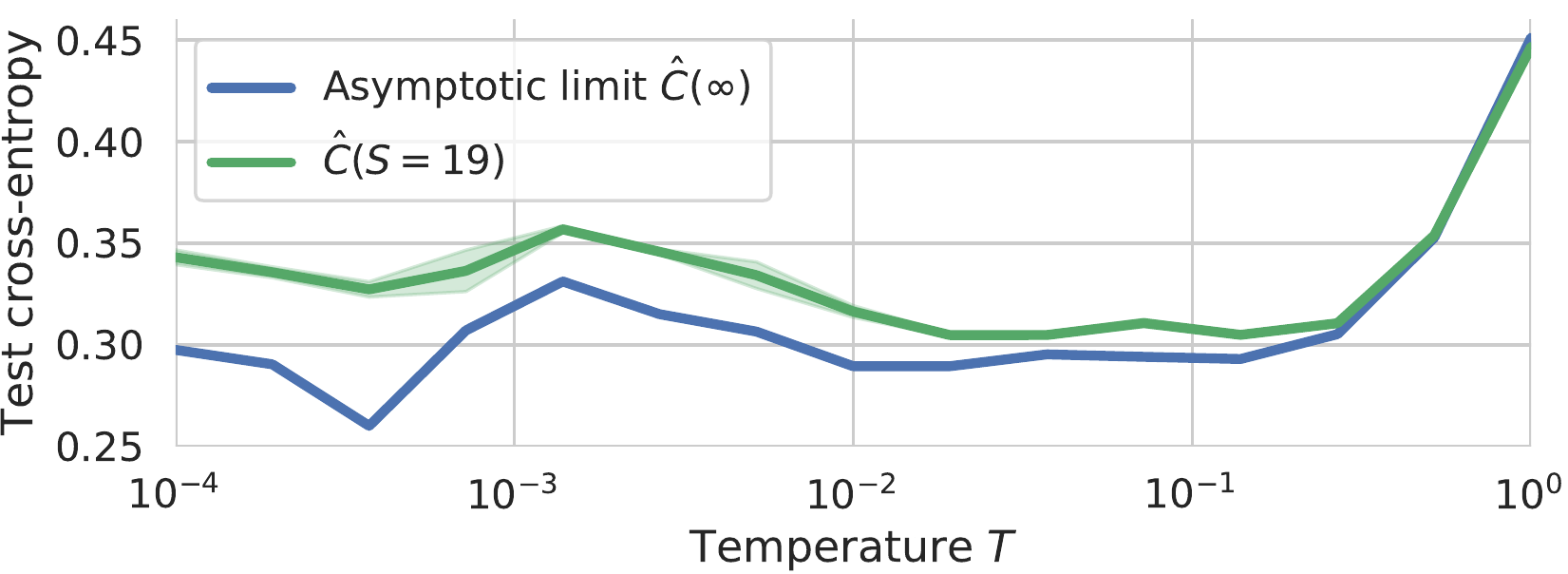}}%
\vspace{-0.3cm}%
\caption{Ensemble variance for ResNet-20/CIFAR-10 (\textbf{top}) and CNN-LSTM/IMDB (\textbf{bottom}) does not explain poor performance at $T=1$: even in the infinite limit the performance $C(\infty)$ remains poor compared to $T<1$.}
\label{fig:resnet-asymptotics-tempdep}%
\vspace{-0.2cm}%
\end{figure}

\section{Cold posteriors improve uncertainty metrics.}\label{sec:tempplots-appendix}
In the main paper we show that cold posteriors improve prediction performance in terms of accuracy and cross entropy. Figure~\ref{fig:tempplots-appendix} and Figure~\ref{fig:tempplots-appendix2} show that for both the ResNet-20 and the CNN-LSTM model, cold posteriors also improve the uncertainty metrics Brier score~~\cite{brier1950verification} and expected calibration error (ECE)~\cite{naeini2015obtaining}.

\begin{figure}[!t]
\vspace{-0.2cm}%
\center{\includegraphics[width=\columnwidth]{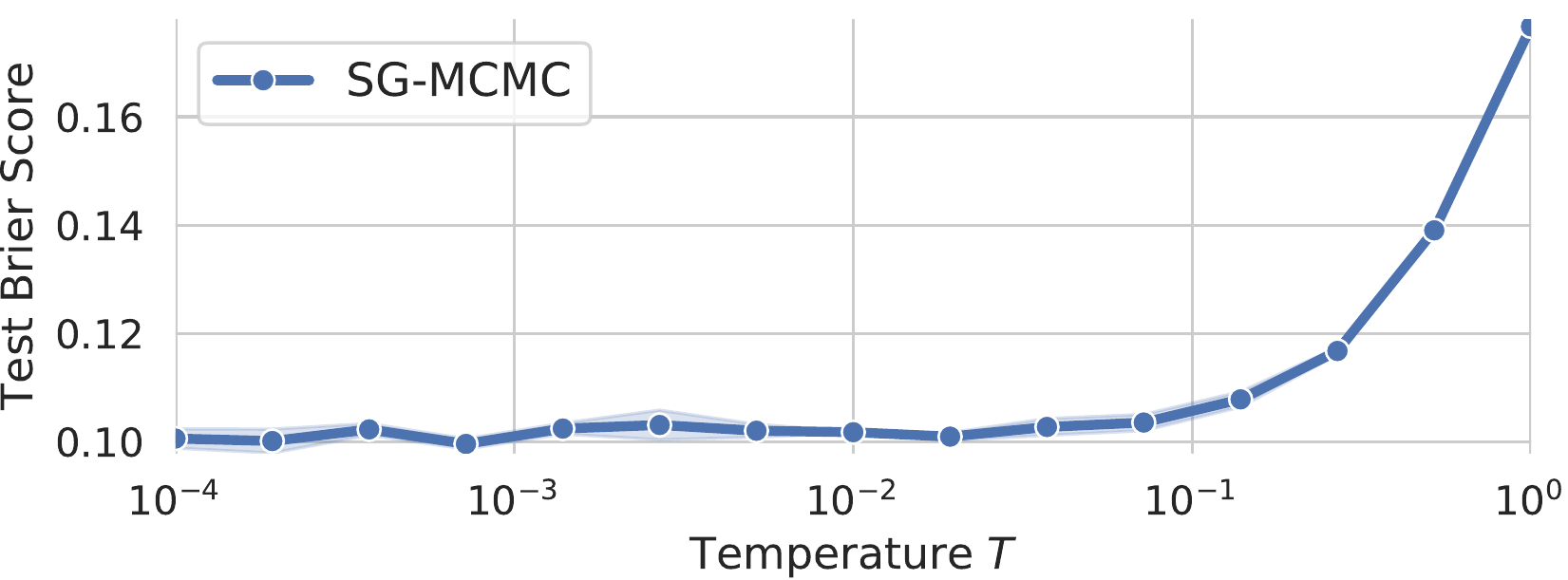}}%
\center{\includegraphics[width=\columnwidth]{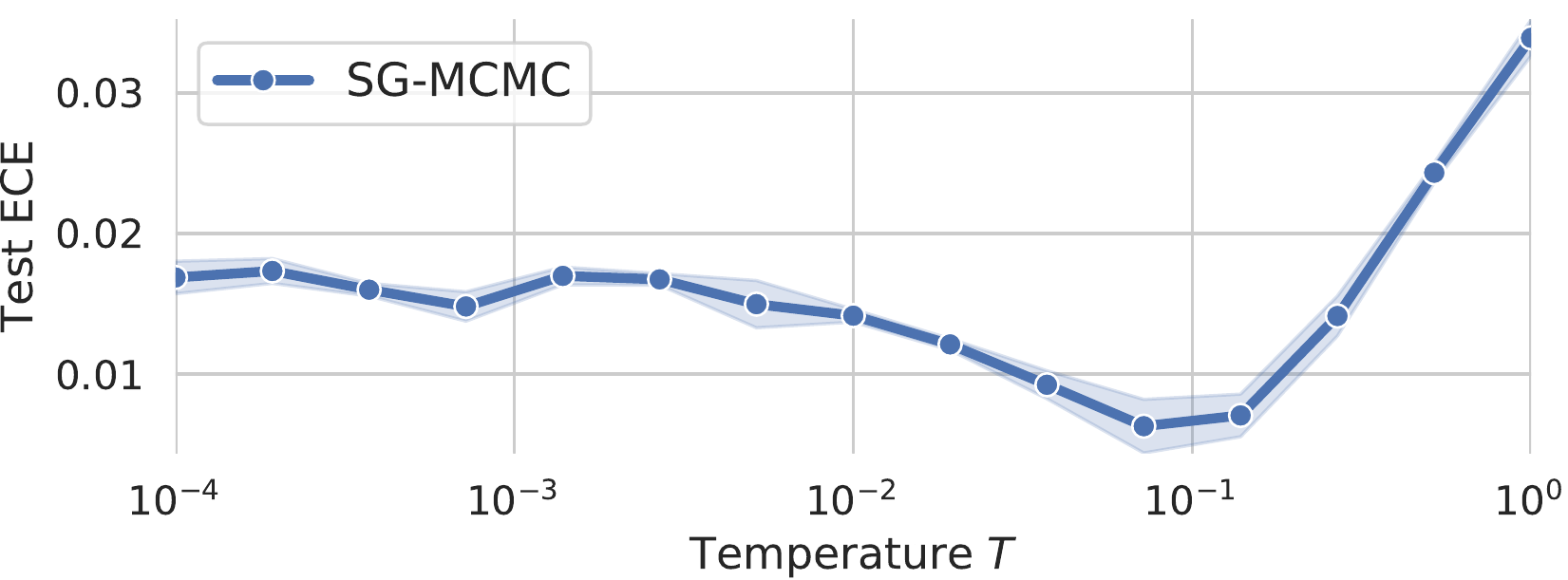}}%
\vspace{-0.3cm}%
\caption{ResNet-20/CIFAR-10: In the main paper we show that cold posteriors improve prediction performance in terms of accuracy and cross entropy (Figure~\ref{fig:resnet-tempdep-acc} and Figure~\ref{fig:resnet-tempdep-ce}). This plot shows that cold posteriors also improve the uncertainty metrics Brier score and expected calibration error (ECE) (lower is better).}
\label{fig:tempplots-appendix}%
\end{figure}

\begin{figure}[!t]
\vspace{-0.2cm}%
\center{\includegraphics[width=\columnwidth]{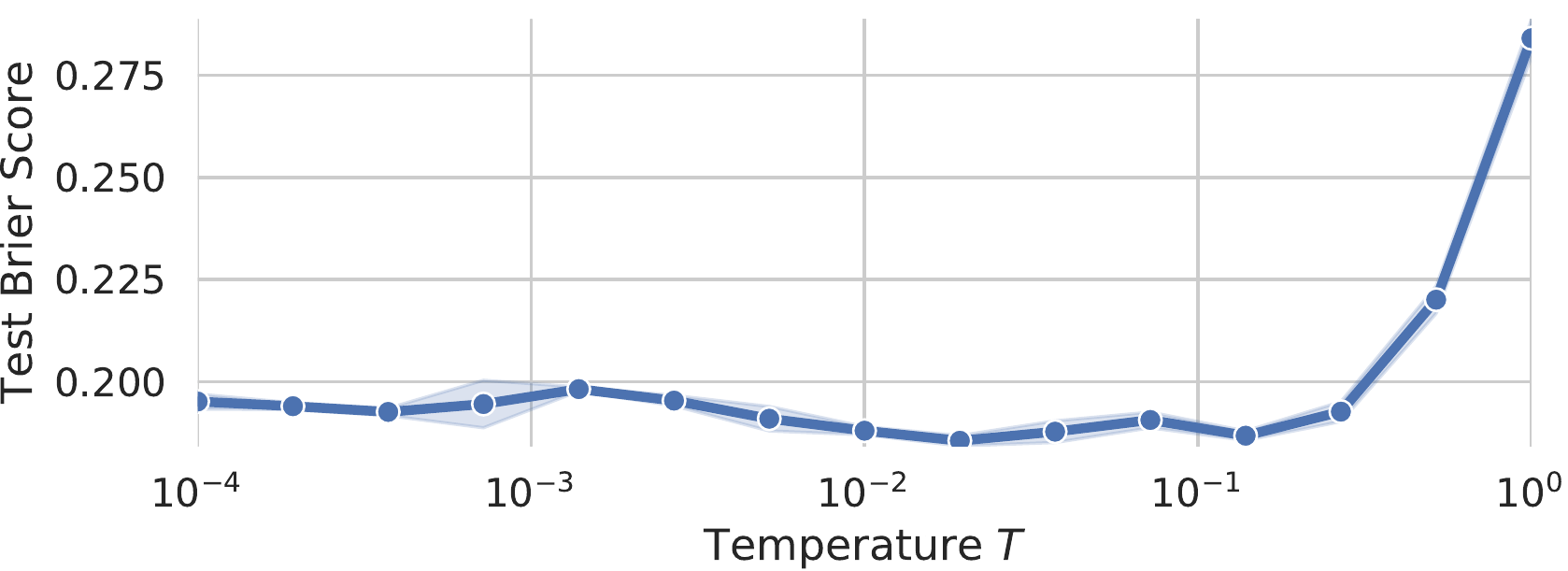}}%
\center{\includegraphics[width=\columnwidth]{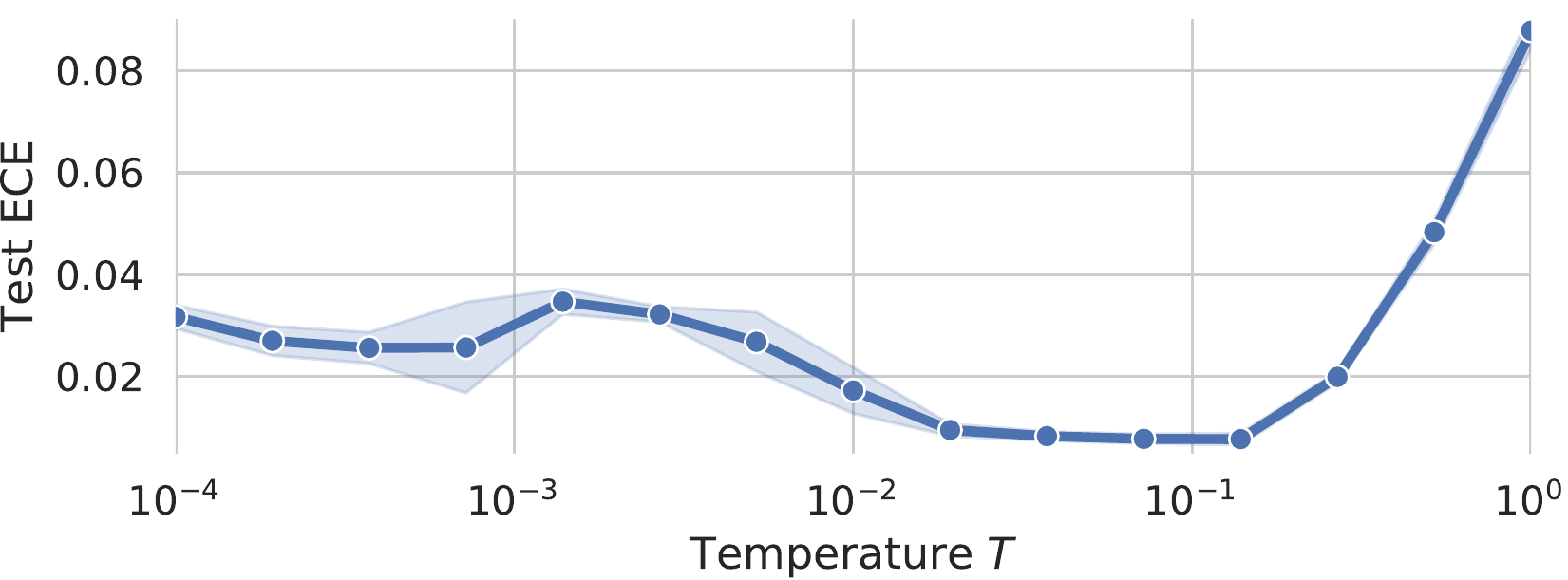}}%
\vspace{-0.3cm}%
\caption{CNN-LSTM/IMDB: 
Cold posteriors also improve the uncertainty metrics Brier score and expected calibration error (ECE) (lower is better).
The plots for accuracy and cross entropy are shown in Figure~\ref{fig:imdb-tempdep}.}
\label{fig:tempplots-appendix2}%
\end{figure}

\section{Details on the Experiment for the Implicit Initialization Prior in SGD Hypothesis}\label{sec:implic_init-appendix}
SGD and SG-MCMC are setup as described in Appendix~\ref{sec:resnet-details}. In the main paper the test accuracy as function of epochs is shown in Figure~\ref{fig:temperature_010}. In Figure~\ref{fig:ce_temperature_010} we additionally report the test cross entropy for the same experiment. SGD initialized by the last model of the SG-MCMC sampling dynamics also recovers the same performance in terms of cross entropy as vanilla SGD.

\begin{figure}%
\vspace{-0.25cm}%
\center{\includegraphics[width=\columnwidth]{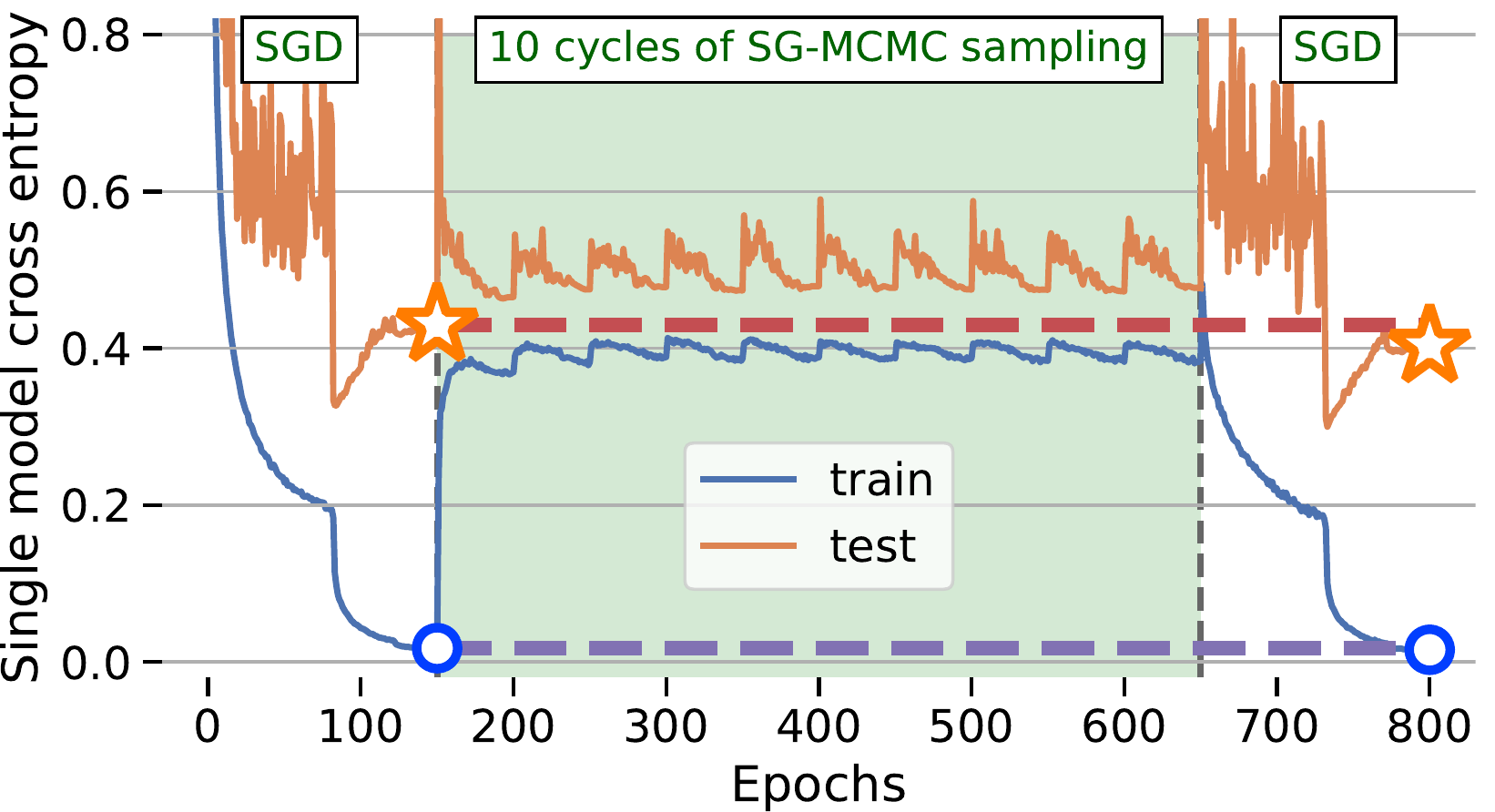}}%
\vspace{-0.3cm}%
\caption{Do the SG-MCMC dynamics harm a beneficial initialization bias used by SGD?
We first train a ResNet-20 on CIFAR-10 via SGD, then switch over to SG-MCMC sampling and finally switch back to SGD optimization.
We report the single-model test cross entropy of SGD and the SG-MCMC chain as function of epochs. SGD recovers from being initialized by the SG-MCMC state.}%
\label{fig:ce_temperature_010}%
\vspace{-0.25cm}%
\end{figure}

\section{Diagnostics: Temperatures}\label{sec:diagnostics}
The following proposition adapted from~\citep[Section 6.1.5]{leimkuhler2016molecular} provides a general way to construct temperature observables.

\begin{proposition}[Constructing Temperature Observables]\label{prop:temperature}
Given a Hamiltonian $H(\pars,\moms)$ corresponding to Langevin dynamics,
\begin{equation}
H(\pars,\moms) = \frac{1}{T} U(\pars) + \frac{1}{2} \moms^T \mathbf{M}^{-1} \moms,
\label{eqn:langevin-hamiltonian}
\end{equation}
and an arbitrary smooth vector field $B: \mathbb{R}^d \times \mathbb{R}^d \to \mathbb{R}^d \times \mathbb{R}^d$ satisfying
\begin{compactitem}
\item $0 < \mathbb{E}_{(\pars,\moms)}[\langle B(\pars,\moms), \nabla H(\pars,\moms)\rangle] < \infty$,
\item $0 < \mathbb{E}_{(\pars,\moms)}[\langle \mathbf{1}_{2d}, \nabla B(\pars,\moms)\rangle] < \infty$, and
\item $\|B(\pars,\moms) \, \exp(- H(\pars,\moms))\| < \infty$ for all $(\pars,\moms) \in \mathbb{R}^d \times \mathbb{R}^d$,
\end{compactitem}
then
\begin{equation}
T = \frac{\mathbb{E}_{(\pars,\moms)}[\langle B(\pars,\moms), \nabla H(\pars,\moms)\rangle]}{
    \mathbb{E}_{(\pars,\moms)}[\langle \mathbf{1}_{2d}, \nabla B(\pars,\moms)\rangle]}.
    \label{eqn:temperature-observable}
\end{equation}
\end{proposition}

Note that for the Hamiltonian~(\ref{eqn:langevin-hamiltonian}) we have, assuming a symmetric preconditioner, $(\mathbf{M}^{-1})^T = \mathbf{M}^{-1}$,
\begin{eqnarray}
\nabla_{\pars} H(\pars,\moms) & = & \frac{1}{T} \nabla_{\pars} U(\pars),\\
\nabla_{\moms} H(\pars,\moms) & = & \mathbf{M}^{-1} \, \moms.
\end{eqnarray}

\subsection{Kinetic Temperature Estimation}\label{sec:kinetic-temp}
Simulating the Langevin dynamics, equations~(\ref{eqn:langevin1}--\ref{eqn:langevin2}) from the main paper, produces moments $\moms$ which are jointly distributed according to a multivariate Normal distribution,~\citep{leimkuhler2016molecular},
\begin{equation}
\moms \sim \mathcal{N}(0,\mathbf{M}).\label{eqn:moment-distribution}
\end{equation}
The \emph{kinetic temperature} $\hat{T}_K(\moms)$ is derived from the moments as
\begin{equation}
\hat{T}_K(\moms) := \frac{\moms^T \, \mathbf{M}^{-1} \, \moms}{d},
\end{equation}
and we have that for a perfect simulation of the dynamics we achieve $\mathbb{E}[\hat{T}_K(\moms)] = T$, where $T$ is the target temperature of the system,~\citep{leimkuhler2016molecular}.
This can be seen by instantiating Proposition~\ref{prop:temperature} for the Langevin Hamiltonian and $B_K(\pars,\moms)=\left[\begin{array}{c}\mathbf{0}\\ \moms\end{array}\right]$.

In general we only approximately solve the SDE and errors in the solution arise due to discretization, minibatch noise, or lack of full equilibration to the stationary distribution.
Therefore, we can use $\hat{T}_K(\moms)$ as a diagnostic to measure the temperature of the current system state, and a deviation from the target temperature could diagnose poor solution accuracy.
To this end, we know that if $\moms \sim \mathcal{N}(0,\mathbf{M})$ then $(\mathbf{M}^{-1/2} \, \moms) \sim \mathcal{N}(0,I_d)$ and thus the inner product $(\mathbf{M}^{-1/2} \, \moms)^T \, (\mathbf{M}^{-1/2} \, \moms) = \moms^T \, \mathbf{M}^{-1} \, \moms$ is distributed according to a standard $\chi^2$-distribution with $d$ degrees of freedom,
\begin{equation}
    (\moms^T \mathbf{M}^{-1} \moms) \sim \chi^2(d).
    \label{eqn:tk-chi2}
\end{equation}
The $\chi^2(d)$ distribution has mean $d$ and variance $2d$ and we can use the tail probabilities to test whether the observed temperature could arise from an accurate discretization of the SDE~(\ref{eqn:langevin1}--\ref{eqn:langevin2}).
For a given confidence level $c \in (0,1)$, e.g. $c=0.99$, we define the confidence interval
\begin{equation}
J_{T_K}(d,c) := \left(\frac{T}{d} F^{-1}_{\chi^2(d)}\left(\frac{1-c}{2}\right),\,
    \frac{T}{d} F^{-1}_{\chi^2(d)}\left(\frac{1+c}{2}\right)\right),
\end{equation}
where $F^{-1}_{\chi^2(d)}$ is the inverse cumulative distribution function of the $\chi^2$ distribution with $d$ degrees of freedom.
By construction if~(\ref{eqn:tk-chi2}) holds, then $\hat{T}_K(\moms) \in J_{T_K}(d,c)$ with probability $c$ exactly.

Therefore, if $c$ is close to one, say $c = 0.99$, and we find that $\hat{T}_K(\moms) \notin J(d,c)$ this indicates issues of discretization error or convergence of the SDE~(\ref{eqn:langevin1}--\ref{eqn:langevin2}).

Because~(\ref{eqn:moment-distribution}) holds for any subvector of $\moms$, we can create one kinetic temperature estimate for each model variable separately, such as one or two scalar temperature estimates for each layer (e.g. one for the weights and one for the bias of a \texttt{Dense} layer).
We found per-layer temperature estimates helpful in diagnosing convergence issues and this directly led to the creation of our layerwise preconditioner.

\subsection{Configurational Temperature Estimation}\label{sec:conf-temp}

The so called \emph{configurational temperature}\footnote{Sometimes other quantities are also refered to as configurational temperature, see~\citep[Section 6.1.5]{leimkuhler2016molecular}.} is defined as
\begin{equation}
\hat{T}_C(\pars, \nabla_{\pars} U(\pars)) = \frac{\langle \pars, \nabla_{\pars} U(\pars) \rangle}{d}.
\label{eqn:temp-conf}
\end{equation}
For a perfect simulation of SDE~(\ref{eqn:langevin1}--\ref{eqn:langevin2}) we have $\mathbb{E}[\hat{T}_C]=T$, where $T$ is the target temperature of the system.
This can be seen by instantiating Proposition~\ref{prop:temperature} for the Langevin Hamiltonian and $B_C(\pars,\moms)=\left[\begin{array}{c}\pars\\ \mathbf{0}\end{array}\right]$.

As for the kinetic temperature diagnostic, we can instantiate Proposition~\ref{prop:temperature} for arbitrary subsets of parameters by a suitable choice of $B_C(\pars,\moms)$.
However, whereas for the kinetic temperature the exact sampling distribution of the estimate is known in the form of a scaled $\chi^2$ distribution, we are not aware of a characterization of the sampling distribution of configurational temperature estimates.
It is likely this sampling distribution depends on $U(\pars)$ and thus does not have a simple form.
Proposition~\ref{prop:temperature} only asserts that under the true target distribution we have
\begin{equation}
    \mathbb{E}_{\pars \sim \exp(-U(\pars)/T)}[\hat{T}_C(\pars,\nabla_\pars U(\pars))] = T.
\end{equation}
Because~(\ref{eqn:temp-conf}) is the empirical average of per parameter random variables, if all these variables have finite variance the central limit theorem asserts that for large $d$ we can expect
\begin{equation}
    \hat{T}_C(\pars,\nabla_\pars U(\pars)) \sim \mathcal{N}(T, \sigma^2_{T_C}),
\end{equation}
with unknown variance $\sigma^2_{T_C}$.

Recent work of~\citet{yaida2018fluctuationdissipationsgd} provides a similar diagnostic, equation (FDR1') in their work, to the configurational temperature~(\ref{eqn:temp-conf}) for the SGD equilibrium distribution under finite time dynamics.  However, our goal here is different: whereas~\citet{yaida2018fluctuationdissipationsgd} is interested in diagnosing convergence to the SGD equilibrium distribution in order to adjust learning rates we instead want to diagnose discrepancy of our current dynamics against the true target distribution.

\section{Simulation Accuracy Ablation Study}\label{sec:simablation}
Equipped with the diagnostics of Section~\ref{sec:diagnostics} we can now study how accurate our algorithms simulate the Langevin dynamics.
We will demonstrate that layerwise preconditioning and cyclical time stepping are individually effective at improving simulation accuracy, however, only by combining these two methods we can achieve high simulation accuracy on the CNN-LSTM model as measured by our diagnostics.

\paragraph{Setup.}
We perform the same ResNet-20 CIFAR-10 and CNN-LSTM IMDB experiments as in the main paper, but consider four variations of our algorithm: with and without preconditioning, and with and without cosine time stepping schedules.
In case no preconditioner is used we simply set $\mathbf{M}=I$ for all iterations.
In case no cosine time stepping is used we simply set $C(t) = 1$ for all iterations.

Independent of whether cosine time stepping is used we divide the iterations into cycles and for each method consider all models at the end of a cycle, where we hope simulation accuracy is the highest.
We then evaluate the temperature diagnostics for all model variables.
For the kinetic temperatures, if simulation is accurate then 99 percent of the variables should on average lie in the 99\% high probability region under the sampling distribution.
For the configurational temperature we can only report the average configurational temperature across all the end-of-cycle models.

\paragraph{Results.}
We report the results in Table~\ref{tab:resnet-simaccuracy-ablation} and Table~\ref{tab:cnnlstm-simaccuracy-ablation} and visualize the kinetic temperatures in Figures~\ref{fig:resnet-simablation-precond-cosine}~to~\ref{fig:resnet-simablation-noprecond-flat} and
Figures~\ref{fig:cnnlstm-simaccuracy-ablation-p-c}~to~\ref{fig:cnnlstm-simaccuracy-np-f}.

The results indicate that both cosine time stepping and layerwise preconditioning have a beneficial effect on simulation accuracy.  For ResNet-20 cyclical time stepping is sufficient for high simulation accuracy, but it is by itself not able to achieve high accuracy on the CNN-LSTM model.  For both models the combination of cyclical time stepping and preconditioning  (Figure~\ref{fig:resnet-simablation-precond-cosine} and Figure~\ref{fig:cnnlstm-simaccuracy-ablation-p-c}) achieves a high simulation accuracy, that is, all kinetic temperatures match the sampling distribution of the Langevin dynamics, indicating---at least with respect to the power of our diagnostics---accurate simulation.

Another interesting observation can be seen in Table~\ref{tab:resnet-simaccuracy-ablation}:
we can achieve a high accuracy of $\geq 88$ percent even in cases where the simulation accuracy is poor.
This indicates that optimization is different from accurate Langevin dynamics simulation.

\begin{table*}[th!]
\begin{center}%
\begin{tabular}{cccccc}\toprule
Precond & Cyclic & $\hat{\mathbb{E}}[\hat{T}_K \in \mathcal{R}_{99}]$ & $\hat{\mathbb{E}}[\hat{T}_C]$ & Accuracy (\%) & Cross-entropy\\ \midrule
\cmark & \cmark & 0.989$\pm$0.0014 & 0.94$\pm$0.011 & 88.2$\pm$0.11 & 0.358$\pm$0.0011\\
\xmark & \cmark & 0.9772$\pm$0.00059 & 1.02$\pm$0.018 & 88.49$\pm$0.014 & 0.3500$\pm$0.00064\\
\cmark & \xmark & 0.905$\pm$0.0019 & 1.23$\pm$0.046 & 88.0$\pm$0.10 & 0.3808$\pm$0.00064\\
\xmark & \xmark & 0.676$\pm$0.0052 & 1.7$\pm$0.18 & 86.86$\pm$0.072 & 0.507$\pm$0.0080\\
\bottomrule
\end{tabular}%
\caption{ResNet-20 CIFAR-10 simulation accuracy ablation at $T=1$:
layerwise preconditioning and cyclical time stepping each have a beneficial effect on improving inference accuracy and the effect is complementary.
$\hat{\mathbb{E}}[\hat{T}_K \in \mathcal{R}_{99}]$ is the empirically estimated probability that the kinetic temperature statistics are in the 99\% confidence interval, the ideal value is 0.99.
$\hat{\mathbb{E}}[\hat{T}_C]$ is the empirical average of the configurational temperature estimates, the ideal value is 1.0.
For both quantities we take the value achieved at the end of each cycle, that is, whenever $C(t)=0$ and average all the resulting values.
The deviation is given in $\pm \textrm{SEM}$ where $\textrm{SEM}$ is the standard error of the mean estimated from three independent experiment replicates.
Both preconditioning and cyclical time stepping are effective at improving the simulation accuracy.}
\label{tab:resnet-simaccuracy-ablation}%
\end{center}%
\vspace{0.3cm}%
\end{table*}

\begin{table*}[th!]
\begin{center}%
\begin{tabular}{cccccc}\toprule
Precond & Cyclic & $\hat{\mathbb{E}}[\hat{T}_K \in \mathcal{R}_{99}]$ & $\hat{\mathbb{E}}[\hat{T}_C]$ & Accuracy (\%) & Cross-entropy\\ \midrule
\cmark & \cmark & 0.954$\pm$0.0053 & 0.99122$\pm$0.000079 & 81.95$\pm$0.22 & 0.425$\pm$0.0032\\
\xmark & \cmark & 0.761$\pm$0.0095 & 1.012$\pm$0.0088 & 51.3$\pm$0.65 & 0.6925$\pm$0.00019\\
\cmark & \xmark & 0.49$\pm$0.012 & 0.9933$\pm$0.00019 & 74.5$\pm$0.49 & 0.579$\pm$0.0048\\
\xmark & \xmark & 0.384$\pm$0.0018 & 1.0141$\pm$0.00066 & 0.49997$\pm$0.000039 & 0.698$\pm$0.0013\\
\bottomrule
\end{tabular}%
\caption{CNN-LSTM IMDB simulation accuracy ablation at $T=1$: with \emph{both} layerwise preconditioning and cyclical time stepping we can achieve high inference accuracy as measured by configurational and kinetic temperature diagnostics.  Just using one (either preconditioning or cyclical time stepping) is insufficient for high inference accuracy.
This is markedly different from the results obtained for ResNet-20 CIFAR-10 (Table~\ref{tab:resnet-simaccuracy-ablation}), indicating that perhaps the ResNet posterior is easier to sample from.}
\label{tab:cnnlstm-simaccuracy-ablation}%
\end{center}%
\vspace{0.3cm}%
\end{table*}

\begin{figure*}[!t]
\center{\includegraphics[width=\textwidth]{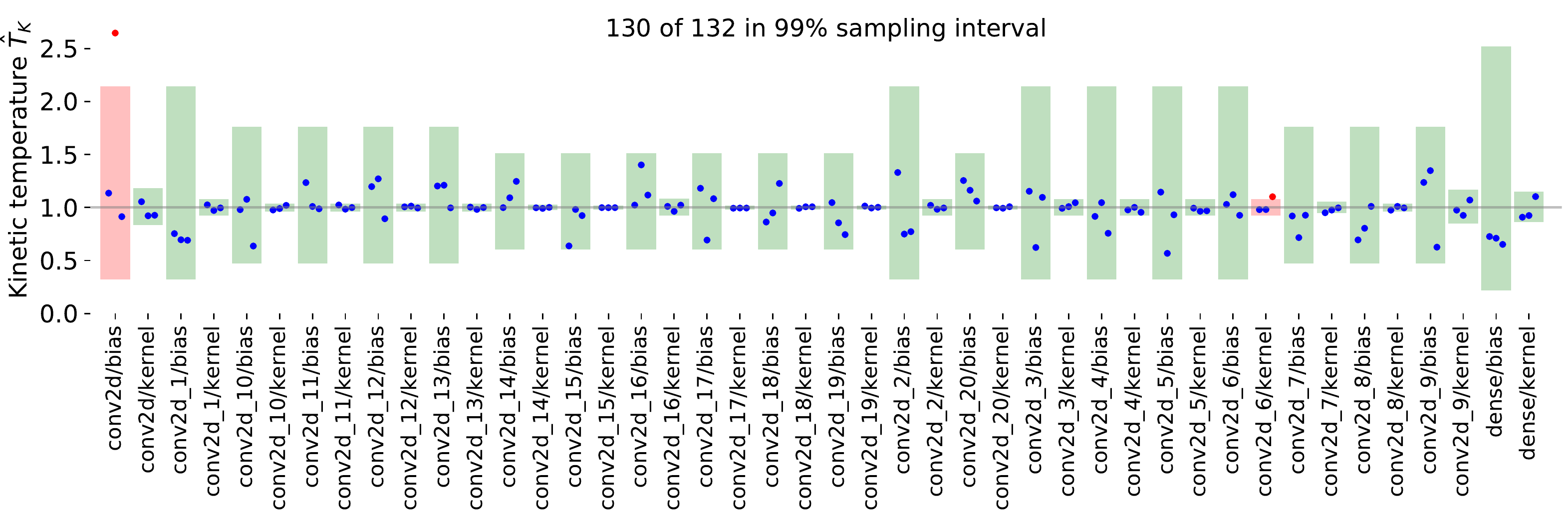}}%
\caption{ResNet-20 CIFAR-10 Langevin per-variable kinetic temperature estimates \textbf{with preconditioning} and \textbf{with cosine time stepping schedule}.
The green bars show the 99\% true sampling distribution of the Kinetic temperature sample.  The blue dots show the actual kinetic temperature samples at the end of sampling.  About 1\% of variables should be outside the green boxes, which matches the empirical count (2 out of 132 samples), indicating an accurate simulation of the Langevin dynamics at the end of each cycle.}%
\label{fig:resnet-simablation-precond-cosine}%
\vspace{0.3cm}%
\end{figure*}

\begin{figure*}[!t]
\center{\includegraphics[width=\textwidth]{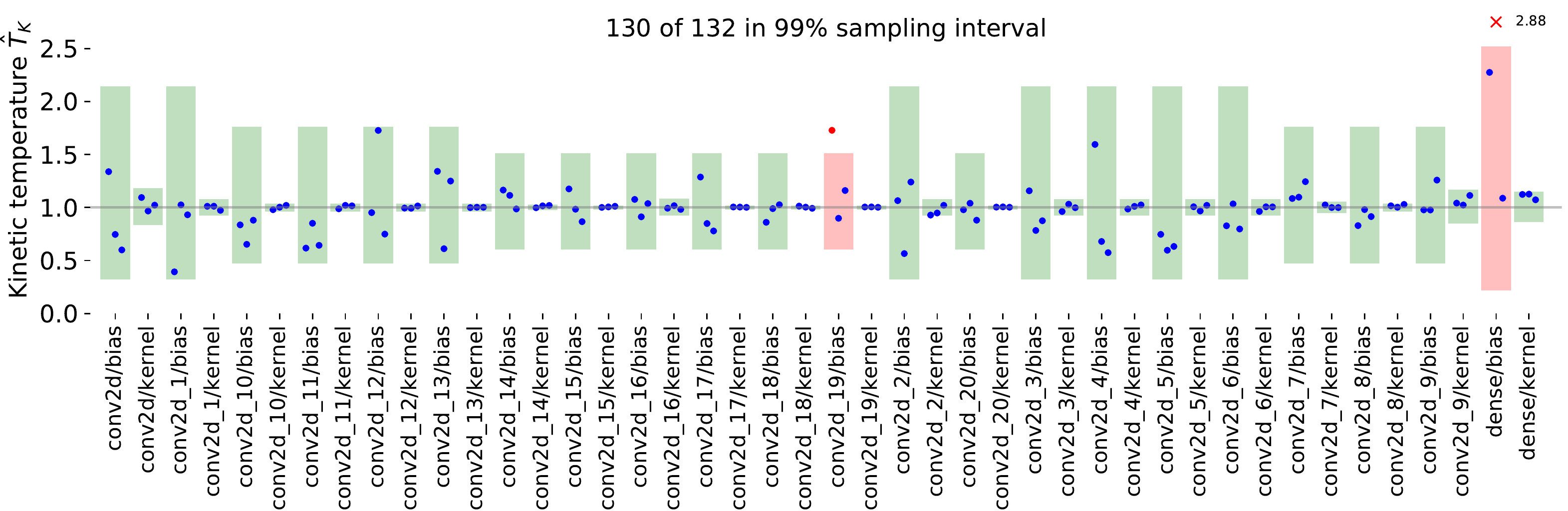}}%
\caption{ResNet-20 CIFAR-10 Langevin per-variable kinetic temperature estimates \textbf{without preconditioning} but \textbf{with cosine time stepping schedule}.  Two out of 132 variables are outside the 99\% hpd region, indicating accurate simulation.}%
\label{fig:resnet-simablation-noprecond-cosine}%
\vspace{0.3cm}%
\end{figure*}

\begin{figure*}[!t]
\center{\includegraphics[width=\textwidth]{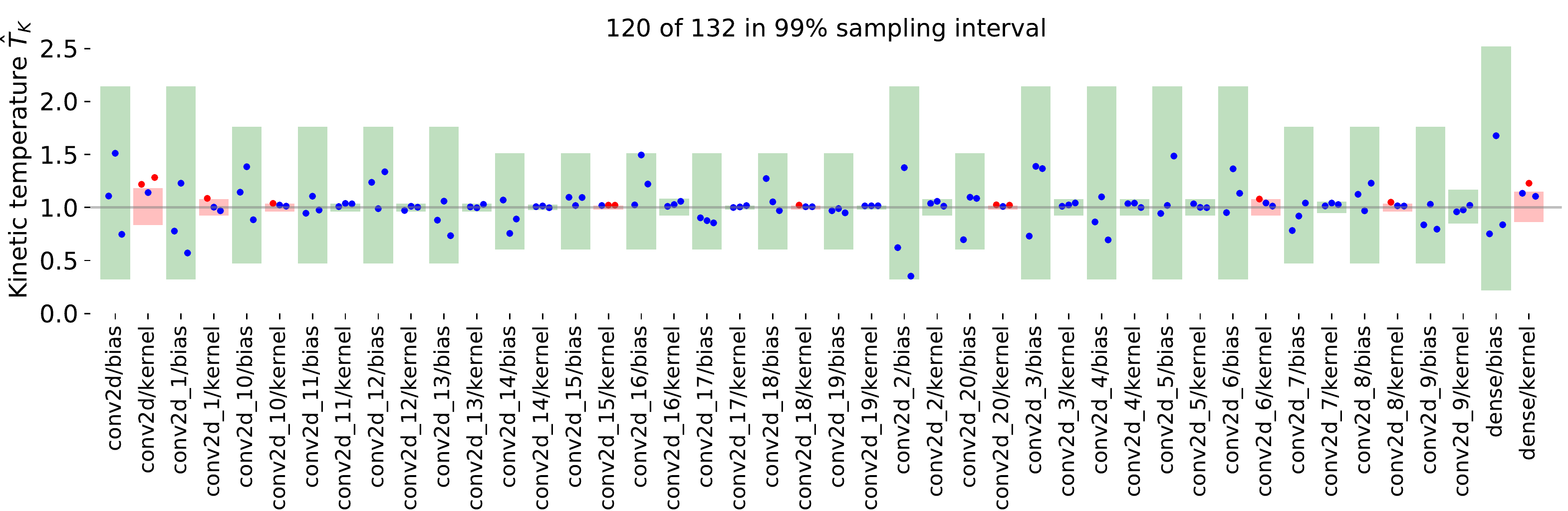}}%
\caption{ResNet-20 CIFAR-10 Langevin per-variable kinetic temperature estimates \textbf{with preconditioning} but \textbf{without cosine time stepping schedule} (flat schedule).  12 out of 132 variables are too hot (boxes in red) and lie outside the acceptable region, indicating an inaccurate simulation of the Langevin dynamics.  However, there is a marked improvement due to preconditioning compared to no preconditioning (Figure~\ref{fig:resnet-simablation-noprecond-flat}).}%
\label{fig:resnet-simablation-precond-flat}%
\vspace{0.3cm}%
\end{figure*}

\begin{figure*}[!t]
\center{\includegraphics[width=\textwidth]{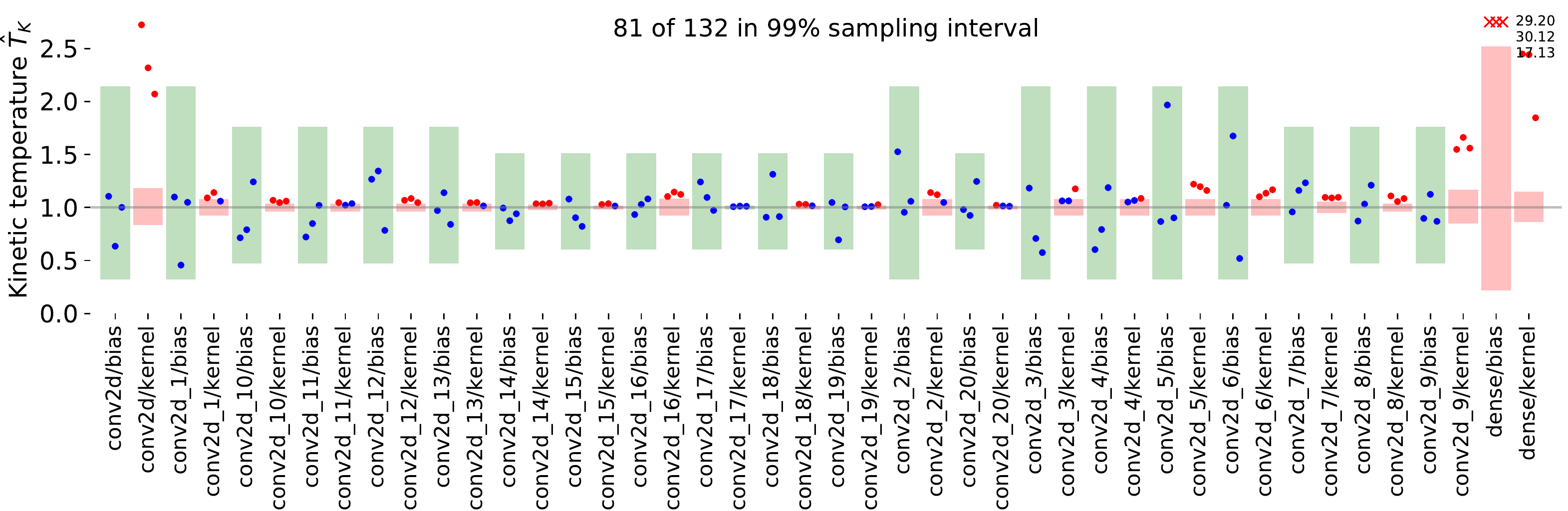}}%
\caption{ResNet-20 CIFAR-10 Langevin per-variable kinetic temperature estimates \textbf{without preconditioning} and \textbf{without cosine time stepping schedule} (flat schedule).
51 out of 132 kinetic temperature samples are too hot (shaded in red) and lie outside the acceptable region, sometimes severely so, indicating a very poor simulation accuracy for the Langevin dynamics.}%
\label{fig:resnet-simablation-noprecond-flat}%
\vspace{0.3cm}%
\end{figure*}

\begin{figure*}%
\hfill%
\begin{subfigure}[t]{0.43\textwidth}%
\includegraphics[width=\textwidth]{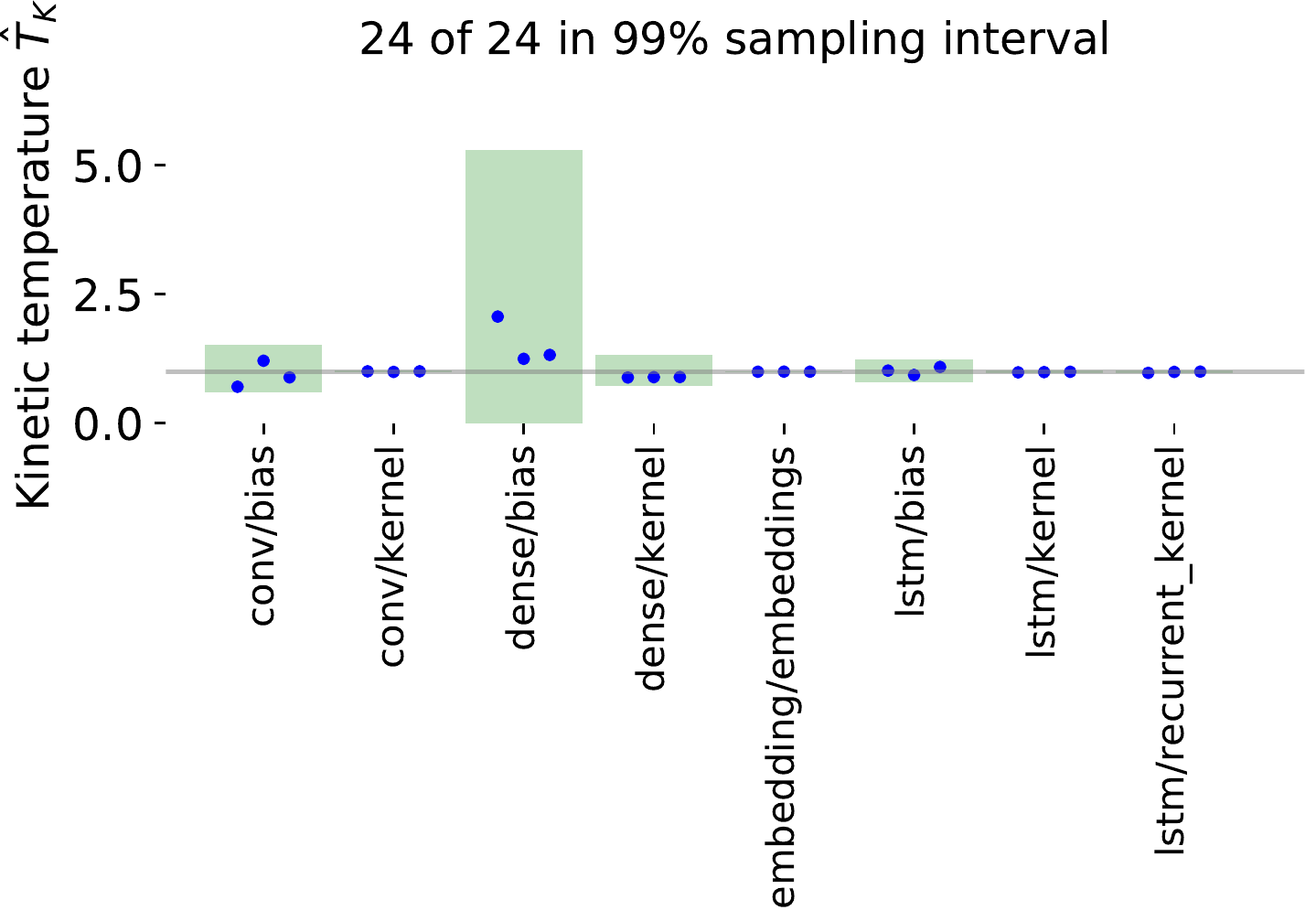}%
\caption{Preconditioning, cosine stepping}%
\label{fig:cnnlstm-simaccuracy-ablation-p-c}%
\end{subfigure}%
\hfill%
\begin{subfigure}[t]{0.43\textwidth}%
\includegraphics[width=\textwidth]{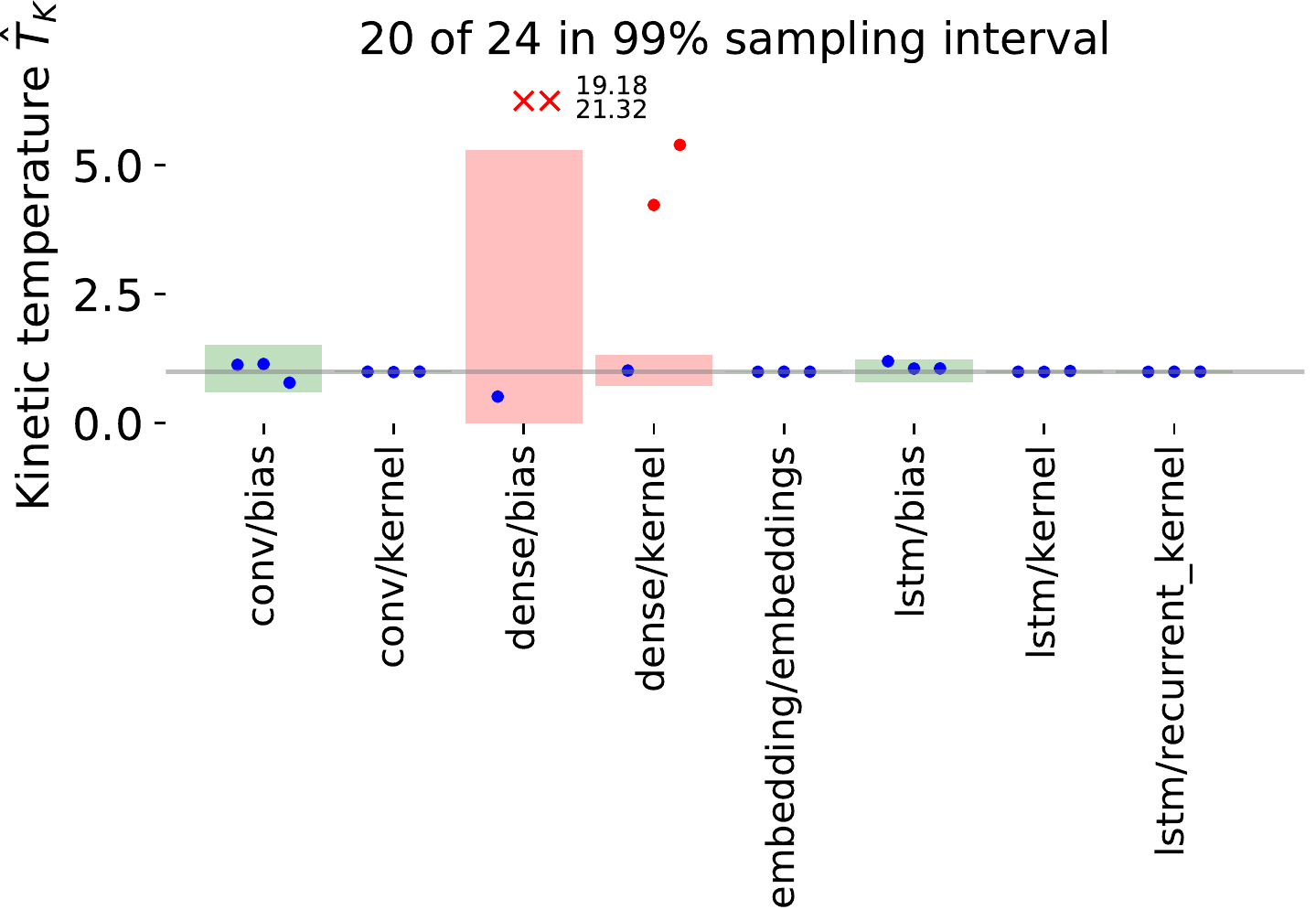}%
\caption{No preconditioning, cosine stepping}%
\label{fig:cnnlstm-simaccuracy-ablation-np-c}%
\end{subfigure}%
\hfill%
\newline%
\vspace{0.75cm}%
\hfill%
\begin{subfigure}[t]{0.43\textwidth}%
\vspace{0.5cm}%
\includegraphics[width=\textwidth]{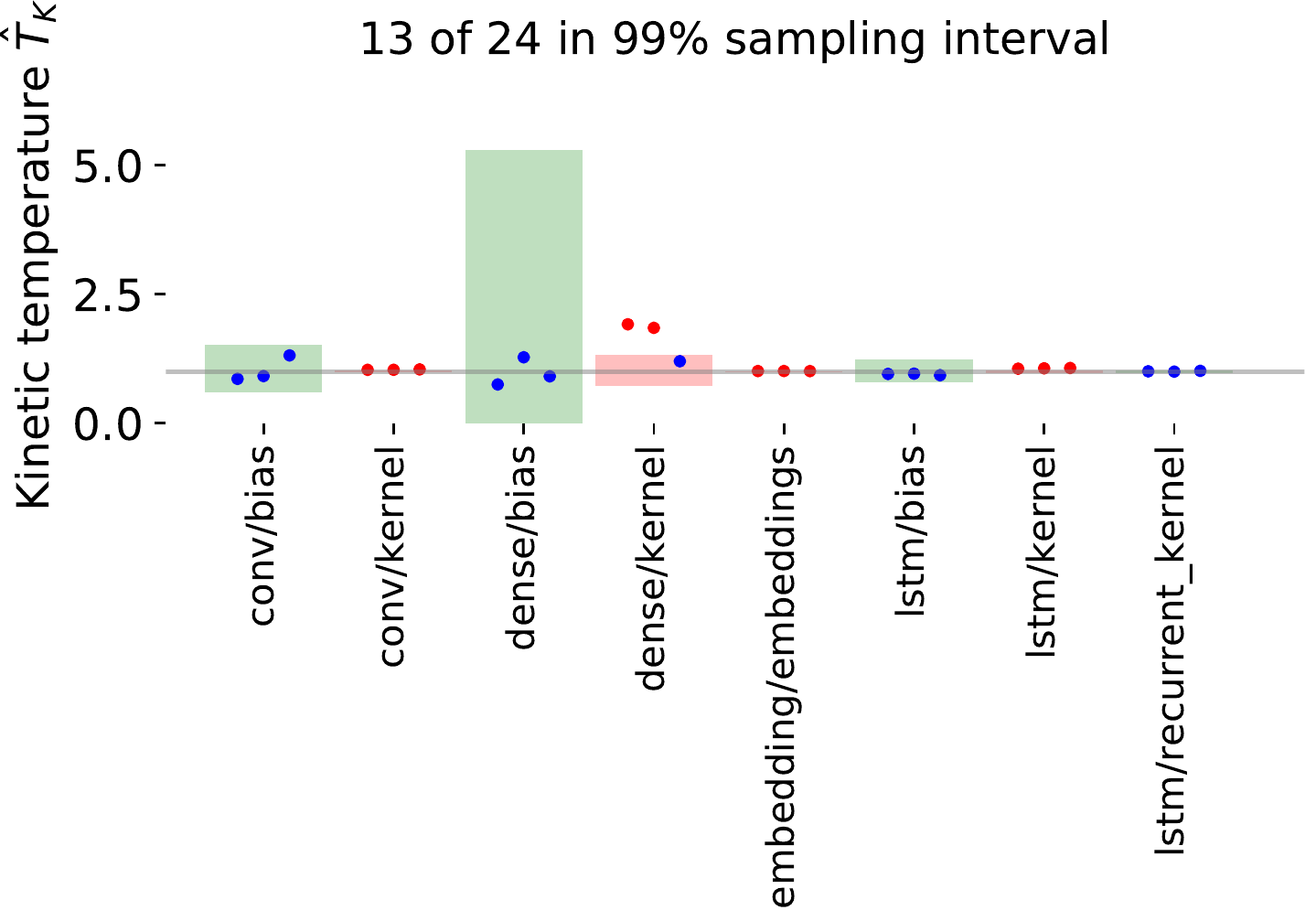}%
\caption{Preconditioning, no cosine stepping}%
\label{fig:cnnlstm-simaccuracy-ablation-p-f}%
\end{subfigure}%
\hfill%
\begin{subfigure}[t]{0.43\textwidth}%
\vspace{0.5cm}%
\includegraphics[width=\textwidth]{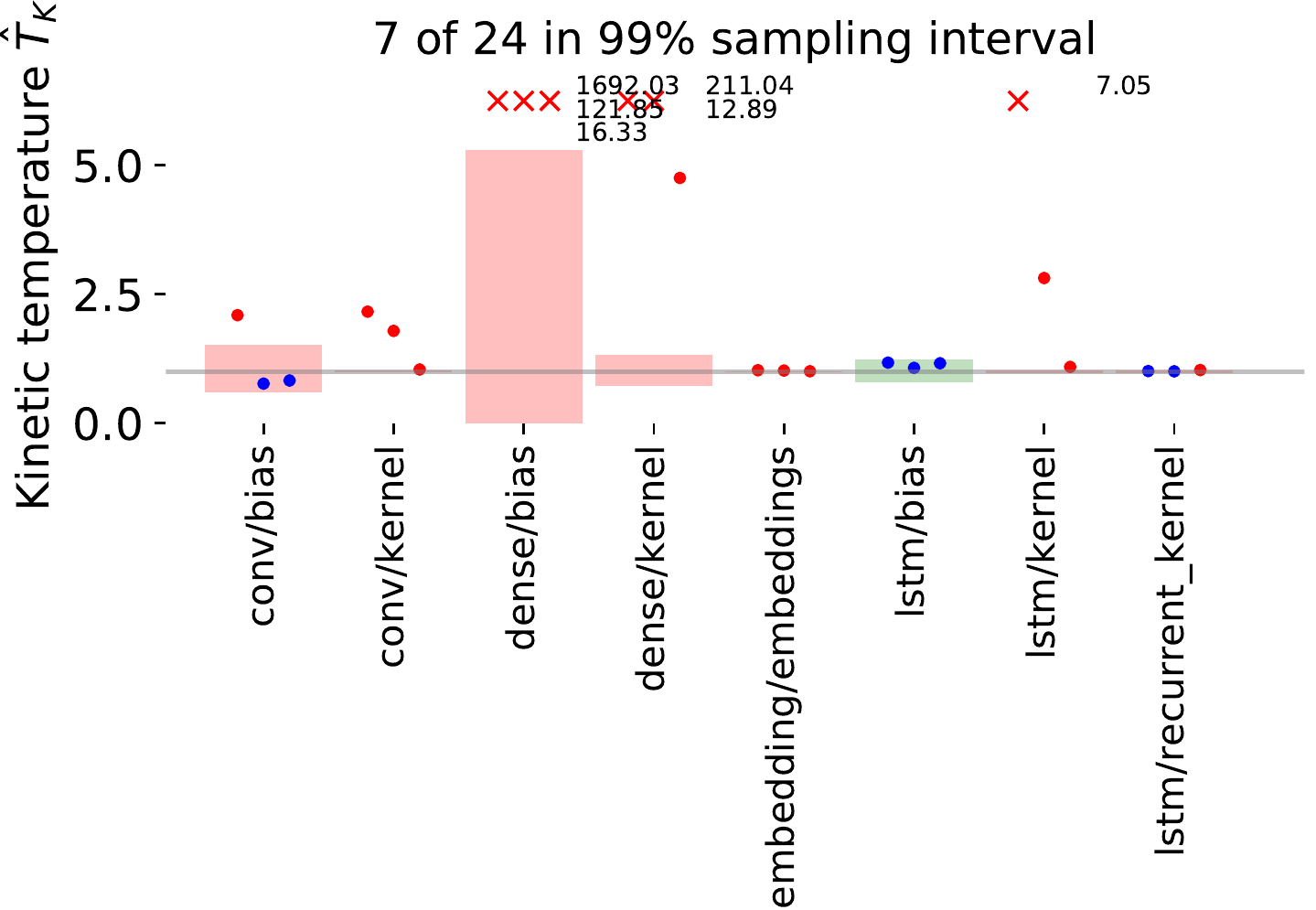}%
\caption{No preconditioning, no cosine stepping}%
\label{fig:cnnlstm-simaccuracy-np-f}%
\end{subfigure}%
\hfill%
\caption{CNN-LSTM IMDB Langevin per-variable kinetic temperature estimates at temperature $T=1$ for four different simulation settings: with and without preconditioning, with and without cosine time stepping.
The only accurate simulation is obtained with \emph{both} preconditioning \emph{and} cosine time stepping.}%
\label{fig:cnnlstm-simaccuracy-ablation}%
\end{figure*}

\section{Dirty Likelihood Functions}\label{sec:dirtylikelihood}

\hypothesis{Dirty Likelihood Hypothesis}{Deep learning practices that violate the likelihood principle (batch normalization, dropout, data augmentation) cause deviation from the Bayes posterior.}
We now discuss how batch normalization, dropout, and data augmentation produce non-trivial modifications to the likelihood function.
We call the resulting likelihood functions ``dirty'' to distinguish them from clean likelihood functions without such modifications.
Our discussion will suggest that these techniques can be seen as a computational efficient ``\emph{Jensen posterior}'' approximation of a proper Bayesian posterior of another model.
Our analysis builds on and generalizes previous Bayesian interpretations, \citep{noh2017regularizing,atanov2018stochasticbatchnormalization,shekhovtsov2018stochasticnormalizations,nalisnick2019dropout,inoue2019multisampledropout}.
In Section~\ref{sec:frn-experiment} we perform an experiment to demonstrate that the dirty likelihood cannot explain cold posteriors.

\begin{wrapfigure}{r}{0.4\columnwidth}
    \centering
\scalebox{0.9}{%
\begin{tikzpicture}[x=0.7cm,y=0.7cm]
\node[obs]                       (y)     {$y_i$} ; %
\node[latent, left=0.7cm of y]   (theta) {$\pars$} ; %
\node[latent, above=of y]        (xa)    {$x'_i$} ; %
\node[latent, right=0.7cm of xa] (z)     {$z_i$} ; %
\node[obs, above=of xa]          (x)     {$x_i$} ; %

\edge {x,z} {xa} ; %
\edge {xa,theta} {y} ; %

\plate {} {(x)(z)(xa)(y)} {$i=1,\dots,n$} ;
\end{tikzpicture}%
}%
\vspace{-0.25cm}%
\caption{Augmented model with added latent variable $z_i$.}%
\label{fig:dnn-latent}%
\end{wrapfigure}

\subsection{Augmented Latent Model}
To accommodate popular deep learning methods we first augment the probabilistic model $p(y|x,\pars)$ itself by adding a \emph{latent variable} $z$.
The augmented model is $p(y|x,z,\pars)$ and we can obtain the \emph{effective model} $p(y|x,\pars) = \int p(y|x,z,\pars) \, p(z) dz$.
For a dataset $\mathcal{D}=\{(x_i,y_i)\}_{i=1,\dots,n}$, where we denote $X=(x_1,\dots,x_n)$ and $Y=(y_1,\dots,y_n)$, the resulting model has as likelihood function in $\pars$  that is the \emph{marginal likelihood}, obtained by integrating over all $z_i$ variables,
\begin{eqnarray}
p(Y \,|\, X, \pars) & = & \prod_{i=1}^n \, p(y_i \,|\, x_i,\pars)\\
& = & \prod_{i=1}^n \mathbb{E}_{z_i \sim p(z_i)}[p(y_i \,|\, x_i, z_i, \pars)].\label{eqn:augmented-marginal-likelihood}
\end{eqnarray}
Note that in~(\ref{eqn:augmented-marginal-likelihood}) the latent variable $z_i$ is integrated out and therefore the marginal likelihood is a deterministic function.

\subsection{Log-likelihood Bound and Jensen Posterior}\label{sec:jensen}
Given a prior $p(\pars)$ the log-posterior for the augmented model in Figure~\ref{fig:dnn-latent} takes the form
\begin{align}
& \log p(\pars \,|\, \mathcal{D})\\
&=  C + \log p(\pars) + \sum_{i=1}^n \log \mathbb{E}_{z_i \sim p(z_i)}[p(y_i \,|\, x_i, z_i, \pars)],\label{eqn:posterior}
\end{align}
where we can now apply \emph{Jensen's inequality}, $f(\mathbb{E}[x]) \geq \mathbb{E}[f(x)]$ for concave $f = \log$,
\begin{align}
& \geq C + \log p(\pars) + \sum_{i=1}^n \mathbb{E}_{z_i \sim p(z_i)}[\log p(y_i \,|\, x_i, z_i, \pars)],\label{eqn:jensen}
\end{align}
where $C=-\log p(Y|X)$ is the negative model evidence and is constant in $\pars$.
We call equation~(\ref{eqn:jensen}) the \emph{Jensen bound} to the log-posterior $\log p(\pars|\mathcal{D})$.

\paragraph{Jensen Posterior.}
Because we can estimate~(\ref{eqn:jensen}) in an unbiased manner, we will see that many popular methods such as dropout and data augmentation can be cast as special cases of the Jensen bound.
We also define the \emph{Jensen posterior} as the posterior distribution associated with~(\ref{eqn:jensen}).
Formally, the Jensen posterior is
\begin{align}
& p_J(\pars \,|\, \mathcal{D}) :\propto\\
& \qquad p(\pars) \, \prod_{i=1}^n \exp\left(\mathbb{E}_{z_i \sim p(z_i)}\left[\log p(y_i \,|\, x_i, z_i, \pars)\right]\right).\label{eqn:jensenposterior}
\end{align}
Given this object, can we relate its properties to the properties of the full posterior, and can the Jensen posterior serve as a meaningful surrogate to the true posterior?
We first observe that $p_J(\pars \,|\,\mathcal{D})$ indeed defines a probability distribution over parameters:
with a proper prior $p(\pars)$, we have $p(\pars\,|\,\mathcal{D}) \geq p_J(\pars\,|\,\mathcal{D})$
by (\ref{eqn:posterior}--\ref{eqn:jensen}), thus
$\int p_J(\pars \,|\,\mathcal{D}) \,\textrm{d}\pars \leq \int p(\pars\,|\,\mathcal{D}) \,\textrm{d}\pars < \infty$.

\paragraph{Jensen Prior.}
We now show that the Jensen posterior can be interpreted as a full Bayesian posterior in a different model.
In particular, we give a construction which retains the likelihood of the original model but modifies the prior.
In the function that re-weights the prior the data set appears; this is not to be understood as a prior which depends on the observed data.
Instead, we can think of this as an existence proof, that is, if we were to have chosen this modified prior then the resulting Jensen posterior under the modified Jensen prior corresponds to the full Bayesian posterior under the original prior.

In a sense the result is vacuous because any desirable posterior can be obtained by such re-weighting.
However, the proof illustrates the structure of how the Jensen posterior deviates from the true posterior through a set of weighting functions; each weighting function measures a local \emph{Jensen gap} related to each instance.
Although we did not pursue this line, the local Jensen gap~(\ref{eqn:wi}) can be numerically estimated and may prove to be a useful quantity in itself.

\begin{proposition}[Jensen Prior]\label{prop:jensen-prior}
For a proper prior $p(\pars)$ and a fixed dataset $\mathcal{D}$, we can define a prior $p_J(\pars)$ such that when using this modified prior in the Jensen posterior we have
\begin{equation}
p_J(\pars \,|\, \mathcal{D}) = p(\pars \,|\, \mathcal{D}).\label{eqn:jensen-valid}
\end{equation}
In particular, this implies that any Jensen posterior can be interpreted as the posterior distribution of the same model under a different prior.
\end{proposition}
\begin{proof}
We have the true posterior
\begin{equation}
p(\pars \,|\, \mathcal{D})
= p(\pars) \, \prod_{i=1}^n \int p(y_i \,|\, x_i, z_i, \pars) \, p(z_i) \,\textrm{d}z_i,\label{eqn:posterior2}
\end{equation}
and the Jensen posterior as
\begin{equation}
p_J(\pars \,|\, \mathcal{D})
:= p(\pars) \, \prod_{i=1}^n \exp\left(\mathbb{E}_{z_i \sim p(z_i)}\left[\log p(y_i \,|\, x_i, z_i, \pars)\right]\right),\label{eqn:jensenposterior2}
\end{equation}
respectively.  If we define the \emph{Jensen prior},
\begin{equation}
p_J(\pars) :\propto w(\pars) \, p(\pars),
\end{equation}
where we set the weighting function $w(\pars) := \prod_{i=1}^n w_i(\pars)$, with the individual weighting functions defined as
\begin{equation}
w_i(\pars) := \frac{\int p(y_i\,|\,x_i,z_i,\pars) \, p(z_i) \,\textrm{d}z_i}{
\exp \left(\mathbb{E}_{z_i\sim p(z_i)}[\log p(y_i \,|\, x_i, z_i, \pars)]\right)}.\label{eqn:wi}
\end{equation}
Due to Jensen's inequality we have $w_i(\pars) \leq 1$ and hence $w(\pars) \leq 1$ and thus $p_J(\pars)$ is normalizable.
Using $p_J(\pars)$ as prior in~(\ref{eqn:jensenposterior2}) we obtain
\begin{align}
& p_J(\pars \,|\, \mathcal{D})\\
& \propto p_J(\pars) \, \prod_{i=1}^n \exp\left(\mathbb{E}_{z_i \sim p(z_i)}\left[\log p(y_i \,|\, x_i, z_i, \pars)\right]\right),\\
& = p(\pars) \left(\prod_{i=1}^n w_i(\pars)\right)\\
& \qquad\qquad \prod_{i=1}^n \exp\left(\mathbb{E}_{z_i \sim p(z_i)}\left[\log p(y_i \,|\, x_i, z_i, \pars)\right]\right),\\
& = p(\pars) \prod_{i=1}^n \int p(y_i \,|\, x_i, z_i, \pars) \, p(z_i) \,\textrm{d}z_i\\
& \propto p(\pars \,|\, \mathcal{D}).
\end{align}
This constructively demonstrates the result~(\ref{eqn:jensen-valid}).
\end{proof}

We now interpret current deep learning methods as optimizing the Jensen posterior.

\subsection{Deep Learning Techniques Optimize Jensen Posteriors}
\paragraph{Dropout.}
In \emph{dropout} we sample random binary masks $z_i \sim p(z_i)$ and multiply network activations with such masks~\citep{srivastava2014dropout}.
Specializing the above latent variable model to dropout gives an interpretation of doing
maximum aposteriori (MAP) estimation on the Jensen posterior $p_J(\pars\,|\,X,Y)$.

The connection between dropout and applying Jensen's bound has been discovered before by several groups~\cite{noh2017regularizing}, \cite{nalisnick2019dropout},
\cite{inoue2019multisampledropout}, and contrasts sharply with the variational inference interpretation of dropout,~\citep{kingma2015variationaldropout,gal2016dropout}.
Recent variants of dropout such as \emph{noise-in}~\citep{dieng2018noisin} can also be interpreted in the same way.

The Jensen prior interpretation justifies the use of standard dropout in Bayesian neural networks:
the inferred posterior is the Jensen posterior which is also a Bayesian posterior under the Jensen prior.

\paragraph{Data Augmentation.}
Data augmentation is a simple and intuitive way to insert high-level prior knowledge into neural networks:
by targeted augmentation of the available training data we can encode invariances with respect to natural transformation or noise, leading to better generalization,~\cite{perez2017dataaugmentation}.

Data augmentation is also an instance of the above latent variable model, where $z_i$ now corresponds to randomly sampled parameters of an augmentation, for example, whether to flip an image along the vertical axis or not.

Interestingly, the above model suggests that to obtain better predictive performance at test time, the posterior predictive should be obtained by averaging the individual posterior predictive distributions over multiple latent variable realizations.
Indeed this is what early work on convolutional networks did,~\cite{he2015prelu,he2016resnet}, improving predictive performance significantly.

The Jensen prior interpretation again justifies the use of approximate Bayesian inference techniques targeting the Jensen posterior.
In particular, our theory suggests that the dataset size $n$ should \emph{not} be adjusted to account for augmentation.

\paragraph{Batch Normalization.}
As a practical technique batch normalization~\citep{ioffe2015batchnormalization} accelerates and stabilizes learning in deep neural networks.
The model of Figure~\ref{fig:dnn-latent} cannot directly serve to interpret batch normalization due to the dependence of batch normalization statistics on the batch.
We therefore need to extend the model to incorporate a random choice of batches yielding continuous random batch normalization statistics as proposed earlier~\citep{atanov2018stochasticbatchnormalization,shekhovtsov2018stochasticnormalizations}.

\begin{wrapfigure}{r}{0.4\columnwidth}
    \centering
\scalebox{0.9}{%
\begin{tikzpicture}[x=0.7cm,y=0.7cm]
\node[obs]                       (y)     {$y_i$} ; %
\node[latent, left=0.7cm of xa]   (theta) {$\pars$} ; %
\node[latent, above=of y]        (xa)    {$x'_i$} ; %
\node[latent, right=0.7cm of xa] (z)     {$z_i$} ; %
\node[obs, above=of xa]          (x)     {$x_i$} ; %

\edge {x,z} {xa} ; %
\edge {xa,theta} {y} ; %
\edge {x} {theta} ; %
\edge {theta} {xa} ; %

\plate {} {(x)(z)(xa)(y)} {$i=1,\dots,n$} ;
\end{tikzpicture}%
}%
\vspace{-0.25cm}%
\caption{Augmented model for batch normalization.}%
\label{fig:dnn-latent-batchnorm}%
\end{wrapfigure}
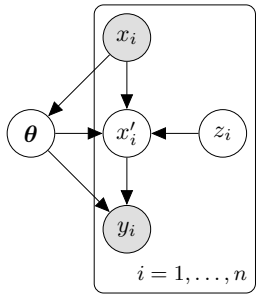

Formally such variation of batch normalization corresponds to the model shown in Figure~\ref{fig:dnn-latent-batchnorm}, where
$(x_i)_i \to \pars$ signifies the additional randomness in $p(\pars|X)$ due to random batches, and
$(\pars,x_i,z_i) \to x'_i$ are the resulting random outputs of the network, where $z_i$ is a per-instance randomness source~\citep{atanov2018stochasticbatchnormalization}.

With the above modifications all derivations in Section~\ref{sec:jensen} hold and batch normalization has a Jensen posterior.
In particular, the Jensen interpretation also suggests to perform batch normalization at test-time, averaging over multiple different batches composed of training set samples.

\subsection{Dirty Likelihood Experiment}\label{sec:frn-experiment}

The dirty likelihood hypothesis is plausible for the ResNet-20 experiments which use data augmentation and batch normalization, however, our CNN-LSTM model does have a clean likelihood function already.

To gain further confidence that this hypothesis cannot explain cold posterior we train a ResNet-20 without batch normalization or data augmentation.

\experiment{Clean Likelihood ResNet Experiment}:
we disable data augmentation and replace batch normalization with filter response normalization,~\citep{singh2019filterresponsenormalization}.
Without data augmentation and without batch normalization we now have a clean likelihood function and SG-MCMC targets a true underlying Bayes posterior.

Figure~\ref{fig:resnet-frn} on page~\pageref{fig:resnet-frn} shows the predictive test performance as a function of temperature.
We clearly see that for small temperatures $T \ll 1$ the removal of data augmentation and batch normalization leads to a higher standard error over the three runs, so that indeed data augmentation and batch normalization had a stabilizing effect on training and mitigated overfitting.
However, for test accuracy the best performance by the SG-MCMC ensemble model is still achieved for $T < 1$.
In particular, for test accuracy the best accuracy of $87.8\pm0.16\%$ is achieved at $T=0.0193$, comparing to a worse predictive accuracy of $87.1\pm0.13\%$ at temperature $T=1$.
For test cross entropy the performance achieved at $T=0.0193$ with $0.393\pm0.015$ is comparable to
$0.3918\pm0.0021$ achieved at $T=1$.

\begin{figure}[!t]
\center{\includegraphics[width=\columnwidth]{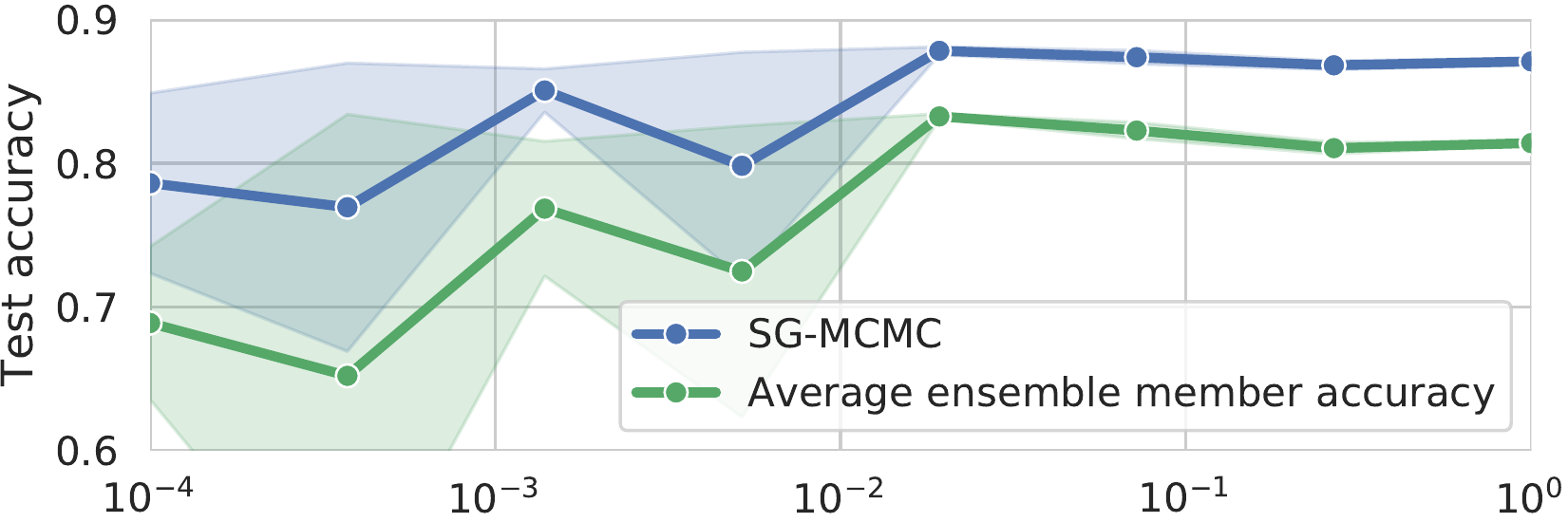}}%
\center{\includegraphics[width=\columnwidth]{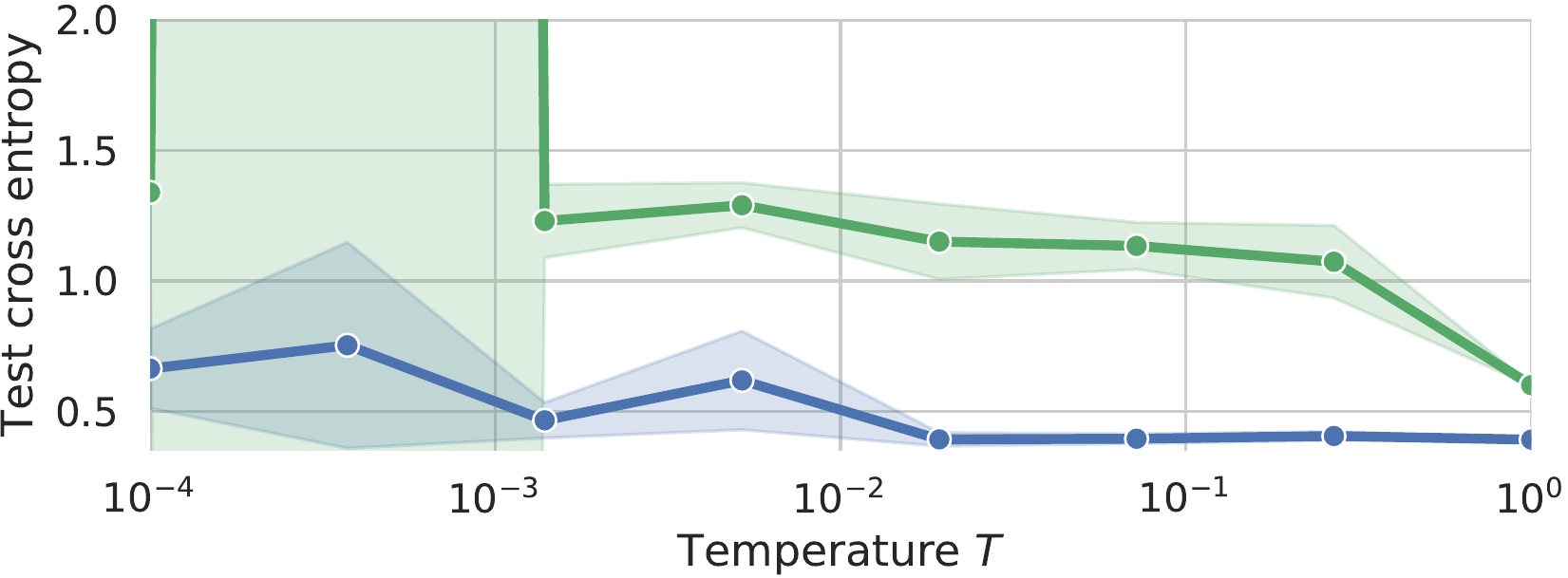}}%
\vspace{-0.3cm}%
\caption{ResNet-20 with filter response normalization (FRN) instead of batch normalization and without any use of data augmentation.
}%
\label{fig:resnet-frn}%
\end{figure}

The clean likelihood ResNet experiment is slightly inconclusive as there is now a less marked improvement when going to lower temperatures.
However, our CNN-LSTM IMDB model already had a clean likelihood function.
Therefore, while dirty likelihoods may play a role in shaping the posterior that SG-MCMC methods simulate from they likely do not account for the cold posterior effect.

\section{Prior Predictive Analysis for Different Prior Scales}\label{sec:prior-predictive-gpgaussian}

Our experiments in the main paper (Section~\ref{sec:prior}) clearly demonstrate that the prior $p(\pars)=\mathcal{N}(0,I)$ is bad in that it places prior mass on the same highly concentrated class probabilities for all training instances.

What other priors could we use?
The literature contains significant prior work on this question.
\citet{neal1995bayesianneuralnetworks} examined priors for shallow neural networks and identified scaling laws and correspondence to Gaussian process kernels.
Recently a number of works added to Neal's analysis by extending the results to
deep and wide neural networks~\citep{lee2018dnngp,matthews2018gaussiannetworks,yang2019wideneuralnetworksgp},
convolutional networks~\citep{garrigaalonso2019convnetgp}, and
Bayesian neural networks~\citep{novak2019bayesiandeepcnngp}.

A related line of work explores random functions defined by the initialization process of a deep neural network.
\citet{glorot2010init} and~\citet{he2015prelu} developed efficient random initialization schemes for deep neural networks and a more formal analysis of information flow in random functions defined by neural networks is given by~\citet{schoenholz2017deepinformationpropagation} and~\citet{hayou2018initializationactivation}.
All these works derive variance-scaling laws for independent Gaussian priors.
The precise scaling law depends on the network layer and the activation function being used.
For the same architecture and activation the scaling laws generally agree with those obtained from the Gaussian process perspective. Figure~\ref{fig:priorvariance} shows that the cold posterior effect is present regardless of the scaling of the variance of the Normal prior. In the following we investigate certain scaling laws of the prior more detailed.

\subsection{He-Scaled Normal Prior, $\mathcal{N}(0,I)$ for Biases}
To remain as close as possible to our existing setup we investigate a He-scaling prior,
equation~(14) in~\citep{he2015prelu}.

\begin{equation}
p(\pars_j) = \mathcal{N}\left(0, \frac{2}{b_j}\right),
\label{eqn:he-prior}
\end{equation}
where $b_j$ is the \emph{fan-in} of the $j$'th
layer.\footnote{For a \texttt{Dense} layer the fan-in is the number of input dimensions, for a \texttt{Conv2D} layer with a kernel of size $k$-by-$k$ and $d$ input channels the fan-in is $b_j = k^2 d$.}

The scaling law derived by~\citet{he2015prelu} does not cover the bias terms in a model.
This is due to the work considering only initialization---\citep{he2015prelu} initialized all biases to zero---whereas we would like to have proper priors for all model variables.
We therefore choose the original $\mathcal{N}(0,I)$ prior for all bias variables in our model.

\experiment{He-scaled Prior Predictive Experiment:}
For our ResNet-20 setup on CIFAR-10 we use our He-scaled-Normal prior to once again carry out the prior predictive experiment that was originally done in
Section~\ref{sec:prior}, Figures~\ref{fig:resnet-priorsamples} and~\ref{fig:resnet-priorpredictive} of the main paper.
Figure~\ref{fig:resnet-priorstudy-kaiming} show the prior predictive results for the new prior.
The basic conclusion remains unchanged: despite scaling the convolution weights and dense layer weights by the He-scaling law in the prior the prior predictive distributions remain highly concentrated around the same distribution for all training instances.

Why do functions under this prior remain concentrated?
Perhaps it is due to the loose $\mathcal{N}(0,I)$ prior for the bias terms such that any concentration in early layers is amplified in later layers?  We investigate this further in Section~\ref{sec:he-smallbias}.

\begin{figure*}[t]%
\begin{subfigure}[t]{0.65\textwidth}%
\includegraphics[width=0.49\textwidth]{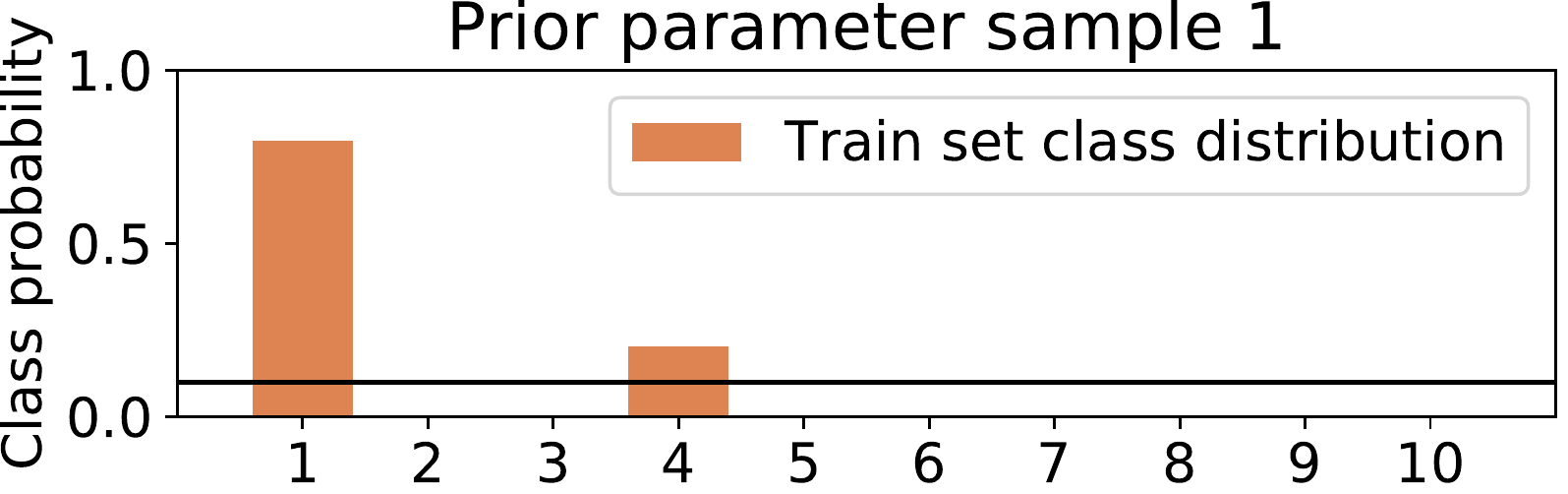}%
\hfill%
\includegraphics[width=0.49\textwidth]{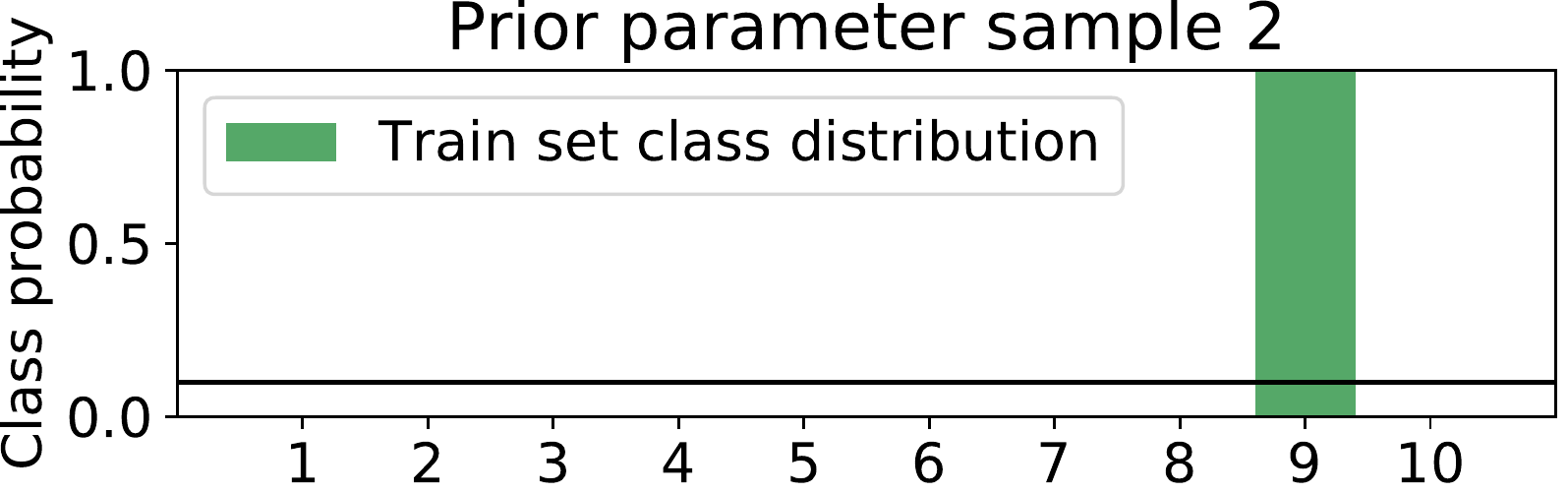}%
\caption{Typical predictive distributions for 10 classes under the prior, averaged over the entire training set, $\mathbb{E}_{x \sim p(x)}[p(y|x,\pars^{(i)})]$.  Each plot is for one sample $\pars^{(i)} \sim p(\pars)$.
Given a sample $\pars^{(i)}$ the average training data class distribution is still highly concentrated around the same classes for all $x$.}%
\label{fig:resnet-priorsamples-kaiming}%
\end{subfigure}%
\hfill%
\begin{subfigure}[t]{0.32\textwidth}%
\includegraphics[width=\textwidth]{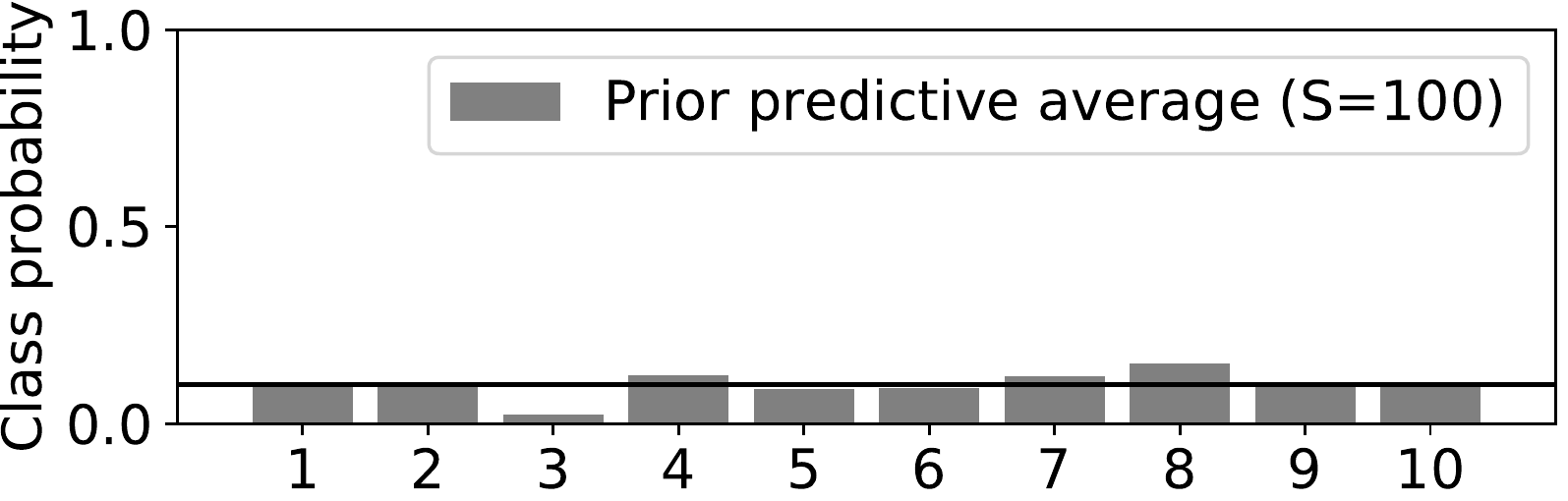}%
\caption{Prior predictive $\mathbb{E}_{x \sim p(x)}[\mathbb{E}_{\pars \sim p(\pars)}[p(y|x,\pars)]]$ over 10 classes for a Kaiming-scaling prior, estimated using $S=100$ samples $\pars^{(i)}$ and all training images.}%
\label{fig:resnet-priorpredictive-kaiming}%
\end{subfigure}%
\caption{ResNet-20/CIFAR-10 prior predictive study for a He-scaled Normal prior for \texttt{Conv2D} and \texttt{Dense} layers and a $\mathcal{N}(0,I)$ prior for all bias terms.
This prior concentrates prior mass on functions which output the \emph{same} concentrated label distribution for \emph{all} training instances.  It is therefore a bad prior.}%
\label{fig:resnet-priorstudy-kaiming}%
\end{figure*}

\experiment{He-scaled Prior ResNet-20 CIFAR-10 Experiment:}
We also perform the original cold posterior experiment from the main paper with the He-scaling Normal prior.
We show the temperature-dependence curves for test accuracy and test cross-entropy in~Figure~\ref{fig:resnet-gpgaussian}.
The overall performance drops compared to the $\mathcal{N}(0,I)$ prior, but the cold posterior effect clearly remains.
With this result and the result from the prior predictive study we can conclude that a simple Normal scaling correction is not enough to yield a sensible prior.

\begin{figure}[!t]
\center{\includegraphics[width=\columnwidth]{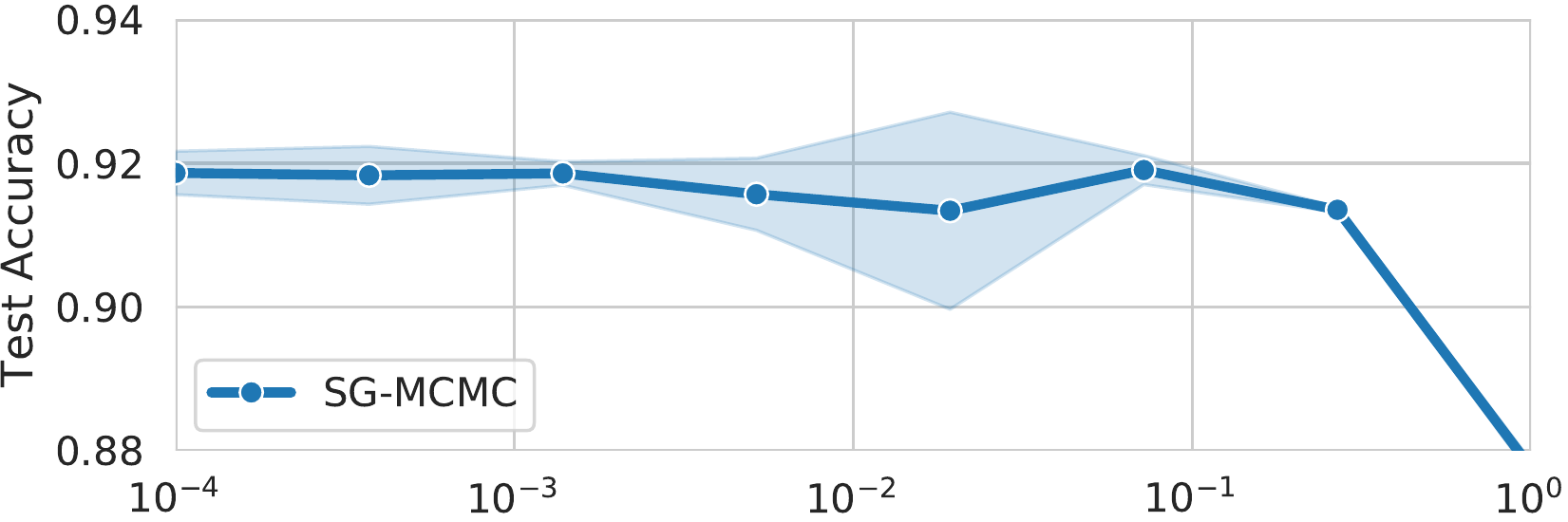}}%
\center{\includegraphics[width=\columnwidth]{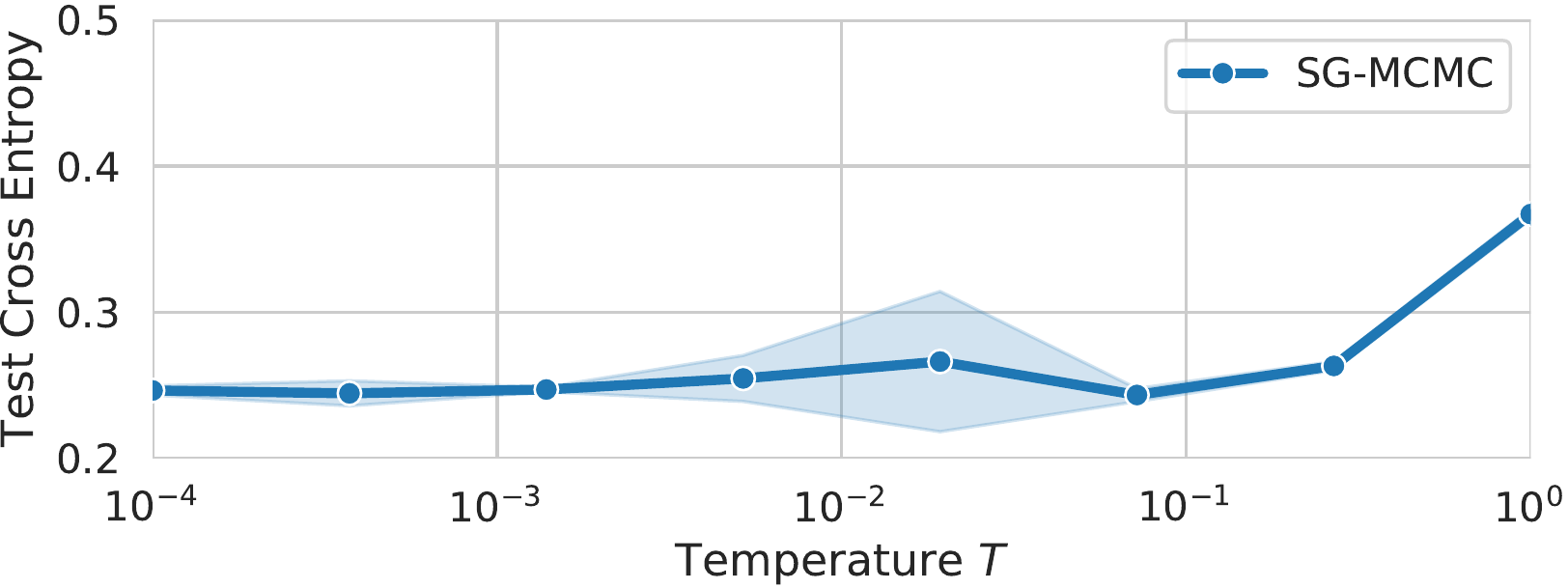}}%
\vspace{-0.3cm}%
\caption{ResNet-20 on CIFAR-10 with He-scaling Normal prior (He-scaled Normal for \texttt{Conv2D} and \texttt{Dense} layers, and $\mathcal{N}(0,I)$ for all bias terms).
The cold posterior effect remains: the poor predictive performance of the Bayes posterior at $T=1$ holds for both accuracy and cross-entropy.}%
\label{fig:resnet-gpgaussian}%
\end{figure}

\subsection{He-Scaled Normal Prior, $\mathcal{N}(0,\epsilon I)$ for Biases}\label{sec:he-smallbias}
In this section we experiment with He-scaling and a very small scale for the bias prior.
There are two motivations for such experimentation:
\emph{first}, He-scaling was originally proposed by~\citet{he2015prelu} for initializing
deep convolutional neural networks and in their initialization all bias terms were initialized to zero.
\emph{Second}, bias terms influence a large number of downstream activations and getting the scale wrong for our bias priors may have the large concentration effect that we observe in the previous prior predictive experiments.

We therefore propose to use a He-scaling Normal prior for all \texttt{Conv2D} and \texttt{Dense} layer weights and to use a $\mathcal{N}(0,\epsilon I)$ prior for all bias terms.
Here we use $\epsilon = \sigma^2$ with $\sigma = 10^{-6}$, essentially sampling all bias terms close to zero as in the original initialization due to~\citep{he2015prelu}.

\experiment{He-scaled Prior, $\mathcal{N}(0,\epsilon I)$ Bias Prior Experiment:}
We draw ResNet-20 models from the prior and evaluate the predicted class distributions on the entire CIFAR-10 training set.
Figure~\ref{fig:resnet-priorsamples-kaiming-smallbias} shows two prior draws and the resulting class distributions marginalized over the entire training set.
Figure~\ref{fig:resnet-priorpredictive-kaiming-smallbias} shows a marginal prior predictive, marginalized over $S=100$ prior draws and the entire training distribution of 50,000 images.
The resulting marginal prior predictive approaches the uniform distribution.
However, the He-scaled prior with $\mathcal{N}(0,\epsilon I)$ for bias terms remains a bad prior: random draws place prior mass on the same concentrated class distribution for all training instances.

\begin{figure*}[t]%
\begin{subfigure}[t]{0.65\textwidth}%
\includegraphics[width=0.49\textwidth]{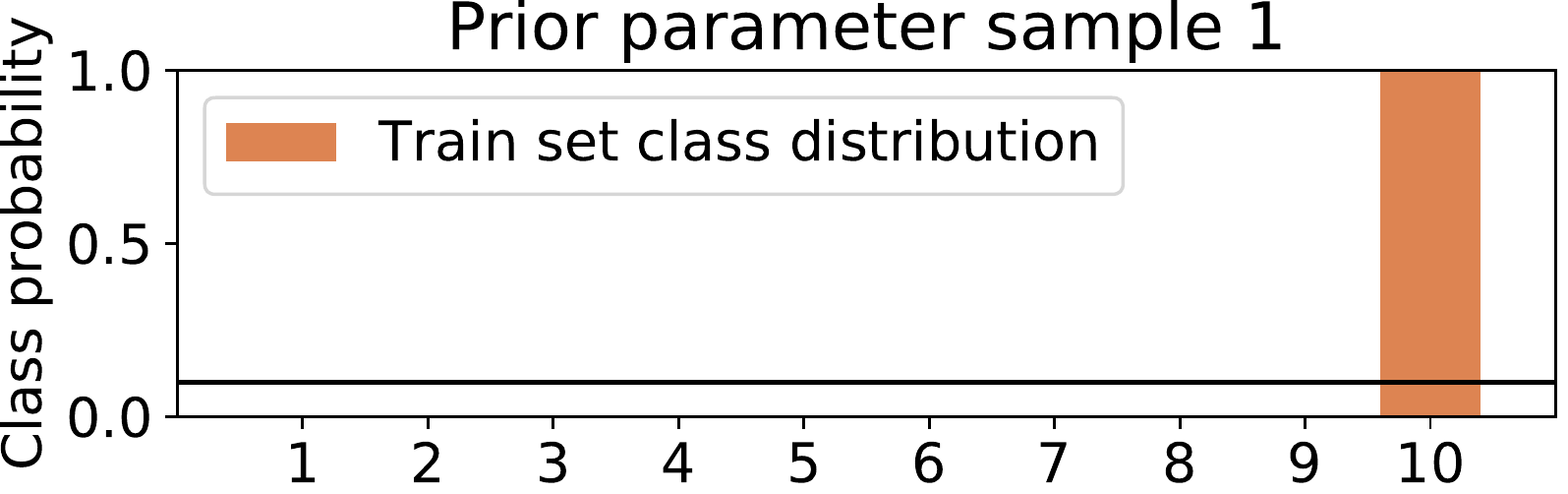}%
\hfill%
\includegraphics[width=0.49\textwidth]{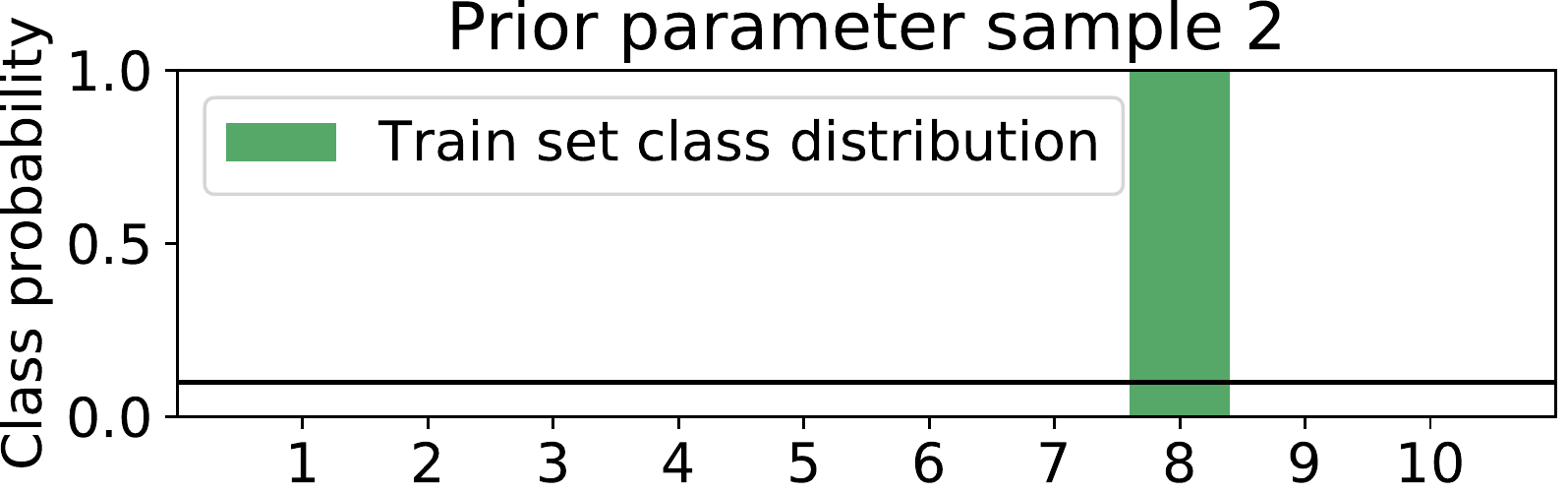}%
\caption{Typical predictive distributions for 10 classes under the prior, averaged over the entire training set, $\mathbb{E}_{x \sim p(x)}[p(y|x,\pars^{(i)})]$.  Each plot is for one sample $\pars^{(i)} \sim p(\pars)$.
Given a sample $\pars^{(i)}$ the average training data class distribution is still highly concentrated around the same classes for all $x$ despite using a small $\mathcal{N}(0,\epsilon I)$ prior for biases.}%
\label{fig:resnet-priorsamples-kaiming-smallbias}%
\end{subfigure}%
\hfill%
\begin{subfigure}[t]{0.32\textwidth}%
\includegraphics[width=\textwidth]{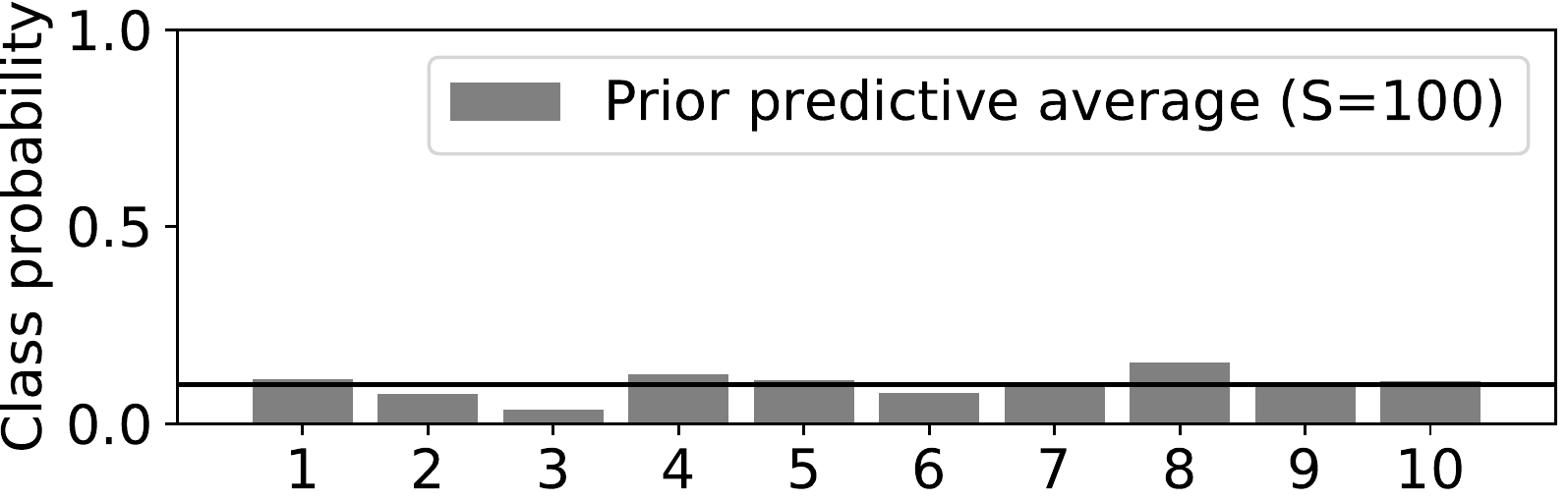}%
\caption{Prior predictive $\mathbb{E}_{x \sim p(x)}[\mathbb{E}_{\pars \sim p(\pars)}[p(y|x,\pars)]]$ over 10 classes for a Kaiming-scaling prior with a $\mathcal{N}(0,\epsilon I)$ prior for bias terms.
We estimate the marginal distribution using $S=100$ samples $\pars^{(i)}$ and all training images.}%
\label{fig:resnet-priorpredictive-kaiming-smallbias}%
\end{subfigure}%
\caption{ResNet-20/CIFAR-10 prior predictive study for a He-scaled Normal prior for \texttt{Conv2D} and \texttt{Dense} layers and a $\mathcal{N}(0,\epsilon I)$ prior for all bias terms.
This prior still concentrates prior mass on functions which output the \emph{same} concentrated label distribution for \emph{all} training instances.  It is therefore a bad prior.}%
\label{fig:resnet-priorstudy-kaiming-smallbias}%
\end{figure*}

\section{Tempering the Observation Model?}\label{sec:templikelhood}
In~\citep{wilson2020bayesianperspective}, Equation~(4) a proposal is made to use a different likelihood function of the form
\begin{equation}
    p_T(y|x,\pars) :\propto p(y|x,\pars)^{1/T}.\label{eqn:observation-model-transformation}
\end{equation}
It is claimed that with this adjusted observation model the cold posterior is simply the ordinary Bayes posterior of the modified model.
Indeed, if we are to plug the right hand side of~(\ref{eqn:observation-model-transformation}) directly into our posterior energy function~(\ref{eqn:U}) we obtain the cold posterior energy function,
\begin{equation}
U_T(\pars) := - \sum_{i=1}^n \frac{1}{T} \log p(y_i|x_i,\pars) - \log p(\pars).
\label{eqn:UT}
\end{equation}
The mistake in this derivation is to ignore that renormalization of $p_T(y|x,\pars)$ must be carried out because the normalizing constant is not invariant of $\pars$.  In particular, this is in contrast to typical applications of Bayes rule for posteriors, where we can indeed write $p(\pars|\mathcal{D}) \propto p(\mathcal{D}|\pars) p(\pars)$ without worries, as here the normalizing constant does not depend on $\pars$.
One consequence of this mistake is that $U_T(\pars)$ is \emph{not} necessarily the energy function of a Bayes posterior.

Instead, for the tempered observation model proposed by~\citep{wilson2020bayesianperspective} the correctly normalized observation likelihood is
\begin{equation}
p_T(y|x,\pars) = \frac{p(y|x,\pars)^{1/T}}{\int p(y|x,\pars)^{1/T} \,\textrm{d}y}.
\end{equation}
Using this normalized observation model, the correct Bayes posterior energy corresponding to $p_T(y|x,\pars)$ is
\begin{eqnarray}
{\tilde U}_T(\pars) & := & - \sum_{i=1}^n \frac{1}{T} \log p(y_i|x_i,\pars) - \log p(\pars)\nonumber\\
    & & + \sum_{i=1}^n \log \int p(y|x_i,\pars)^{1/T} \,\textrm{d}y.\label{eqn:UT-correction-term}
\label{eqn:UT-corrected}
\end{eqnarray}
Therefore, when the observation model is transformed as in~(\ref{eqn:observation-model-transformation}) and as suggested by~\citep{wilson2020bayesianperspective}, then in order to obtain a normalized observation model we must include the correction term~(\ref{eqn:UT-correction-term}) and this produces a modified energy function, ${\tilde U}_T(\pars)$, that differs from the actual cold posterior energy function $U_T(\pars)$.

\paragraph{Is there a way to ``fix'' this mistake?}
I.e. can one construct an observation model such that the resulting Bayes posterior corresponds to a tempered version of the Bayes posterior of the original observation model?
For classification we found a way: we assign probability to a pseudo event ``$\emptyset$'' which cannot occur.
To see this, assume a classification model $p(y|x)$ where $y \in \{1,2,\dots,K\}$.
Clearly $\sum_{k=1}^K p(y=k|x) = 1$.
Given a temperature $T \leq 1$ we define
\begin{equation}
    \tilde{p}(y=k|x) := p(y=k|x)^{1/T}.
\end{equation}
Clearly for $0 < T \leq 1$ we have that $f(x) = x^{1/T}$ is a monotonic function in $x \in [0,1]$ and $f(x) = x^{1/T} \leq x$, and therefore $\sum_k \tilde{p}(y=k|x) \leq 1$.
We absorb the remaining probability mass into a pseudo event ``$\emptyset$'',
\begin{equation}
    \tilde{p}(y=\emptyset|x) := 1 - \sum_{k=1}^K \tilde{p}(y=k|x),
\end{equation}
such that the resulting distribution $\tilde{p}$ over $K+1$ basic events sums to one,
\begin{equation}
\sum_{k \in \{1,2,\dots,K,\emptyset\}} \tilde{p}(y=k|x) = 1.
\end{equation}

Now observe that for any event in $\{1,2,\dots,K\}$ that actually \emph{can} occur we have
\begin{equation}
\log \tilde{p}(y=k|x) = \frac{1}{T} \log p(y=k|x),
\end{equation}
that is, we have achieved the effect of temperature scaling when using $\tilde{p}$ as observation model.
While formally possible, can we make sense of this transformation and introduction of a pseudo event?

To us it seems entirely non-Bayesian to artificially introduce events into a model while knowing with perfect certainty that these events cannot happen and then allow the model to assign probability mass to those events.
It is non-Bayesian because our knowledge with respect to the new event is perfect: it cannot occur.  Therefore a model should respect this knowledge of the world.

\section{Details: Generation of a Synthetic Dataset Based on an MLP Drawn From its Prior Distribution}\label{sec:generate_data_from_mlp}

In this section, we describe how we generate a synthetic dataset based on a multi-layer perceptron (MLP) drawn from its prior distribution, as used in Section~\ref{sec:hyp_biased_sgmcmc} of the main paper.

We generate synthetic data by (i) drawing a MLP from its prior distribution, i.e., $\texttt{mlp}_\pars$ with $\pars \sim p(\pars)$, (ii) sampling input data point $x$'s~$\in \mathbb{R}^5$ from a standard normal distribution and (iii) sampling label $y$'s~$\in \{1,2,3\}$ from the resulting logits $\texttt{mlp}_\pars(x)$.
We take $\texttt{mlp}_\pars$ to be of depth 2, with 10 units and relu activation functions. We generate $n=100$ points for inference and 10,000 for evaluations.

The choice of $p(\pars)$ requires some care. On the one hand, a naive choice of normal priors with unit standard deviation leads to a degenerated dataset that concentrates all its outputs on a single class. On the other hand, normal priors with a smaller standard deviation\footnote{Default value of \texttt{tf.random\_normal\_initializer}.}, e.g., 0.05, lead to a less spiky label distribution but with little dependence on the input $x$'s.

As a result, we considered a He normal prior~\citep{he2015prelu} for the weights of $\texttt{mlp}_\pars$ and a normal prior, with standard deviation 0.05, for the bias terms. We similarly adapted the choice of the priors for the MLPs used to learn over the data generated in this way.

\section{Details about Hamiltonian Monte Carlo}\label{sec:practical_usage_hmc}

In this section, we describe practical considerations about Hamiltonian Monte Carlo (HMC) and present further results about its comparison with SG-MCMC (see Section~\ref{sec:hyp_biased_sgmcmc}).

HMC mainly exposes four hyperparameters that need to be set~\citep{neal2011mcmc}:
\begin{compactitem}
\item The number $L$ of steps of the leapfrog integrator,
\item The step size $\varepsilon$ in the leapfrog integrator,
\item The number $b$ of steps of the burn-in phase,
\item The number $S$ of samples to generate.
\end{compactitem}

\subsection{Hyperparameter choices}
In our experiments with HMC, we have set $S=2500$, generating a total of $25000$ samples after the burn-in phase and keeping one sample every ten samples.

For the burn-in phase,  we investigated in preliminary experiments the effect of varying the number of steps $b \in \{500, 1000, 5000\}$, noticing that our diagnostics (as later described) started to stabilize for $b=1000$, so that we decided to use $b=5000$ out of precaution (even though it may not be the most efficient option).

We thereafter searched a good combination of leapfrog steps and step size for $L \in \{5, 10, 100\}$ and $\varepsilon \in \{0.001, 0.01, 0.1\}$. The results of the nine possible combinations are reported in Figure~\ref{fig:sgmcmc_vc_hmc_all_leapfrog_steps_stepsize_pairs}, after aggregating 5 different runs (i.e., from 5 different random initial conditions). The influence of the step size in our experiments was likely reduced by the fact that we used the dual averaging step-size adaptation scheme from~\citet{hoffman2014no}, as implemented in \textit{Tensorflow Probability}~\citep{dillon2017tensorflow}.\footnote{\texttt{tfp.mcmc.DualAveragingStepSizeAdaptation}.}

\subsection{Convergence monitoring}

We monitor convergence by first inspecting trace plots and second by computing standard diagnostics, namely the effective sample size (ESS)~\citep{brooks2011handbook} and the potential scale reduction (PSRF)~\citep{gelman1992psrf}.

\paragraph{Trace plots.}
In Figures~\ref{fig:diagnostics_hmc_1}-\ref{fig:diagnostics_hmc_2}-\ref{fig:diagnostics_hmc_3}, we detail the inspection of the 5 different chains for the choice $L=100$ and $\varepsilon=0.1$ (which corresponds to the results of the sampler shown in the main paper). As practical diagnostic tools, we consider \textit{trace plots} where we monitor the evolution of some statistics with respect to the generated HMC samples (e.g., see~Section 24.4 in~\citet{murphy2012machine}, and references therein, for an introduction in a machine learning context). We compute trace plots for different depths of the MLP (in $\{1, 2, 3\}$) and different\footnote{We limit ourselves to four temperatures to avoid clutter.} temperatures, $T \in \{0.001, 0.0024, 0.014, 1.0\}$.

In addition to monitoring the evolution of the cross entropy for $S' \in \{1, 2, \dots, S\}$ HMC samples (see Figure~\ref{fig:diagnostics_hmc_1}), we also consider the following statistics:
\begin{itemize}
    \item \textbf{Mean of the predictive entropy:} Let us denote by $\mathcal{D}_\text{held-out}$ the held-out set of pairs $(x, y)$ and $\mathcal{E}_\pars(x)$ the entropy of the softmax output at the input $x$
    $$
    \mathcal{E}_\pars(x) = -\sum_{c} p(y=c|x,\pars) \log{p(y=c|x,\pars)},
    $$
    together with its average over the held-out set
    $$
    \mathcal{E}_\pars= \frac{1}{|\mathcal{D}_\text{held-out}|} \sum_{(x,y) \in \mathcal{D}_\text{held-out}} \mathcal{E}_\pars(x).
    $$
    For $S' \in \{1, 2, \dots, S\}$ samples collected along the trajectory of HMC, we report in Figure~\ref{fig:diagnostics_hmc_2} the estimate
    $$
    \hat{\mathcal{E}} = \frac{1}{S'} \sum_{s=1}^{S'} \mathcal{E}_{\pars_s}
    \approx
    \bar{\mathcal{E}} = \int \mathcal{E}_\pars \cdot p(\pars|\mathcal{D}) d\pars,
    $$
    which we refer to as the mean of the predictive entropy.
    \item \textbf{Standard deviation of the predictive entropy:} We also consider the monitoring of the second moment of the predictive entropy. With the above notation, we estimate
    $$
    \frac{1}{S'-1} \sum_{s=1}^{S'} (\mathcal{E}_{\pars_s} - \hat{\mathcal{E}})^2 
    \approx
    \int (\mathcal{E}_\pars - \bar{\mathcal{E}})^2 \cdot p(\pars|\mathcal{D}) d\pars
    $$
    and report its square root in Figure~\ref{fig:diagnostics_hmc_3}, which we refer to as the standard deviation of the predictive entropy.
\end{itemize}
As a general observation, we can see on Figures~\ref{fig:diagnostics_hmc_1}-\ref{fig:diagnostics_hmc_2}-\ref{fig:diagnostics_hmc_3} that, overall, the 5 different chains tend to exhibit a converging behavior for the three examined statistics, with typically more dispersion as the depth and the temperature increase (which is reflected by the ranges of the y-axis in the plots of Figures~\ref{fig:diagnostics_hmc_1}-\ref{fig:diagnostics_hmc_2}-\ref{fig:diagnostics_hmc_3} that get wider as $T$ and the depth become larger). 

\paragraph{ESS and PSRF.}
The effective sample size (ESS)~\citep{brooks2011handbook} measures how independent the samples are in terms of the auto-correlations within the sequence at different lags.
The potential scale reduction factor (PSRF)~\cite{gelman1992psrf} assesses the convergence of the chains (to the same target distribution) by testing for equality of means. 

We computed ESS and PSRF for our HMC simulation (with 100 leapfrog steps and a step size of 0.1, as reported in the main paper). We used the TFP implementations \texttt{tfp.mcmc.$\{$effective\_sample\_size, potential\_scale\_reduction$\}$}. Figure~\ref{fig:ess_and_psrf_and_kl} (left, middle) displays the ESS and PSRF with respect to the different temperature levels, for the 3 MLP depths. Both ESS and PSRF were averaged over the model parameters.

We observe that in the regime $T$ in  $[0.05, 1]$, the diagnostics indicate an approximate convergence (PSRF < 1.05 and ESS in $[1800, S]$, with $S=2500$ total samples) for the 3 MLP depths. On the other hand, in the regime $T$ in [0.001, 0.05], the diagnostics only continue to indicate an approximate convergence for the depth 1. For depths 2 and 3, both diagnostics substantially degrade, e.g., ESS down to $\approx189$ for depth 3.

\subsection{KL divergence between predictive distributions}

In Section~\ref{sec:hyp_biased_sgmcmc}, we compare side by side the cross-entropy of SG-MCMC and HMC for the different temperature levels, exhibiting a close agreement.

As an alternative visualization of this comparison, we computed the (symmetrized) KL divergence between the SG-MCMC and HMC predictive distributions (i.e., in our setting, categorical distributions with 3 classes). 

For SG-MCMC and HMC (instantiated with 100 leapfrog steps and a step size of 0.1, as reported in the main paper), 
Figure~\ref{fig:ess_and_psrf_and_kl} (right) displays the (symmetrized) KL  with respect to the different temperature levels, for the 3 MLP depths (averaged over the seeds). We observe that all KLs are small (in the order of $\approx10^{-5}$ for depth 1, and $\approx 10^{-3}$ for depths 2 and 3).

\newpage

\begin{figure*}[!t]
\center{\includegraphics[keepaspectratio=true,width=0.55\textwidth]{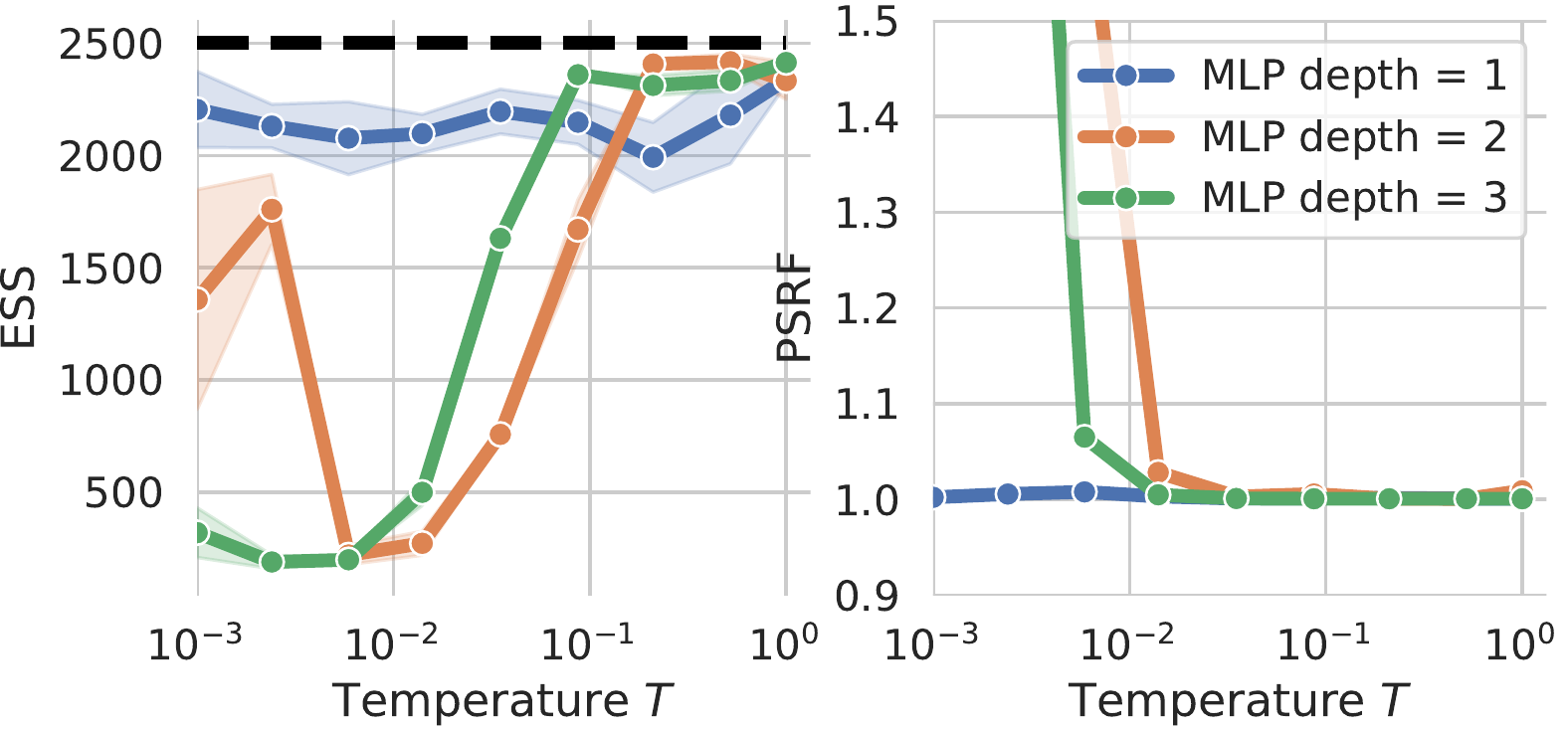}
\includegraphics[keepaspectratio=true,width=0.35\textwidth]{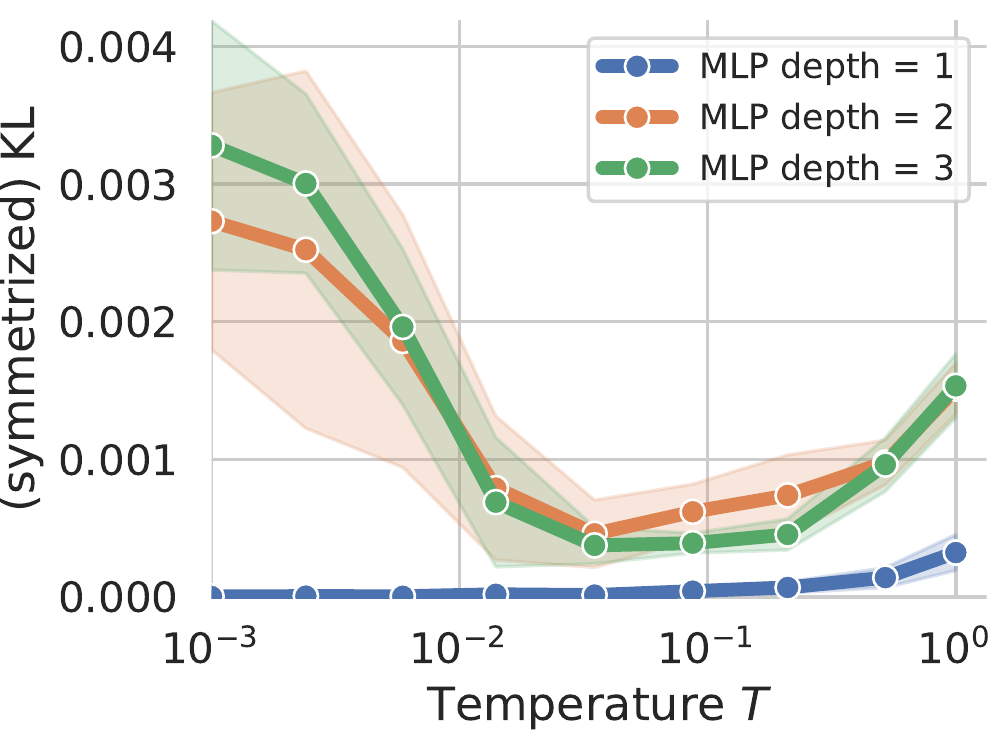}
}
\caption{For HMC (instantiated with 100 leapfrog steps and a step size of 0.1, as reported in the main paper Section~\ref{sec:hyp_biased_sgmcmc}), we report the effective sample size (left) and potential scale reduction factor (middle) with respect to the different temperature levels. On the left plot, the black dash line corresponds to the $S=2500$ samples and ESS=$S$ indicates no correlation in the sequences. For PSRF, approximate convergence is generally considered when PSRF $< 1.2$~\citep{gelman1992psrf}. (right) KL divergence between the predictive distributions of HMC and SG-MCMC.}%
\label{fig:ess_and_psrf_and_kl}%
\end{figure*}

\begin{figure*}[!t]
\includegraphics[width=0.33\textwidth]{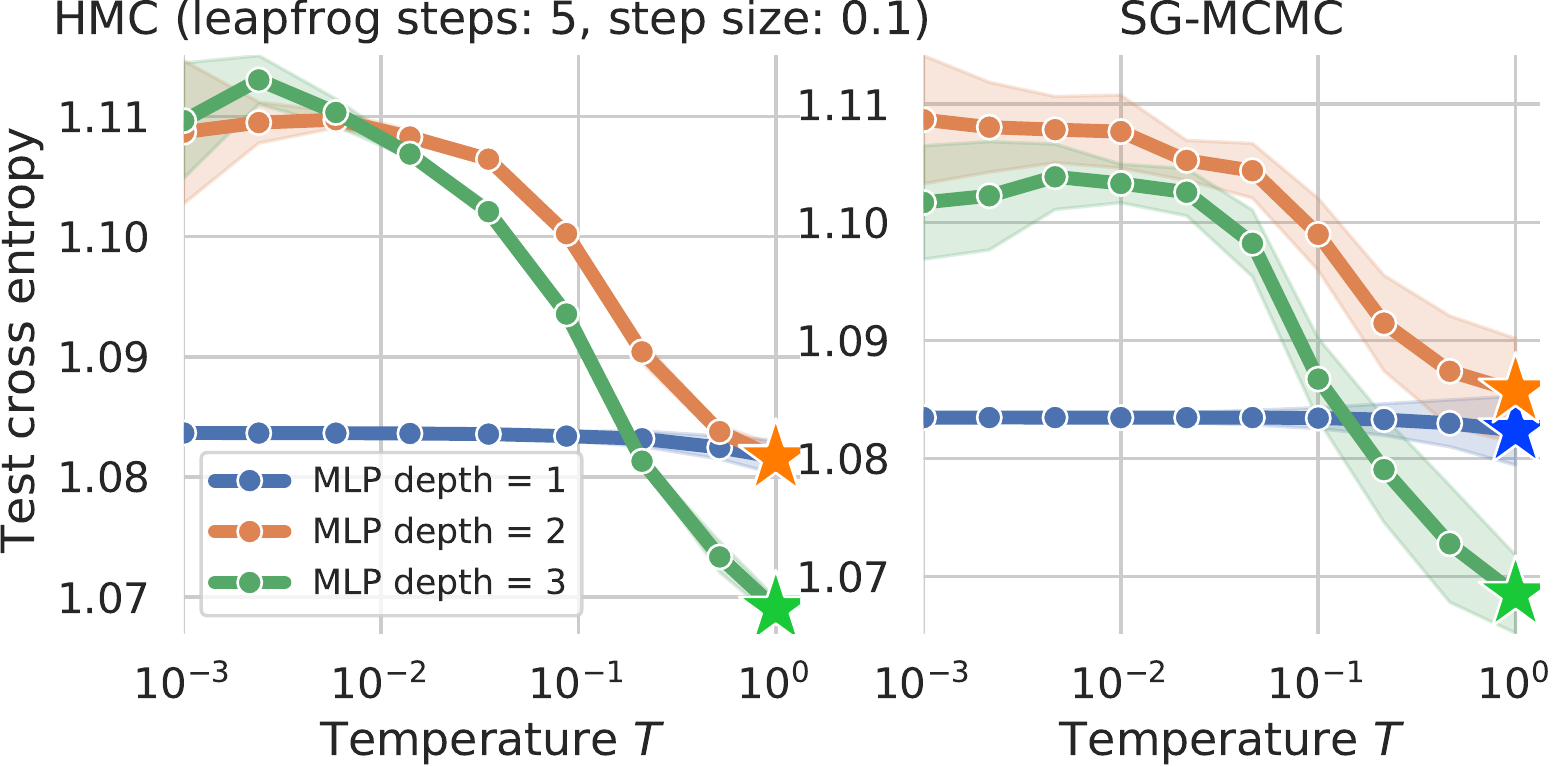}
\includegraphics[width=0.33\textwidth]{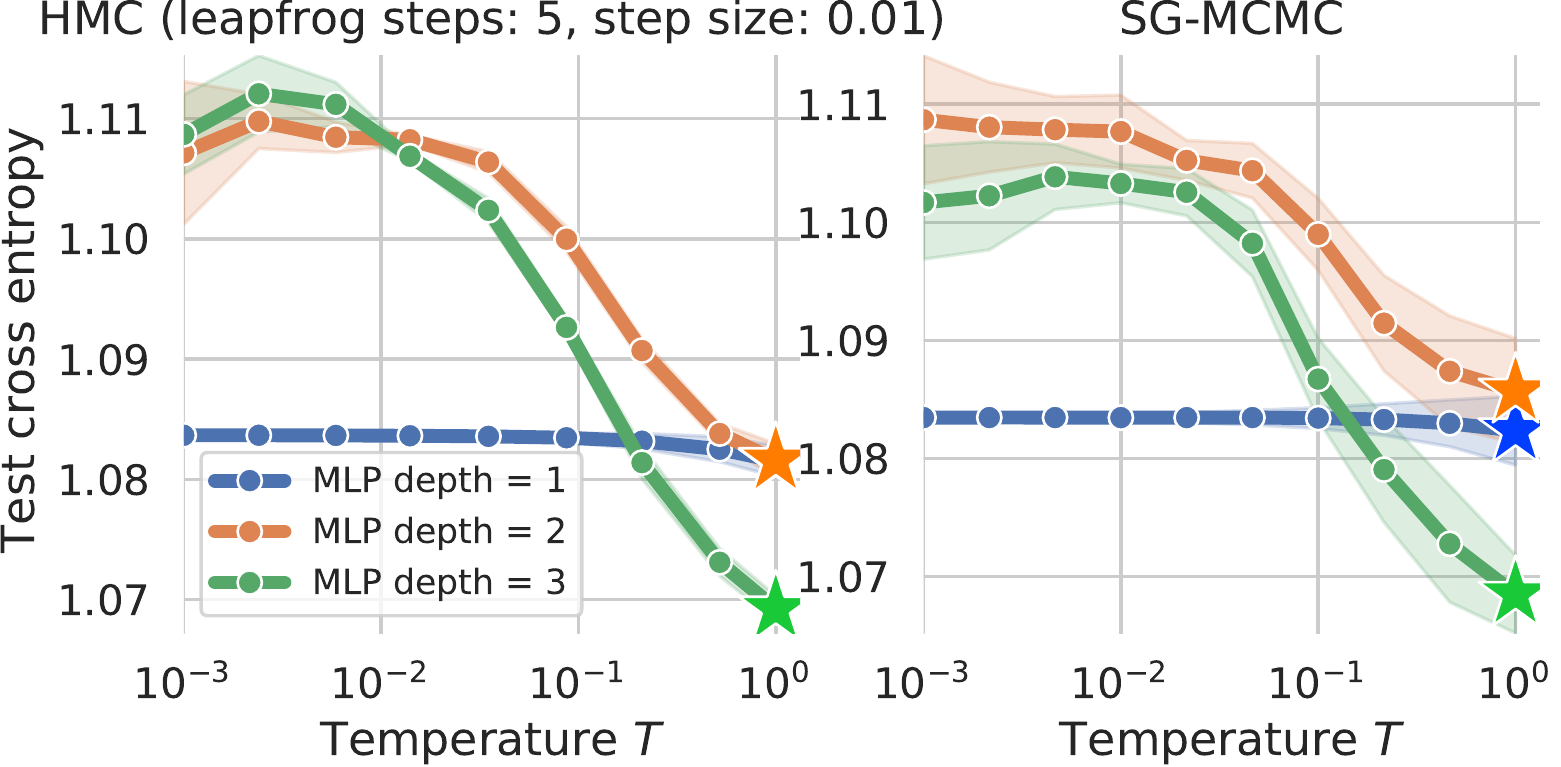}
\includegraphics[width=0.33\textwidth]{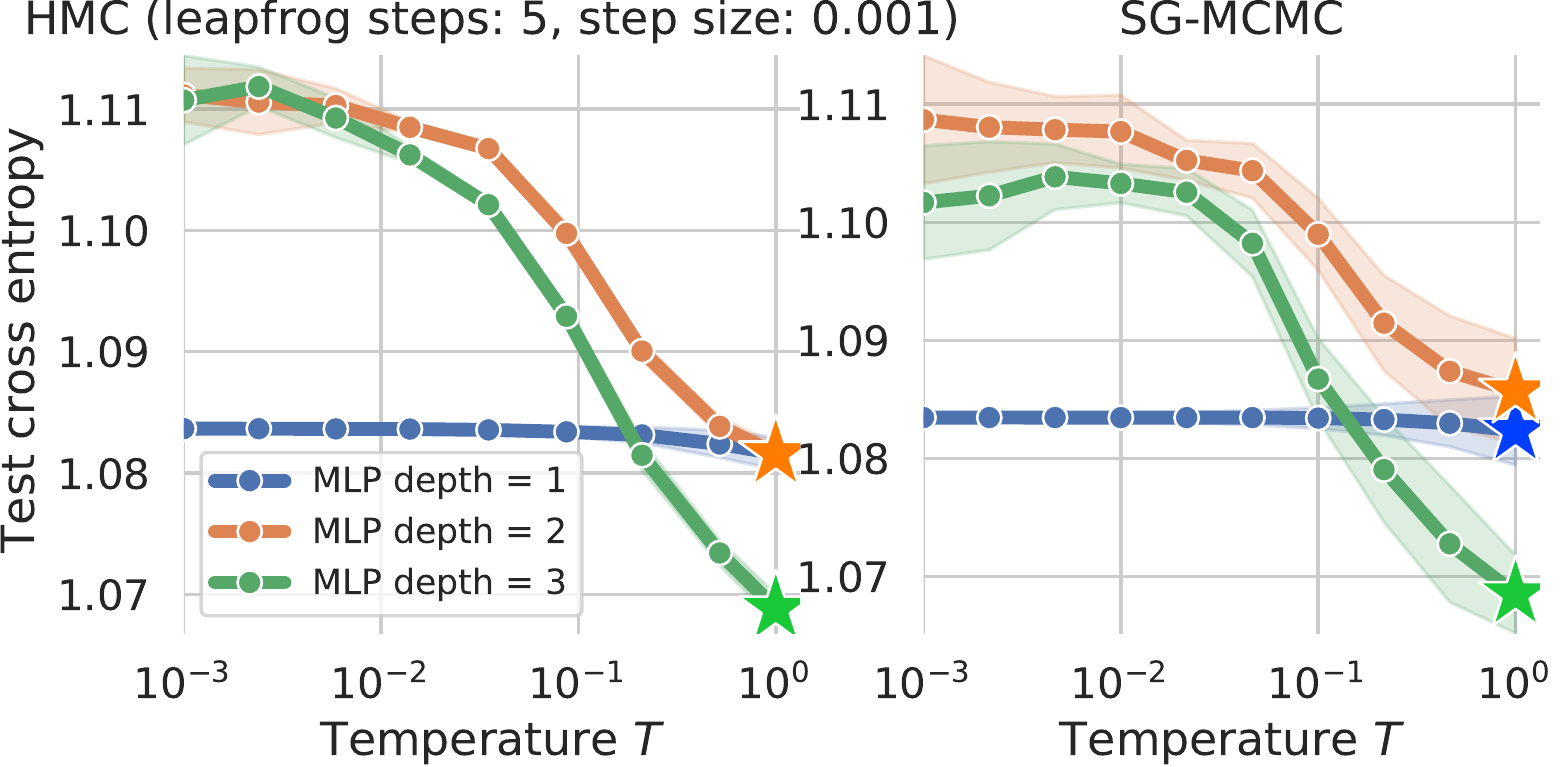}\\~\\~
\includegraphics[width=0.33\textwidth]{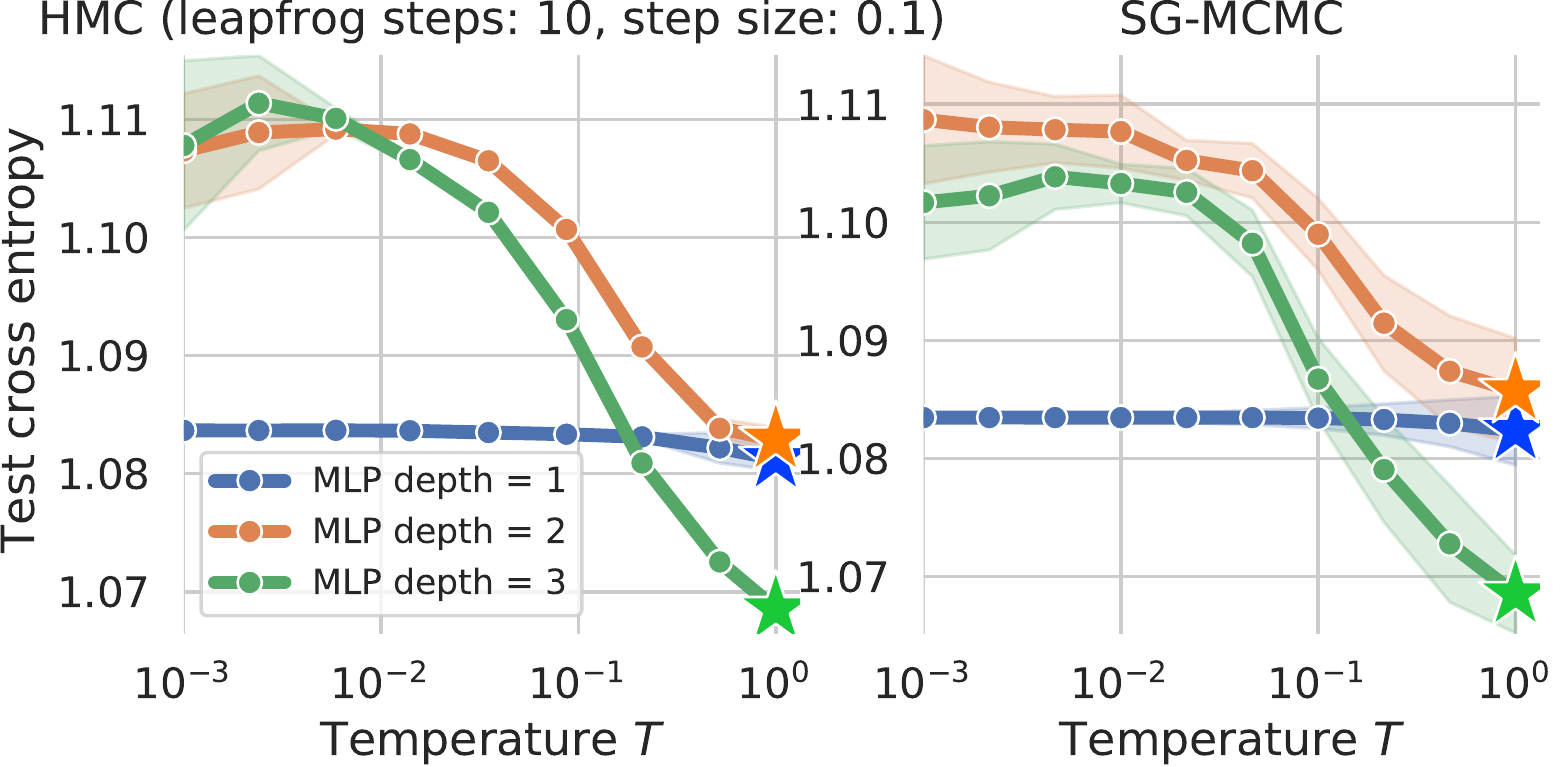}
\includegraphics[width=0.33\textwidth]{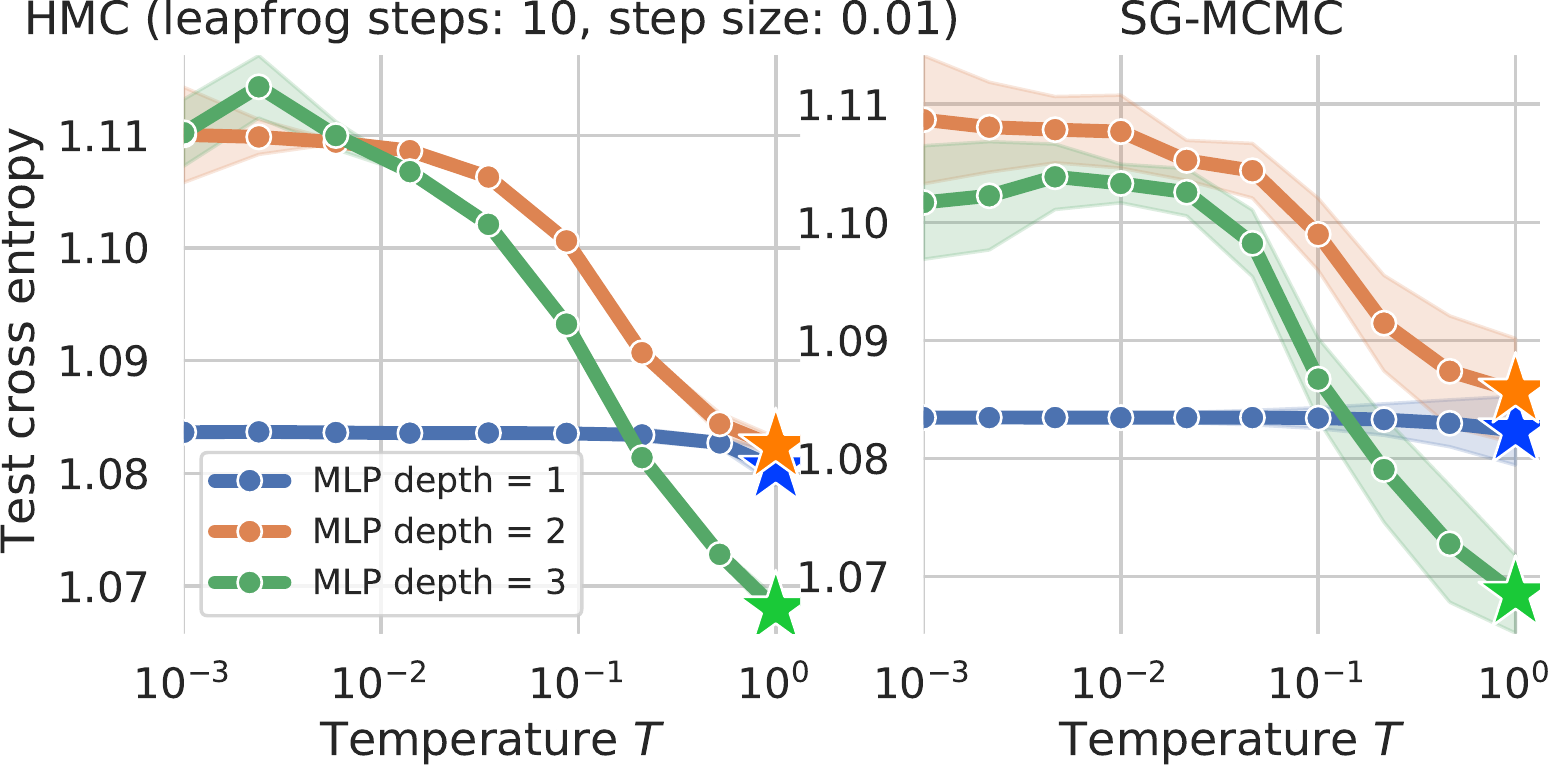}
\includegraphics[width=0.33\textwidth]{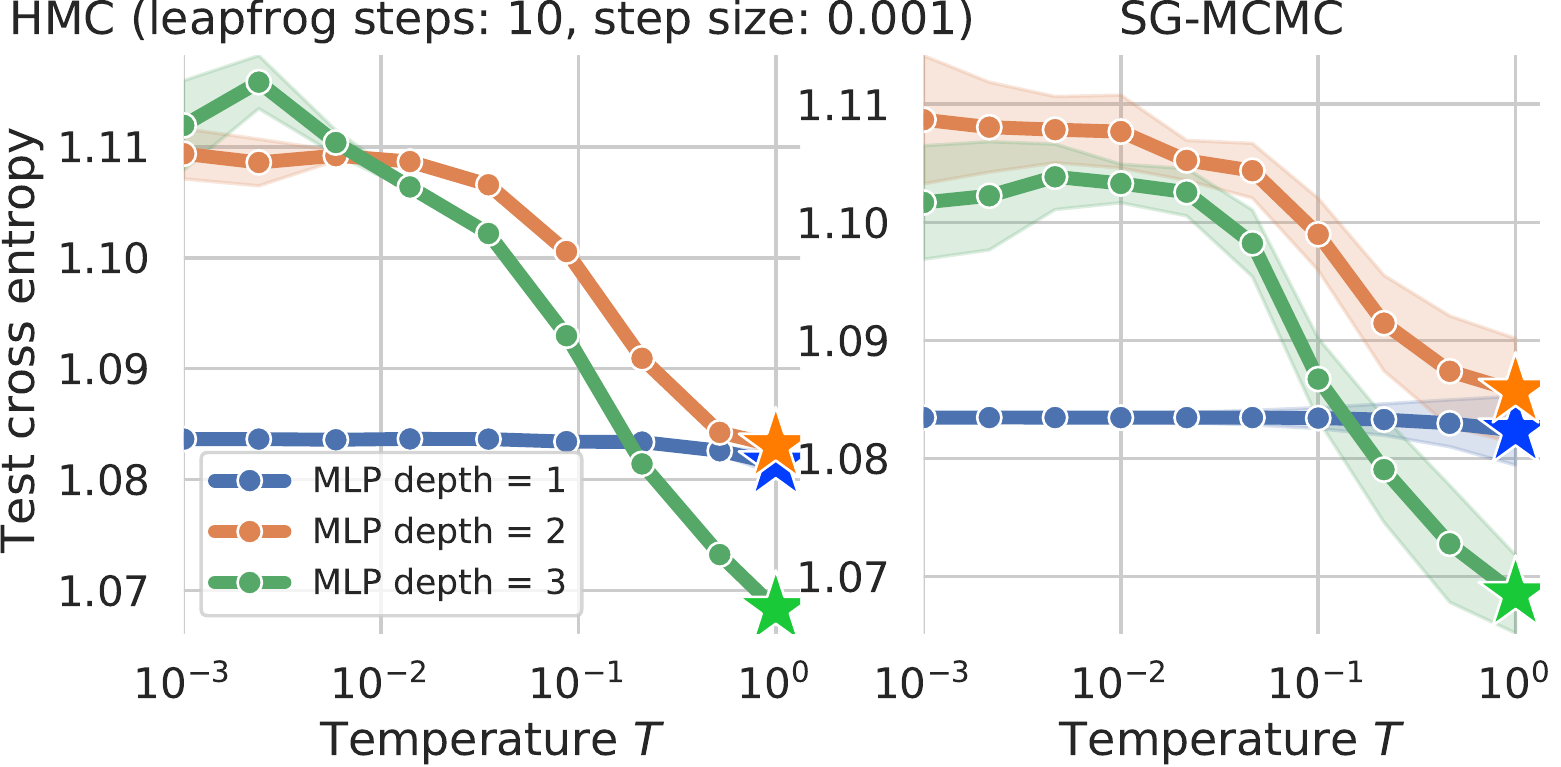}\\~\\~
\includegraphics[width=0.33\textwidth]{figures/sgmcmc_vs_hmc/sgmcmc_vs_hmc_leapfrog100_log-stepsize-1.pdf}
\includegraphics[width=0.33\textwidth]{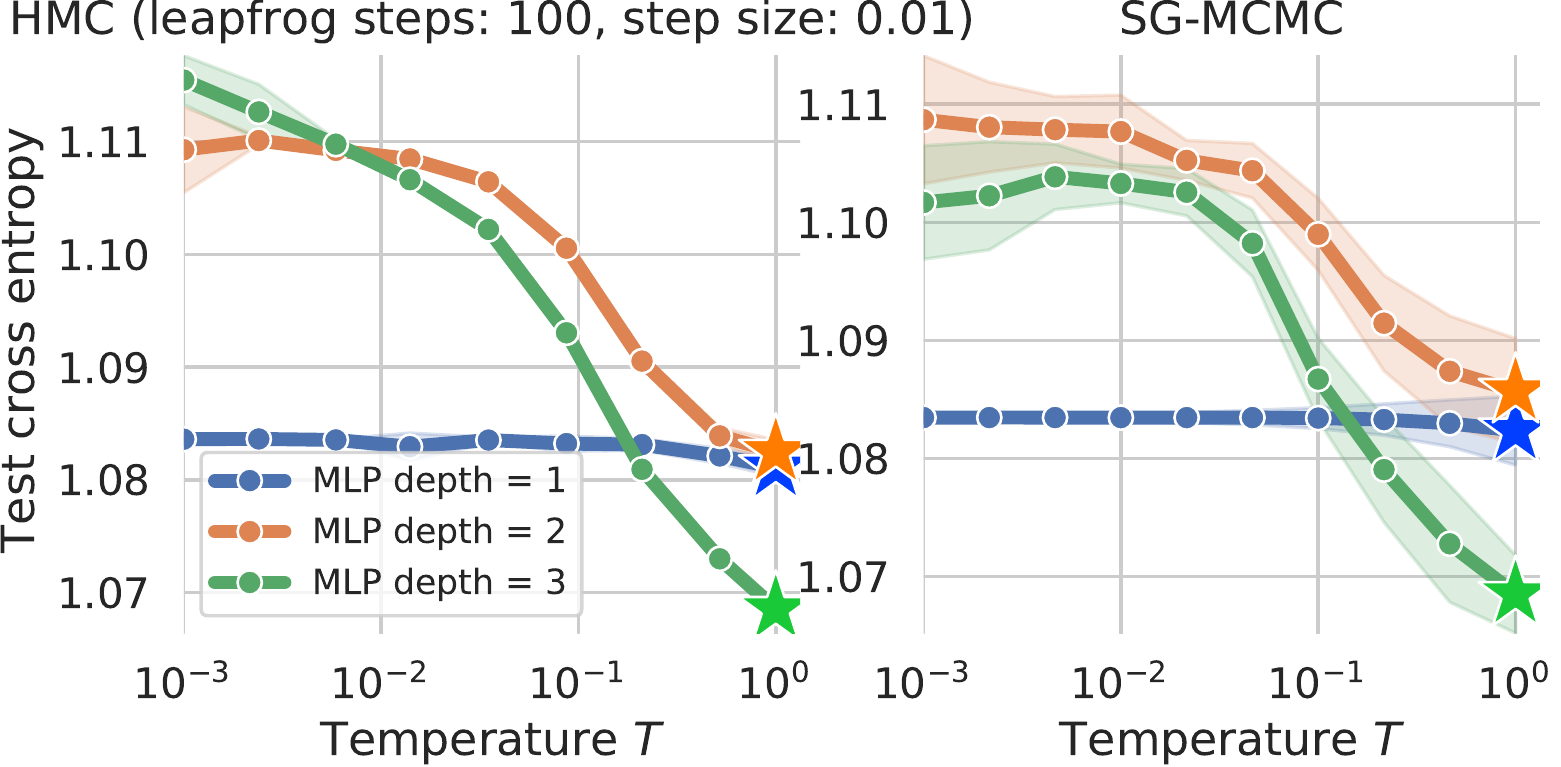}
\includegraphics[width=0.33\textwidth]{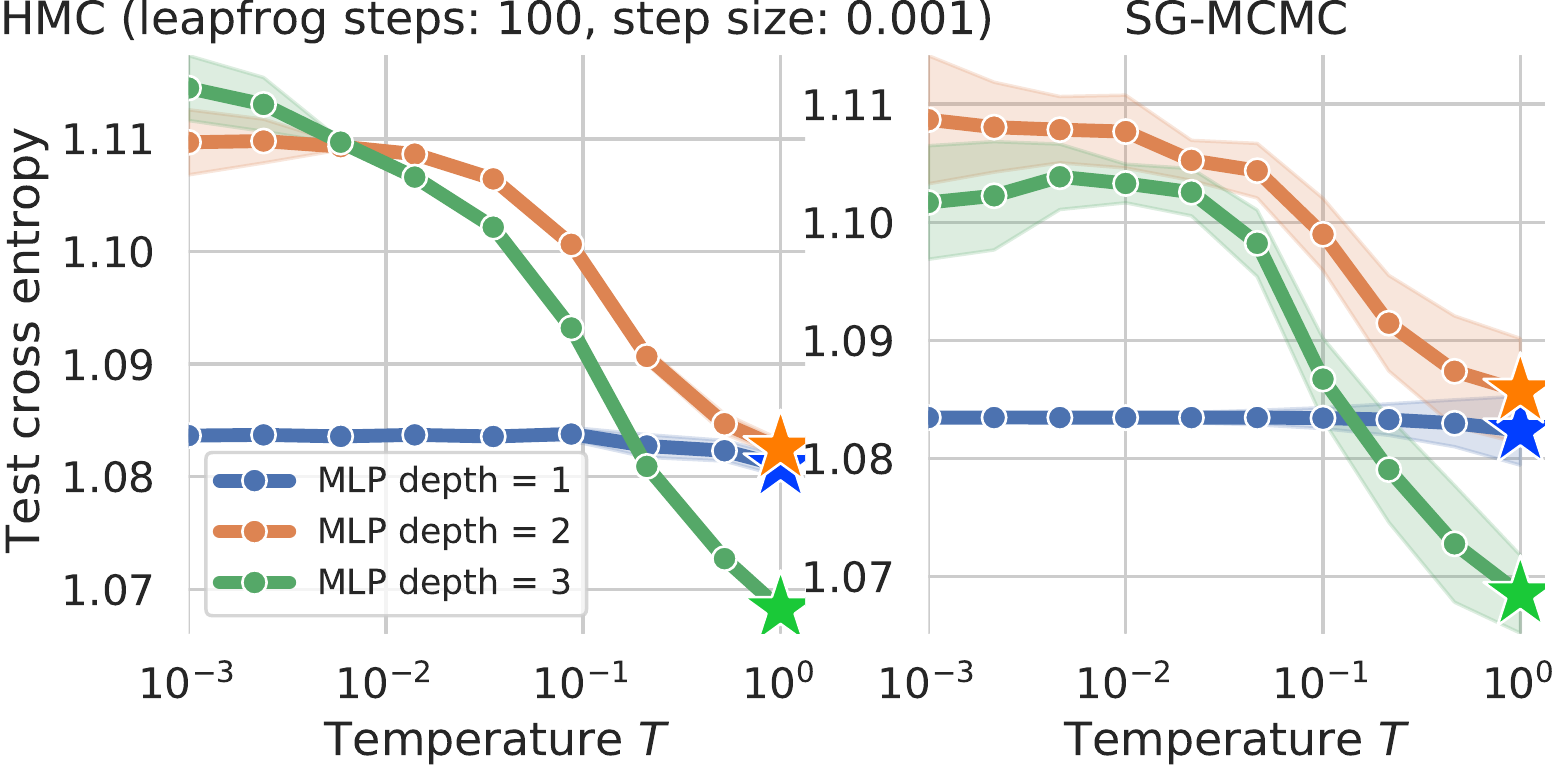}
\caption{Comparisons between SG-MCMC and HMC instantiated with different choices of leapfrog steps $L$ in $\{5, 10, 100\}$ and step sizes $\varepsilon$ in $\{0.001, 0.01, 0.1\}$. The curves show the (held-out) cross entropy versus different temperature levels, aggregated over 5 different runs, for MLPs of various depths (in $\{1,2,3\}$ with fixed number of units 10 and relu activation functions). Details about the dataset used can be found in the core paper. The setting $L=100$ and $\varepsilon=0.1$ corresponds to the results reported in the main paper.}%
\label{fig:sgmcmc_vc_hmc_all_leapfrog_steps_stepsize_pairs}%
\end{figure*}

\begin{figure*}[!t]
\center{\includegraphics[keepaspectratio=true,width=0.75\textwidth]{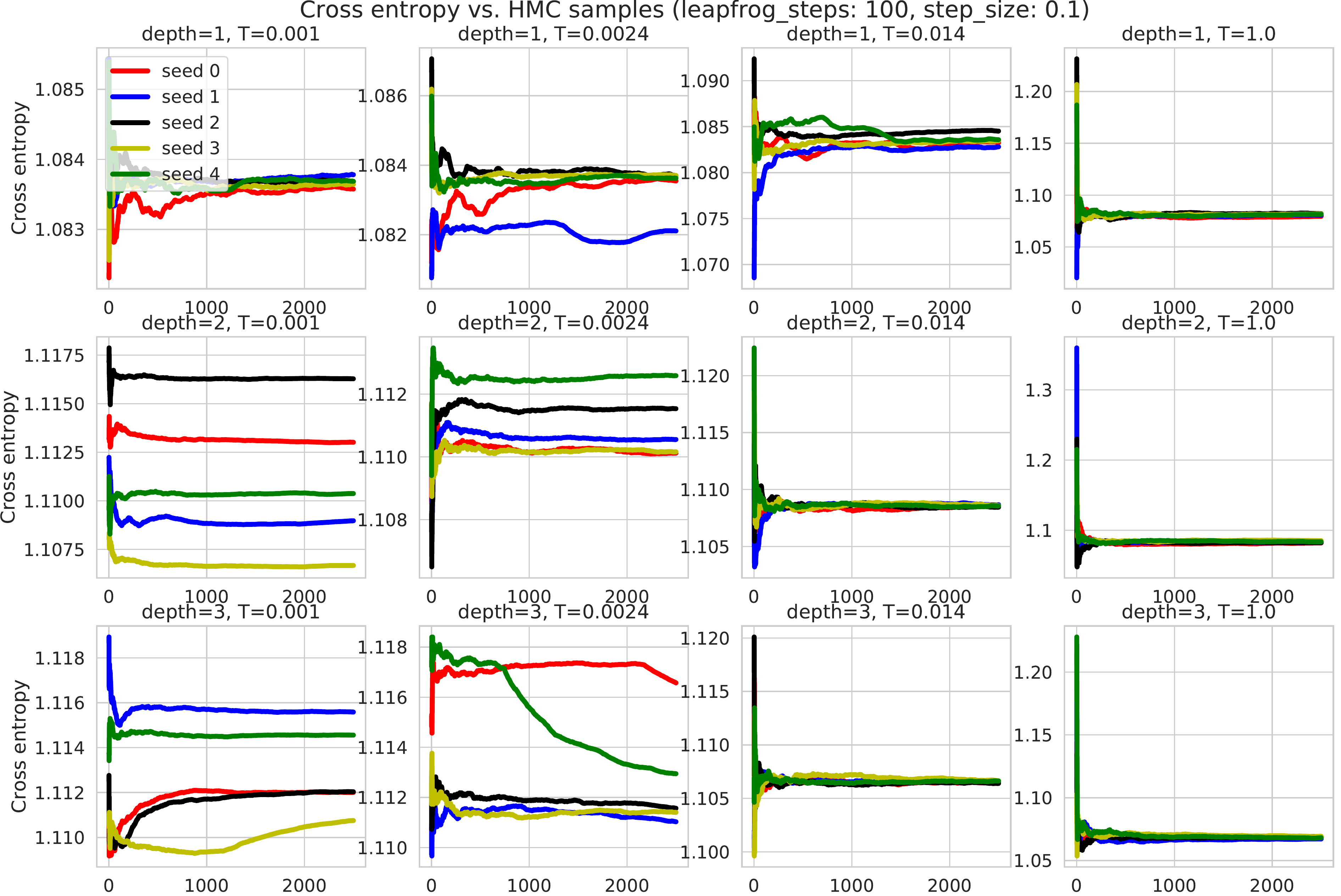}}
\caption{Trace plots of the cross entropy: We display the evolution of 5 different chains with respect to the $S=2500$ HMC samples collected after the burn-in phase, for various depths (rows) and temperatures (columns). Overall, the chains exhibit a converging behavior, with typically more dispersion as the depth and the temperature increase (which is reflected by the ranges of the y-axis that get wider as $T$ and the depth increase).}%
\label{fig:diagnostics_hmc_1}%
\end{figure*}

\begin{figure*}[!t]
\center{\includegraphics[keepaspectratio=true,width=0.75\textwidth]{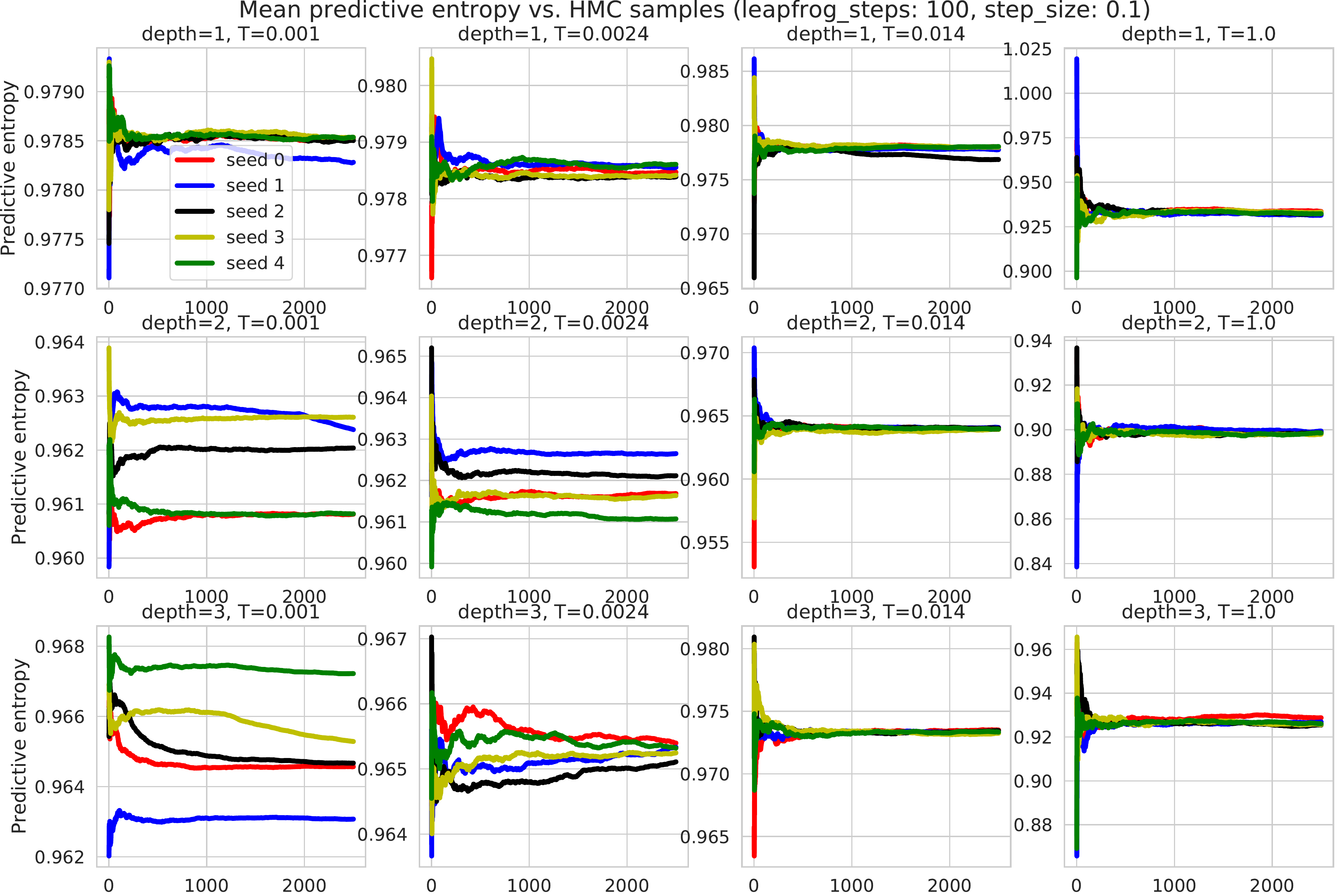}}
\caption{Trace plots of the mean predictive entropy (see definition in Section~\ref{sec:practical_usage_hmc}). We display the evolution of 5 different chains with respect to the $S=2500$ HMC samples collected after the burn-in phase, for various depths (rows) and temperatures (columns). See further discussions in Figure~\ref{fig:diagnostics_hmc_1}.}%
\label{fig:diagnostics_hmc_2}%
\end{figure*}

\clearpage

\begin{figure*}[!t]
\center{\includegraphics[keepaspectratio=true,width=0.75\textwidth]{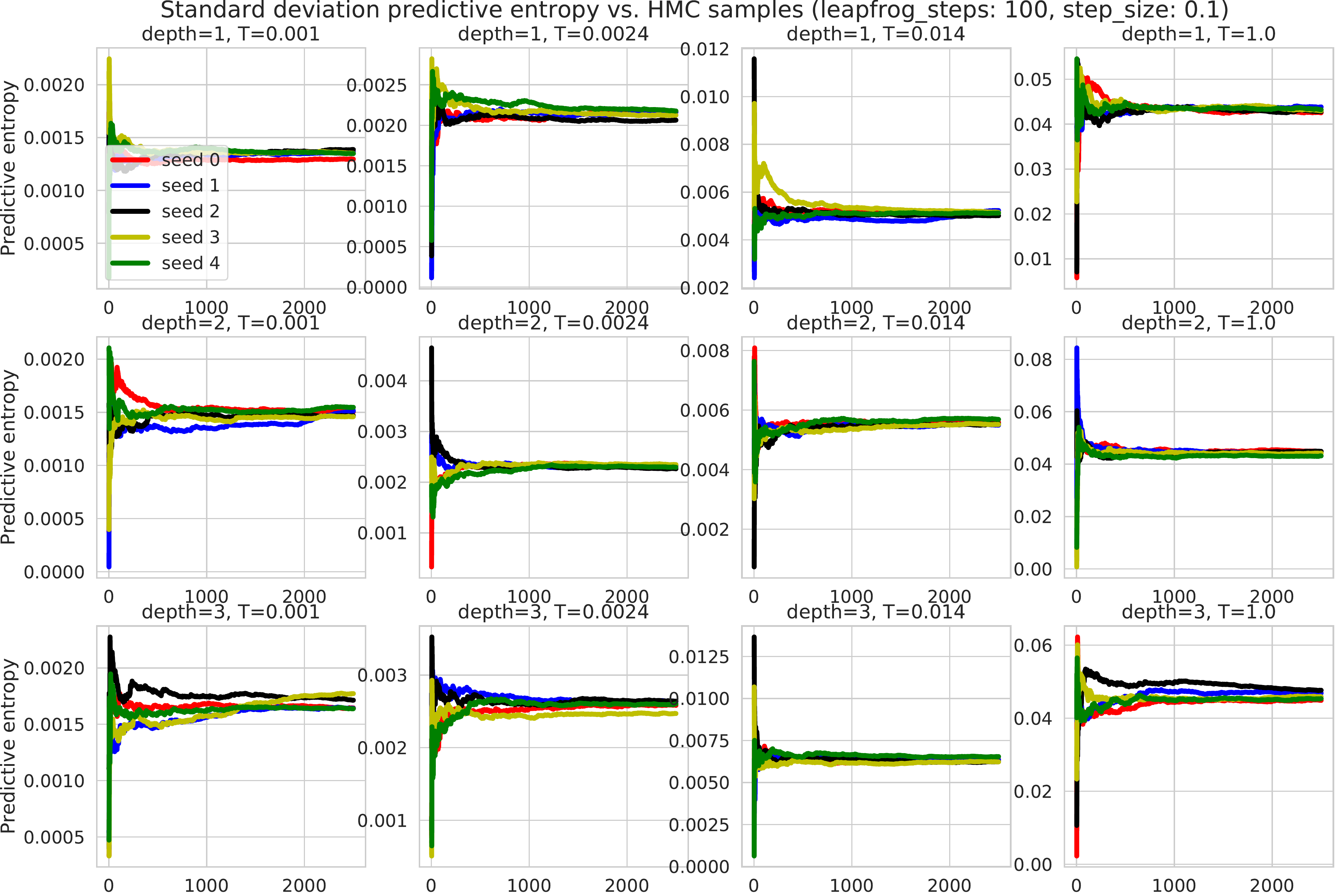}}%
\caption{Trace plots of the standard deviation of the predictive entropy (see definition in Section~\ref{sec:practical_usage_hmc}). We display the evolution of 5 different chains with respect to the $S=2500$ HMC samples collected after the burn-in phase, for various depths (rows) and temperatures (columns). See further discussions in Figure~\ref{fig:diagnostics_hmc_1}.}%
\label{fig:diagnostics_hmc_3}%
\end{figure*}

\clearpage

%% file: arxiv.bbl
\begin{thebibliography}{97}
\providecommand{\natexlab}[1]{#1}
\providecommand{\url}[1]{\texttt{#1}}
\expandafter\ifx\csname urlstyle\endcsname\relax
  \providecommand{\doi}[1]{doi: #1}\else
  \providecommand{\doi}{doi: \begingroup \urlstyle{rm}\Url}\fi

\bibitem[Abadi et~al.(2016)Abadi, Barham, Chen, Chen, Davis, Dean, Devin,
  Ghemawat, Irving, Isard, et~al.]{abadi2016tensorflow}
Abadi, M., Barham, P., Chen, J., Chen, Z., Davis, A., Dean, J., Devin, M.,
  Ghemawat, S., Irving, G., Isard, M., et~al.
\newblock Tensorflow: A system for large-scale machine learning.
\newblock In \emph{12th {USENIX} Symposium on Operating Systems Design and
  Implementation ({OSDI} 16)}, pp.\  265--283, 2016.

\bibitem[Ashukha et~al.(2020)Ashukha, Lyzhov, Molchanov, and
  Vetrov]{ashukha2019uncertaintyestimation}
Ashukha, A., Lyzhov, A., Molchanov, D., and Vetrov, D.
\newblock Pitfalls of in-domain uncertainty estimation and ensembling in deep
  learning.
\newblock In \emph{Eigth International Conference on Learning Representations
  ({ICLR} 2020)}, 2020.

\bibitem[Atanov et~al.(2018)Atanov, Ashukha, Molchanov, Neklyudov, and
  Vetrov]{atanov2018stochasticbatchnormalization}
Atanov, A., Ashukha, A., Molchanov, D., Neklyudov, K., and Vetrov, D.
\newblock Uncertainty estimation via stochastic batch normalization.
\newblock \emph{arXiv preprint arXiv:1802.04893}, 2018.

\bibitem[Bae et~al.(2018)Bae, Zhang, and Grosse]{bae2018kfacvb}
Bae, J., Zhang, G., and Grosse, R.
\newblock Eigenvalue corrected noisy natural gradient.
\newblock \emph{arXiv preprint arXiv:1811.12565}, 2018.

\bibitem[Baldock \& Marzari(2019)Baldock and Marzari]{baldock2019bayesiannn}
Baldock, R.~J. and Marzari, N.
\newblock Bayesian neural networks at finite temperature.
\newblock \emph{arXiv preprint, arXiv:1904.04154}, 2019.

\bibitem[Barber \& Bishop(1998)Barber and Bishop]{barber1998ensemblelearning}
Barber, D. and Bishop, C.~M.
\newblock Ensemble learning for multi-layer networks.
\newblock In \emph{Advances in neural information processing systems}, pp.\
  395--401, 1998.

\bibitem[Berger(1985)]{berger1985statisticaldecisiontheory}
Berger, J.~O.
\newblock \emph{Statistical decision theory and {Bayesian} analysis}.
\newblock Springer, 1985.

\bibitem[Betancourt \& Girolami(2015)Betancourt and
  Girolami]{betancourt2015hamiltonian}
Betancourt, M. and Girolami, M.
\newblock {Hamiltonian Monte Carlo} for hierarchical models.
\newblock \emph{Current trends in Bayesian methodology with applications},
  79:\penalty0 30, 2015.

\bibitem[Bhattacharya et~al.(2019)Bhattacharya, Pati, Yang,
  et~al.]{bhattacharya2019fractionalposteriors}
Bhattacharya, A., Pati, D., Yang, Y., et~al.
\newblock Bayesian fractional posteriors.
\newblock \emph{The Annals of Statistics}, 47\penalty0 (1):\penalty0 39--66,
  2019.

\bibitem[Blundell et~al.(2015)Blundell, Cornebise, Kavukcuoglu, and
  Wierstra]{blundell2015weightuncertainty}
Blundell, C., Cornebise, J., Kavukcuoglu, K., and Wierstra, D.
\newblock Weight uncertainty in neural network.
\newblock 37:\penalty0 1613--1622, 2015.

\bibitem[Brier(1950)]{brier1950verification}
Brier, G.~W.
\newblock Verification of forecasts expressed in terms of probability.
\newblock \emph{Monthly weather review}, 78\penalty0 (1):\penalty0 1--3, 1950.

\bibitem[Brooks et~al.(2011)Brooks, Gelman, Jones, and
  Meng]{brooks2011handbook}
Brooks, S., Gelman, A., Jones, G., and Meng, X.
\newblock \emph{Handbook of Markov Chain Monte Carlo}.
\newblock Chapman \& Hall/CRC Handbooks of Modern Statistical Methods. CRC
  Press, 2011.
\newblock ISBN 9781420079425.
\newblock URL \url{https://books.google.de/books?id=qfRsAIKZ4rIC}.

\bibitem[Chen et~al.(2015)Chen, Ding, and Carin]{chen2015convergencesgmcmc}
Chen, C., Ding, N., and Carin, L.
\newblock On the convergence of stochastic gradient {MCMC} algorithms with
  high-order integrators.
\newblock In \emph{Advances in Neural Information Processing Systems}, pp.\
  2278--2286, 2015.

\bibitem[Chen et~al.(2014)Chen, Fox, and Guestrin]{chen2014stochastichmc}
Chen, T., Fox, E., and Guestrin, C.
\newblock Stochastic gradient {Hamiltonian Monte Carlo}.
\newblock In \emph{International Conference on Machine Learning}, pp.\
  1683--1691, 2014.

\bibitem[de~G.~Matthews et~al.(2018)de~G.~Matthews, Hron, Rowland, Turner, and
  Ghahramani]{matthews2018gaussiannetworks}
de~G.~Matthews, A.~G., Hron, J., Rowland, M., Turner, R.~E., and Ghahramani, Z.
\newblock Gaussian process behaviour in wide deep neural networks.
\newblock In \emph{ICLR}, 2018.

\bibitem[Dieng et~al.(2018)Dieng, Ranganath, Altosaar, and
  Blei]{dieng2018noisin}
Dieng, A.~B., Ranganath, R., Altosaar, J., and Blei, D.~M.
\newblock Noisin: Unbiased regularization for recurrent neural networks.
\newblock \emph{arXiv preprint arXiv:1805.01500}, 2018.

\bibitem[Dillon et~al.(2017)Dillon, Langmore, Tran, Brevdo, Vasudevan, Moore,
  Patton, Alemi, Hoffman, and Saurous]{dillon2017tensorflow}
Dillon, J.~V., Langmore, I., Tran, D., Brevdo, E., Vasudevan, S., Moore, D.,
  Patton, B., Alemi, A., Hoffman, M., and Saurous, R.~A.
\newblock Tensorflow distributions.
\newblock \emph{arXiv preprint arXiv:1711.10604}, 2017.

\bibitem[Ding et~al.(2014)Ding, Fang, Babbush, Chen, Skeel, and
  Neven]{ding2014bayesian}
Ding, N., Fang, Y., Babbush, R., Chen, C., Skeel, R.~D., and Neven, H.
\newblock Bayesian sampling using stochastic gradient thermostats.
\newblock In \emph{Advances in neural information processing systems}, pp.\
  3203--3211, 2014.

\bibitem[Duane et~al.(1987)Duane, Kennedy, Pendleton, and
  Roweth]{duane1987hybrid}
Duane, S., Kennedy, A.~D., Pendleton, B.~J., and Roweth, D.
\newblock Hybrid {Monte Carlo}.
\newblock \emph{Physics letters B}, 195\penalty0 (2):\penalty0 216--222, 1987.

\bibitem[Earl \& Deem(2005)Earl and Deem]{earl2005paralleltempering}
Earl, D.~J. and Deem, M.~W.
\newblock Parallel tempering: Theory, applications, and new perspectives.
\newblock \emph{Physical Chemistry Chemical Physics}, 7\penalty0 (23):\penalty0
  3910--3916, 2005.

\bibitem[Flam-Shepherd et~al.(2017)Flam-Shepherd, Requeima, and
  Duvenaud]{flam2017mapping}
Flam-Shepherd, D., Requeima, J., and Duvenaud, D.
\newblock Mapping gaussian process priors to bayesian neural networks.
\newblock In \emph{NIPS Bayesian deep learning workshop}, 2017.

\bibitem[Fushiki et~al.(2005)]{fushiki2005bootstrapprediction}
Fushiki, T. et~al.
\newblock Bootstrap prediction and {Bayesian} prediction under misspecified
  models.
\newblock \emph{Bernoulli}, 11\penalty0 (4):\penalty0 747--758, 2005.

\bibitem[Gal \& Ghahramani(2016)Gal and Ghahramani]{gal2016dropout}
Gal, Y. and Ghahramani, Z.
\newblock Dropout as a {Bayesian} approximation: Representing model uncertainty
  in deep learning.
\newblock In \emph{international conference on machine learning}, pp.\
  1050--1059, 2016.

\bibitem[Garriga-Alonso et~al.(2019)Garriga-Alonso, Rasmussen, and
  Aitchison]{garrigaalonso2019convnetgp}
Garriga-Alonso, A., Rasmussen, C.~E., and Aitchison, L.
\newblock Deep convolutional networks as shallow gaussian processes.
\newblock In \emph{ICLR}, 2019.

\bibitem[Geisser(1993)]{geisser1993predictiveinference}
Geisser, S.
\newblock \emph{An Introduction to Predictive Inference}.
\newblock Chapman and Hall, New York, 1993.

\bibitem[Gelman \& Rubin(1992)Gelman and Rubin]{gelman1992psrf}
Gelman, A. and Rubin, D.~B.
\newblock Inference from iterative simulation using multiple sequences.
\newblock \emph{Statistical science}, 7\penalty0 (4):\penalty0 457--472, 1992.

\bibitem[Gelman et~al.(2013)Gelman, Carlin, Stern, Dunson, Vehtari, and
  Rubin]{gelman2013bayesiandataanalysis}
Gelman, A., Carlin, J.~B., Stern, H.~S., Dunson, D.~B., Vehtari, A., and Rubin,
  D.~B.
\newblock \emph{Bayesian data analysis}.
\newblock Chapman and Hall/CRC, 2013.

\bibitem[Germain et~al.(2016)Germain, Bach, Lacoste, and
  Lacoste-Julien]{germain2016pacbayesian}
Germain, P., Bach, F., Lacoste, A., and Lacoste-Julien, S.
\newblock Pac-bayesian theory meets bayesian inference.
\newblock In \emph{Advances in Neural Information Processing Systems}, pp.\
  1884--1892, 2016.

\bibitem[Glorot \& Bengio(2010)Glorot and Bengio]{glorot2010init}
Glorot, X. and Bengio, Y.
\newblock Understanding the difficulty of training deep feedforward neural
  networks.
\newblock In \emph{Proceedings of the thirteenth international conference on
  artificial intelligence and statistics}, pp.\  249--256, 2010.

\bibitem[Gr{\"u}nwald et~al.(2017)Gr{\"u}nwald, Van~Ommen,
  et~al.]{grunwald2017inconsistency}
Gr{\"u}nwald, P., Van~Ommen, T., et~al.
\newblock Inconsistency of {Bayesian} inference for misspecified linear models,
  and a proposal for repairing it.
\newblock \emph{Bayesian Analysis}, 12\penalty0 (4):\penalty0 1069--1103, 2017.

\bibitem[Hafner et~al.(2018)Hafner, Tran, Lillicrap, Irpan, and
  Davidson]{hafner2018noise}
Hafner, D., Tran, D., Lillicrap, T., Irpan, A., and Davidson, J.
\newblock Noise contrastive priors for functional uncertainty.
\newblock \emph{arXiv preprint arXiv:1807.09289}, 2018.

\bibitem[H{\"a}ggstr{\"o}m \& Rosenthal(2007)H{\"a}ggstr{\"o}m and
  Rosenthal]{haggstrom2007varianceclt}
H{\"a}ggstr{\"o}m, O. and Rosenthal, J.
\newblock On variance conditions for markov chain clts.
\newblock \emph{Electronic Communications in Probability}, 12:\penalty0
  454--464, 2007.

\bibitem[Hayou et~al.(2018)Hayou, Doucet, and
  Rousseau]{hayou2018initializationactivation}
Hayou, S., Doucet, A., and Rousseau, J.
\newblock On the selection of initialization and activation function for deep
  neural networks.
\newblock \emph{arXiv preprint arXiv:1805.08266}, 2018.

\bibitem[He et~al.(2015)He, Zhang, Ren, and Sun]{he2015prelu}
He, K., Zhang, X., Ren, S., and Sun, J.
\newblock Delving deep into rectifiers: Surpassing human-level performance on
  imagenet classification.
\newblock In \emph{Proceedings of the IEEE international conference on computer
  vision}, pp.\  1026--1034, 2015.

\bibitem[He et~al.(2016)He, Zhang, Ren, and Sun]{he2016resnet}
He, K., Zhang, X., Ren, S., and Sun, J.
\newblock Deep residual learning for image recognition.
\newblock In \emph{Proceedings of the IEEE conference on computer vision and
  pattern recognition}, pp.\  770--778, 2016.

\bibitem[Heber et~al.(2019)Heber, Trstanova, and Leimkuhler]{heber2019tati}
Heber, F., Trstanova, Z., and Leimkuhler, B.
\newblock Tati-thermodynamic analytics toolkit: Tensorflow-based software for
  posterior sampling in machine learning applications.
\newblock \emph{arXiv preprint arXiv:1903.08640}, 2019.

\bibitem[Heek \& Kalchbrenner(2020)Heek and Kalchbrenner]{heek2019sgmcmc}
Heek, J. and Kalchbrenner, N.
\newblock Bayesian inference for large scale image classification.
\newblock In \emph{International Conference on Learning Representations (ICLR
  2020)}, 2020.

\bibitem[Hinton \& Van~Camp(1993)Hinton and Van~Camp]{hinton1993mdl}
Hinton, G. and Van~Camp, D.
\newblock Keeping neural networks simple by minimizing the description length
  of the weights.
\newblock In \emph{in Proc. of the 6th Ann. ACM Conf. on Computational Learning
  Theory}, 1993.

\bibitem[Hoffman \& Gelman(2014)Hoffman and Gelman]{hoffman2014no}
Hoffman, M.~D. and Gelman, A.
\newblock The no-u-turn sampler: adaptively setting path lengths in
  {Hamiltonian Monte Carlo}.
\newblock \emph{Journal of Machine Learning Research}, 15\penalty0
  (1):\penalty0 1593--1623, 2014.

\bibitem[Inoue(2019)]{inoue2019multisampledropout}
Inoue, H.
\newblock Multi-sample dropout for accelerated training and better
  generalization.
\newblock \emph{arXiv preprint arXiv:1905.09788}, 2019.

\bibitem[Ioffe \& Szegedy(2015)Ioffe and Szegedy]{ioffe2015batchnormalization}
Ioffe, S. and Szegedy, C.
\newblock Batch normalization: Accelerating deep network training by reducing
  internal covariate shift.
\newblock \emph{arXiv preprint arXiv:1502.03167}, 2015.

\bibitem[Jansen(2013)]{jansen2013misspecification}
Jansen, L.
\newblock Robust {Bayesian} inference under model misspecification, 2013.
\newblock Master thesis.

\bibitem[Jones et~al.(2004)]{jones2004markovchainclt}
Jones, G.~L. et~al.
\newblock On the {Markov} chain central limit theorem.
\newblock \emph{Probability surveys}, 1\penalty0 (299-320):\penalty0 5--1,
  2004.

\bibitem[Kingma et~al.(2015)Kingma, Salimans, and
  Welling]{kingma2015variationaldropout}
Kingma, D.~P., Salimans, T., and Welling, M.
\newblock Variational dropout and the local reparameterization trick.
\newblock In \emph{Advances in Neural Information Processing Systems}, pp.\
  2575--2583, 2015.

\bibitem[Komaki(1996)]{komaki1996predictive}
Komaki, F.
\newblock On asymptotic properties of predictive distributions.
\newblock \emph{Biometrika}, 83\penalty0 (2):\penalty0 299--313, 1996.

\bibitem[Krizhevsky et~al.(2009)Krizhevsky, Hinton,
  et~al.]{krizhevsky2009cifar}
Krizhevsky, A., Hinton, G., et~al.
\newblock Learning multiple layers of features from tiny images.
\newblock Technical report, Citeseer, 2009.

\bibitem[Lakshminarayanan et~al.(2017)Lakshminarayanan, Pritzel, and
  Blundell]{lakshminarayanan2017deepensembles}
Lakshminarayanan, B., Pritzel, A., and Blundell, C.
\newblock Simple and scalable predictive uncertainty estimation using deep
  ensembles.
\newblock In \emph{Advances in Neural Information Processing Systems 30}. 2017.

\bibitem[Langevin(1908)]{langevin1908}
Langevin, P.
\newblock Sur la th{\'e}orie du mouvement brownien.
\newblock \emph{Compt. Rendus}, 146:\penalty0 530--533, 1908.

\bibitem[Lao et~al.(2020)Lao, Suter, Langmore, Chimisov, Saxena, Sountsov,
  Moore, Saurous, Hoffman, and Dillon]{lao2020tfpmcmc}
Lao, J., Suter, C., Langmore, I., Chimisov, C., Saxena, A., Sountsov, P.,
  Moore, D., Saurous, R.~A., Hoffman, M.~D., and Dillon, J.~V.
\newblock tfp.mcmc: Modern {Markov} chain {Monte Carlo} tools built for modern
  hardware, 2020.

\bibitem[Lee et~al.(2018)Lee, Bahri, Novak, Schoenholz, Pennington, and
  Sohl-Dickstein]{lee2018dnngp}
Lee, J., Bahri, Y., Novak, R., Schoenholz, S.~S., Pennington, J., and
  Sohl-Dickstein, J.
\newblock Deep neural networks as gaussian processes.
\newblock In \emph{ICLR}, 2018.

\bibitem[Leimkuhler \& Matthews(2016)Leimkuhler and
  Matthews]{leimkuhler2016molecular}
Leimkuhler, B. and Matthews, C.
\newblock \emph{Molecular Dynamics}.
\newblock Springer, 2016.

\bibitem[Leimkuhler et~al.(2019)Leimkuhler, Matthews, and
  Vlaar]{leimkuhler2019partitionedlangevin}
Leimkuhler, B., Matthews, C., and Vlaar, T.
\newblock Partitioned integrators for thermodynamic parameterization of neural
  networks.
\newblock \emph{arXiv preprint arXiv:1908.11843}, 2019.

\bibitem[Li et~al.(2016)Li, Chen, Carlson, and Carin]{li2016preconditionedsgld}
Li, C., Chen, C., Carlson, D., and Carin, L.
\newblock Preconditioned stochastic gradient {Langevin} dynamics for deep
  neural networks.
\newblock In \emph{Thirtieth AAAI Conference on Artificial Intelligence}, 2016.

\bibitem[Liao \& Berg(2019)Liao and Berg]{liao2019sharpeningjensen}
Liao, J. and Berg, A.
\newblock Sharpening jensen's inequality.
\newblock \emph{The American Statistician}, 73\penalty0 (3):\penalty0 278--281,
  2019.

\bibitem[Ma et~al.(2015)Ma, Chen, and Fox]{ma2015complete}
Ma, Y.-A., Chen, T., and Fox, E.
\newblock A complete recipe for stochastic gradient {MCMC}.
\newblock In \emph{Advances in Neural Information Processing Systems}, pp.\
  2917--2925, 2015.

\bibitem[MacKay et~al.(1995)]{mackay1995ensemblelearning}
MacKay, D.~J. et~al.
\newblock Ensemble learning and evidence maximization.
\newblock In \emph{Proc. Nips}, volume~10, pp.\  4083. Citeseer, 1995.

\bibitem[Mandt et~al.(2016)Mandt, McInerney, Abrol, Ranganath, and
  Blei]{DBLP:conf/aistats/MandtMARB16}
Mandt, S., McInerney, J., Abrol, F., Ranganath, R., and Blei, D.~M.
\newblock Variational tempering.
\newblock In \emph{Proceedings of the 19th International Conference on
  Artificial Intelligence and Statistics, {AISTATS}}, {JMLR} Workshop and
  Conference Proceedings, 2016.

\bibitem[Mandt et~al.(2017)Mandt, Hoffman, and Blei]{mandt2017stochastic}
Mandt, S., Hoffman, M.~D., and Blei, D.~M.
\newblock Stochastic gradient descent as approximate {Bayesian} inference.
\newblock \emph{The Journal of Machine Learning Research}, 18\penalty0
  (1):\penalty0 4873--4907, 2017.

\bibitem[Masegosa(2019)]{masegosa2019misspecification}
Masegosa, A.~R.
\newblock Learning under model misspecification: Applications to variational
  and ensemble methods.
\newblock \emph{arXiv preprint, arXiv:19012.08335}, 2019.

\bibitem[Masters \& Luschi(2018)Masters and Luschi]{masters2018batchsize}
Masters, D. and Luschi, C.
\newblock Revisiting small batch training for deep neural networks.
\newblock \emph{arXiv preprint arXiv:1804.07612}, 2018.

\bibitem[Murphy(2012)]{murphy2012machine}
Murphy, K.~P.
\newblock \emph{Machine learning: a probabilistic perspective}.
\newblock MIT press, 2012.

\bibitem[Naeini et~al.(2015)Naeini, Cooper, and
  Hauskrecht]{naeini2015obtaining}
Naeini, M.~P., Cooper, G., and Hauskrecht, M.
\newblock Obtaining well calibrated probabilities using bayesian binning.
\newblock In \emph{Twenty-Ninth AAAI Conference on Artificial Intelligence},
  2015.

\bibitem[Nalisnick et~al.(2019)Nalisnick, Hernandez-Lobato, and
  Smyth]{nalisnick2019dropout}
Nalisnick, E., Hernandez-Lobato, J.~M., and Smyth, P.
\newblock Dropout as a structured shrinkage prior.
\newblock In \emph{International Conference on Machine Learning}, pp.\
  4712--4722, 2019.

\bibitem[Neal(1995)]{neal1995bayesianneuralnetworks}
Neal, R.~M.
\newblock \emph{Bayesian learning for neural networks}.
\newblock PhD thesis, University of Toronto, 1995.

\bibitem[Neal et~al.(2011)]{neal2011mcmc}
Neal, R.~M. et~al.
\newblock {MCMC} using {Hamiltonian} dynamics.
\newblock \emph{Handbook of {Markov} chain {Monte Carlo}}, 2\penalty0
  (11):\penalty0 2, 2011.

\bibitem[Nemenman et~al.(2002)Nemenman, Shafee, and
  Bialek]{nemenman2002entropy}
Nemenman, I., Shafee, F., and Bialek, W.
\newblock Entropy and inference, revisited.
\newblock In \emph{Advances in neural information processing systems}, pp.\
  471--478, 2002.

\bibitem[Noh et~al.(2017)Noh, You, Mun, and Han]{noh2017regularizing}
Noh, H., You, T., Mun, J., and Han, B.
\newblock Regularizing deep neural networks by noise: Its interpretation and
  optimization.
\newblock In \emph{Advances in Neural Information Processing Systems}, pp.\
  5109--5118, 2017.

\bibitem[Novak et~al.(2019)Novak, Xiao, Bahri, Lee, Yang, Hron, Abolafia,
  Pennington, and Sohl-Dickstein]{novak2019bayesiandeepcnngp}
Novak, R., Xiao, L., Bahri, Y., Lee, J., Yang, G., Hron, J., Abolafia, D.~A.,
  Pennington, J., and Sohl-Dickstein, J.
\newblock Bayesian deep convolutional networks with many channels are gaussian
  processes.
\newblock In \emph{ICLR}, 2019.

\bibitem[Nowozin(2018)]{nowozin2018jvi}
Nowozin, S.
\newblock Debiasing evidence approximations: On importance-weighted
  autoencoders and jackknife variational inference.
\newblock In \emph{Sixth International Conference on Learning Representations
  ({ICLR} 2018)}, 2018.

\bibitem[Osawa et~al.(2019)Osawa, Swaroop, Jain, Eschenhagen, Turner, Yokota,
  and Khan]{osawa2019practicalbdl}
Osawa, K., Swaroop, S., Jain, A., Eschenhagen, R., Turner, R.~E., Yokota, R.,
  and Khan, M.~E.
\newblock Practical deep learning with {Bayesian} principles.
\newblock \emph{arXiv preprint arXiv:1906.02506}, 2019.

\bibitem[Ovadia et~al.(2019)Ovadia, Fertig, Ren, Nado, Sculley, Nowozin,
  Dillon, Lakshminarayanan, and Snoek]{ovadia2019modeluncertainty}
Ovadia, Y., Fertig, E., Ren, J., Nado, Z., Sculley, D., Nowozin, S., Dillon,
  J.~V., Lakshminarayanan, B., and Snoek, J.
\newblock Can you trust your model's uncertainty? evaluating predictive
  uncertainty under dataset shift.
\newblock In \emph{Advances in Neural Information Processing Systems (NeurIPS
  2019)}, 2019.

\bibitem[Pearce et~al.(2019)Pearce, Zaki, Brintrup, and
  Neely]{pearce2019expressive}
Pearce, T., Zaki, M., Brintrup, A., and Neely, A.
\newblock Expressive priors in bayesian neural networks: Kernel combinations
  and periodic functions.
\newblock \emph{arXiv preprint arXiv:1905.06076}, 2019.

\bibitem[Perez \& Wang(2017)Perez and Wang]{perez2017dataaugmentation}
Perez, L. and Wang, J.
\newblock The effectiveness of data augmentation in image classification using
  deep learning.
\newblock \emph{arXiv preprint arXiv:1712.04621}, 2017.

\bibitem[Ramamoorthi et~al.(2015)Ramamoorthi, Sriram, and
  Martin]{ramamoorthi2015misspecification}
Ramamoorthi, R.~V., Sriram, K., and Martin, R.
\newblock On posterior concentration in misspecified models.
\newblock \emph{Bayesian Anal.}, 10\penalty0 (4):\penalty0 759--789, 12 2015.
\newblock \doi{10.1214/15-BA941}.

\bibitem[S{\"a}rkk{\"a} \& Solin(2019)S{\"a}rkk{\"a} and
  Solin]{sarkka2019appliedsde}
S{\"a}rkk{\"a}, S. and Solin, A.
\newblock \emph{Applied stochastic differential equations}, volume~10.
\newblock Cambridge University Press, 2019.

\bibitem[Schoenholz et~al.(2017)Schoenholz, Gilmer, Ganguli, and
  Sohl-Dickstein]{schoenholz2017deepinformationpropagation}
Schoenholz, S.~S., Gilmer, J., Ganguli, S., and Sohl-Dickstein, J.
\newblock Deep information propagation.
\newblock In \emph{ICLR}, 2017.

\bibitem[Schucany et~al.(1971)Schucany, Gray, and
  Owen]{schucany1971biasreduction}
Schucany, W., Gray, H., and Owen, D.
\newblock On bias reduction in estimation.
\newblock \emph{Journal of the American Statistical Association}, 66\penalty0
  (335):\penalty0 524--533, 1971.

\bibitem[Shang et~al.(2015)Shang, Zhu, Leimkuhler, and
  Storkey]{shang2015covariancethermostat}
Shang, X., Zhu, Z., Leimkuhler, B., and Storkey, A.~J.
\newblock Covariance-controlled adaptive {Langevin} thermostat for large-scale
  {Bayesian} sampling.
\newblock In \emph{Advances in Neural Information Processing Systems}, pp.\
  37--45, 2015.

\bibitem[Shekhovtsov \& Flach(2018)Shekhovtsov and
  Flach]{shekhovtsov2018stochasticnormalizations}
Shekhovtsov, A. and Flach, B.
\newblock Stochastic normalizations as {Bayesian} learning.
\newblock \emph{arXiv preprint arXiv:1811.00639}, 2018.

\bibitem[Simsekli et~al.(2019)Simsekli, Sagun, and
  Gurbuzbalaban]{simsekli2019minibatchnoise}
Simsekli, U., Sagun, L., and Gurbuzbalaban, M.
\newblock A tail-index analysis of stochastic gradient noise in deep neural
  networks.
\newblock \emph{arXiv preprint arXiv:1901.06053}, 2019.

\bibitem[Singh \& Krishnan(2019)Singh and
  Krishnan]{singh2019filterresponsenormalization}
Singh, S. and Krishnan, S.
\newblock Filter response normalization layer: Eliminating batch dependence in
  the training of deep neural networks.
\newblock \emph{arXiv preprint arXiv:1911.09737}, 2019.

\bibitem[Srivastava et~al.(2014)Srivastava, Hinton, Krizhevsky, Sutskever, and
  Salakhutdinov]{srivastava2014dropout}
Srivastava, N., Hinton, G., Krizhevsky, A., Sutskever, I., and Salakhutdinov,
  R.
\newblock Dropout: a simple way to prevent neural networks from overfitting.
\newblock \emph{The journal of machine learning research}, 15\penalty0
  (1):\penalty0 1929--1958, 2014.

\bibitem[Sugita \& Okamoto(1999)Sugita and Okamoto]{sugita1999replicaexchange}
Sugita, Y. and Okamoto, Y.
\newblock Replica-exchange molecular dynamics method for protein folding.
\newblock \emph{Chemical physics letters}, 314\penalty0 (1-2):\penalty0
  141--151, 1999.

\bibitem[Sun et~al.(2019)Sun, Zhang, Shi, and Grosse]{sun2019functionalvb}
Sun, S., Zhang, G., Shi, J., and Grosse, R.
\newblock Functional variational {Bayesian} neural networks.
\newblock \emph{arXiv preprint arXiv:1903.05779}, 2019.

\bibitem[Sutskever et~al.(2013)Sutskever, Martens, Dahl, and
  Hinton]{sutskever2013momentum}
Sutskever, I., Martens, J., Dahl, G., and Hinton, G.
\newblock On the importance of initialization and momentum in deep learning.
\newblock In \emph{International conference on machine learning}, pp.\
  1139--1147, 2013.

\bibitem[Swendsen \& Wang(1986)Swendsen and
  Wang]{swendsen1986replicamontecarlo}
Swendsen, R.~H. and Wang, J.-S.
\newblock Replica {Monte Carlo} simulation of spin-glasses.
\newblock \emph{Physical review letters}, 57\penalty0 (21):\penalty0 2607,
  1986.

\bibitem[Tieleman \& Hinton(2012)Tieleman and Hinton]{tieleman2012rmsprop}
Tieleman, T. and Hinton, G.
\newblock {Lecture 6.5---{RmsProp}: Divide the gradient by a running average of
  its recent magnitude}.
\newblock Coursera: Neural Networks for Machine Learning, 2012.

\bibitem[Welling \& Teh(2011)Welling and Teh]{welling2011bayesian}
Welling, M. and Teh, Y.~W.
\newblock Bayesian learning via stochastic gradient {Langevin} dynamics.
\newblock In \emph{Proceedings of the 28th international conference on machine
  learning (ICML-11)}, pp.\  681--688, 2011.

\bibitem[Wen et~al.(2019)Wen, Tran, and Ba]{wen2019batchensemble}
Wen, Y., Tran, D., and Ba, J.
\newblock {BatchEnsemble}: Efficient ensemble of deep neural networks via
  rank-1 perturbation.
\newblock 2019.
\newblock {Bayesian} deep learning workshop 2019.

\bibitem[Wilson(2019)]{wilson2019bayesian}
Wilson, A.~G.
\newblock The case for {B}ayesian deep learning.
\newblock \emph{NYU Courant Technical Report}, 2019.
\newblock Accessible at \url{https://cims.nyu.edu/~andrewgw/caseforbdl.pdf}.

\bibitem[Wilson \& Izmailov(2020)Wilson and
  Izmailov]{wilson2020bayesianperspective}
Wilson, A.~G. and Izmailov, P.
\newblock Bayesian deep learning and a probabilistic perspective of
  generalization.
\newblock \emph{arXiv preprint arXiv:2002.08791}, 2020.

\bibitem[Wolpert \& Wolf(1995)Wolpert and Wolf]{wolpert1995estimatingentropy}
Wolpert, D.~H. and Wolf, D.~R.
\newblock Estimating functions of probability distributions from a finite set
  of samples.
\newblock \emph{Physical Review E}, 52\penalty0 (6):\penalty0 6841, 1995.

\bibitem[Yaida(2018)]{yaida2018fluctuationdissipationsgd}
Yaida, S.
\newblock Fluctuation-dissipation relations for stochastic gradient descent.
\newblock \emph{arXiv preprint arXiv:1810.00004}, 2018.

\bibitem[Yang(2019)]{yang2019wideneuralnetworksgp}
Yang, G.
\newblock Scaling limits of wide neural networks with weight sharing: Gaussian
  process behavior, gradient independence, and neural tangent kernel
  derivation.
\newblock \emph{arXiv preprint arXiv:1902.04760}, 2019.

\bibitem[Zhang et~al.(2018)Zhang, Sun, Duvenaud, and
  Grosse]{zhang2018noisynaturalgradient}
Zhang, G., Sun, S., Duvenaud, D., and Grosse, R.
\newblock Noisy natural gradient as variational inference.
\newblock \emph{International Conference on Machine Learning}, 2018.

\bibitem[Zhang et~al.(2020)Zhang, Li, Zhang, Chen, and
  Wilson]{zhang2019cyclicalsgmcmc}
Zhang, R., Li, C., Zhang, J., Chen, C., and Wilson, A.~G.
\newblock Cyclical stochastic gradient {MCMC} for {Bayesian} deep learning.
\newblock In \emph{International Conference on Learning Representations (ICLR
  2020)}, 2020.

\bibitem[Zhu et~al.(2019)Zhu, Wu, Yu, Wu, and Ma]{zhu2019anisotropicsgdnoise}
Zhu, Z., Wu, J., Yu, B., Wu, L., and Ma, J.
\newblock The anisotropic noise in stochastic gradient descent: Its behavior of
  escaping from sharp minima and regularization effects.
\newblock In \emph{Proceedings of the 36th International Conference on Machine
  Learning (ICML 2019)}, 2019.

\end{thebibliography}
